\documentclass{article}

\usepackage{microtype}
\usepackage{graphicx}
\usepackage{subfigure}
\usepackage{booktabs} 

\usepackage{hyperref}

\usepackage{enumitem}
\usepackage{graphicx}
\usepackage{amsmath}
\usepackage{amssymb}
\usepackage{booktabs}

\usepackage{amsmath,amssymb,bbm,bm} %
\usepackage{amsthm}
\usepackage{float}
\usepackage{amsmath}
\usepackage{wrapfig}
\usepackage{xspace} 
\newcommand{\ie}{\emph{i.e.,}\xspace}
\newcommand{\eg}{\emph{e.g.,}\xspace}

\usepackage[accepted]{icml2023}

\usepackage{amsmath}
\usepackage{amssymb}
\usepackage{mathtools}
\usepackage{amsthm}

\usepackage[capitalize,noabbrev]{cleveref}

\theoremstyle{plain}
\newtheorem{theorem}{Theorem}

\newtheorem{corollary}{Corollary}
\theoremstyle{definition}
\newtheorem{definition}{Definition}

\theoremstyle{remark}

\usepackage[textsize=tiny]{todonotes}

\icmltitlerunning{Learn to Accumulate Evidence from All Training Samples: Theory and Practice}
\begin{document}

\twocolumn[
\icmltitle{Learn to Accumulate Evidence from All Training Samples: Theory and Practice}




\begin{icmlauthorlist}
\icmlauthor{Deep Pandey}{sch}
\icmlauthor{Qi Yu}{sch}
\end{icmlauthorlist}

\icmlaffiliation{sch}{Rochester Institute of Technology}

\icmlaffiliation{sch}{Rochester Institute of Technology}

\icmlcorrespondingauthor{Qi Yu}{qi.yu@rit.edu}

\icmlkeywords{Machine Learning, ICML, Meta-Learning, Uncertainty-Quantification, Evidential Deep Learning}

\vskip 0.3in
]



\printAffiliationsAndNotice{}  

\begin{abstract}
\label{sec:abstract}
Evidential deep learning, built upon belief theory and subjective logic, offers a principled and computationally efficient way to turn a deterministic neural network uncertainty-aware. The resultant evidential models can quantify fine-grained uncertainty using the learned evidence. To ensure theoretically sound evidential models, the evidence needs to be non-negative, which requires special activation functions for model training and inference. This constraint often leads to inferior predictive performance compared to standard softmax models, making it challenging to extend them to many large-scale datasets. To unveil the real cause of this undesired behavior, we theoretically investigate evidential models and identify a fundamental limitation that explains the inferior performance: existing evidential 
activation functions create {\em zero evidence regions}, which prevent the model to learn from training samples falling into such regions.  
A deeper analysis of evidential activation functions based on our theoretical underpinning inspires the design of a novel regularizer that effectively alleviates this fundamental limitation. Extensive experiments over many challenging real-world datasets and settings confirm our theoretical findings and demonstrate the effectiveness of our proposed approach.    
\end{abstract}

\vspace{-2mm}\section{Introduction}\vspace{-2mm}
\label{sec:introduction}
Deep Learning (DL) models have found great success in many real-world applications such as speech recognition \cite{kamath2019deep}, machine translation \cite{singh2017machine}, and computer vision \cite{voulodimos2018deep}. 
However, these highly expressive models may easily fit the noise in the training data, which leads to overconfident predictions  \cite{nguyen2015deep}. The challenge is further compounded when learning from limited labeled data, which is common for applications from specialized domain (\eg medicine, public safety, and military operations) where data collection and annotation is highly costly. 
Accurate uncertainty quantification is essential for successful application of DL models in these domains. To this end, DL models have been augmented to become uncertainty-aware~\cite{gal2016dropout,blundell2015weight,pearce2020uncertainty}. However, commonly used extensions require expensive sampling operations \cite{gal2016dropout,blundell2015weight}, which significantly increase the computational costs \cite{lakshminarayanan2017simple}. 

The recently developed evidential models bring together evidential theory~\cite{shafer1976mathematical,josang2016subjective} and deep neural architectures 
that turn a deterministic neural network uncertainty-aware. By leveraging the learned evidence,  evidential models are capable of quantifying fine-grained uncertainty that helps to identify the sources of `unknowns'. Furthermore, since only lightweight modifications are introduced to  existing DL architectures, additional computational costs remain minimum. 
Such evidential models have been successfully extended to classification \cite{sensoy2018evidential}, regression \cite{amini2020deep}, meta-learning \cite{Pandey_2022_CVPR}, and open-set recognition \cite{bao2021evidential} settings. 

\begin{wrapfigure}{r}{0.26\textwidth}
\vspace{-7mm}
  \begin{center}
    \hspace{-0.72cm}\includegraphics[width=\linewidth]{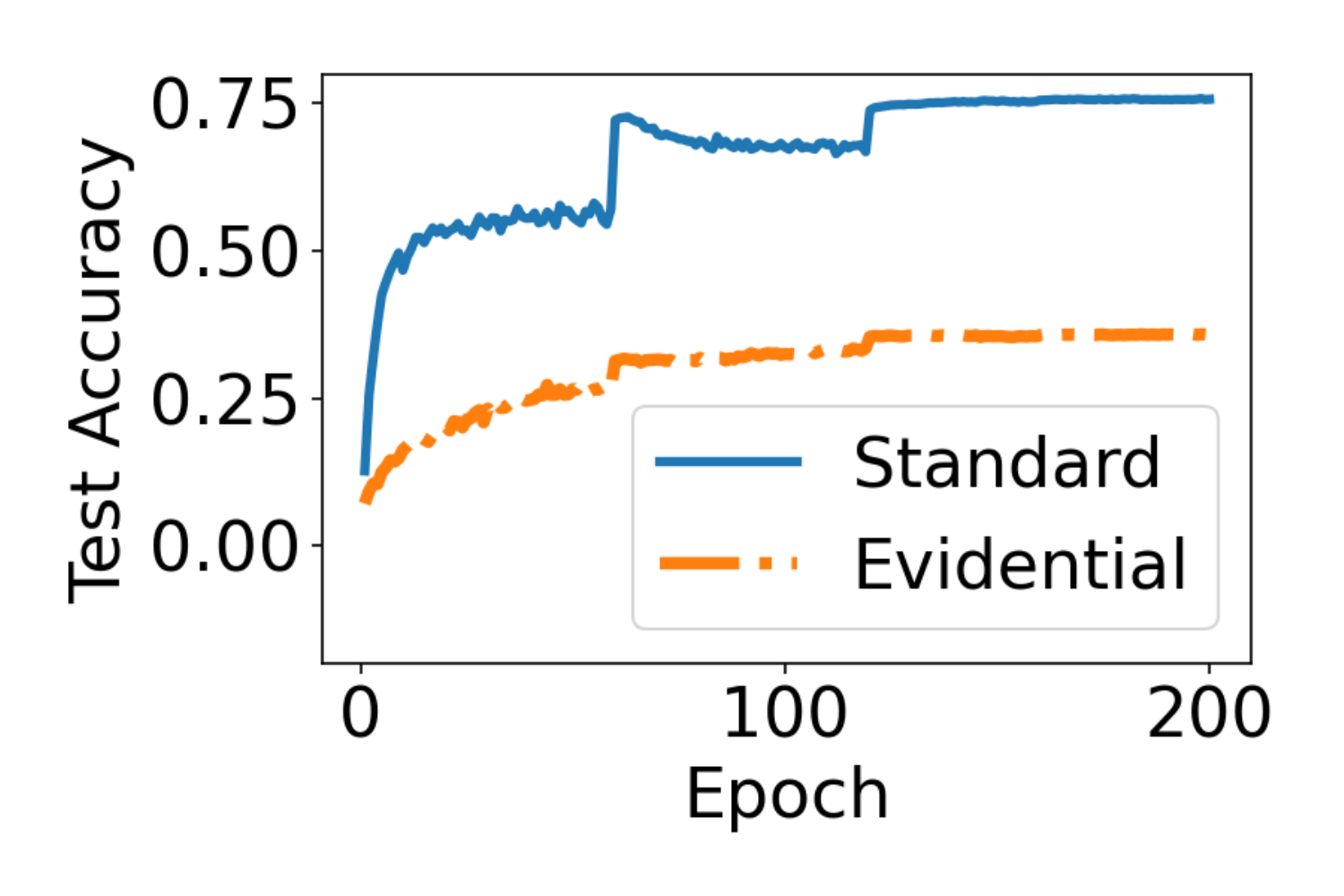}
  \end{center}
  \vspace{-8mm}
  \caption{Cifar100 Result}
  \label{fig:cifar100MSEComp}
\vspace{-4mm}
\end{wrapfigure}
Despite the attractive uncertainty quantification capacity, evidential models are only able to achieve a predictive performance on par with standard deep architectures in relatively simple learning problems. They suffer from a significant performance drop when facing large datasets with  more complex features even in the common classification setting. 
As shown in Figure \ref{fig:cifar100MSEComp}, an evidential model using ReLU activation and an evidential MSE loss \cite{sensoy2018evidential} only achieves  $~36\%$ test accuracy on Cifar100, which is almost 40\% lower than a standard model trained using softmax. Additionally, most evidential models can easily break down with minor architecture changes and/or have a much stronger dependency on hyperparameter tuning to achieve reasonable predictive performance. The experiment section provides more details on these failure cases. 

To train uncertainty-aware evidential models that can also predict well, we perform a novel theoretical analysis with a focus on the standard classification setting to unveil the underlying cause of the performance gap. Our theoretical results show that existing evidential models 
learn sub-optimally compared to corresponding softmax counterparts. Such sub-optimal training is mainly attributed to the inherent {\em learning deficiency} of evidential models that prevents them from learning across all training samples. More specifically, they are incapable to acquire new knowledge from training samples mapped to ``zero-evidence regions" in the evidence space, where the predicted evidence reduces to zero. The sub-optimal learning phenomenon is illustrated in Figure \ref{fig:intuitiveFailureOfRelu} (detailed discussion is presented in Section \ref{sec:intuitiveDescription}). We analyze different variants of evidential models present in the existing literature and observe this limitation across all the models and settings. Our theoretical results inspire the design of a novel \textbf{R}egularized \textbf{E}vidential mo\textbf{d}el (\textbf{RED}) that includes positive evidence regularization in its training objective to battle the learning deficiency. Our major contributions can be summarized as follows:
\begin{itemize}[noitemsep,topsep=0pt,leftmargin=*]
    \item We identify a fundamental limitation of evidential models, \ie lack the capability to learn from any data samples that lie in the ``zero-evidence" region in the evidence space.
    \item We theoretically show the superiority of evidential models with $\exp$ activation over other activation functions.
    \item We conduct novel evidence regularization that enables evidential models to avoid the ``zero-evidence" region so that they can effectively learn from all training samples.
    \item We carry out experiments over multiple challenging real-world datasets to empirically validate the presented theory, and show the effectiveness of our proposed ideas.
\end{itemize}
\begin{figure}[t!] 
\centering
  \includegraphics[width=0.90\linewidth]{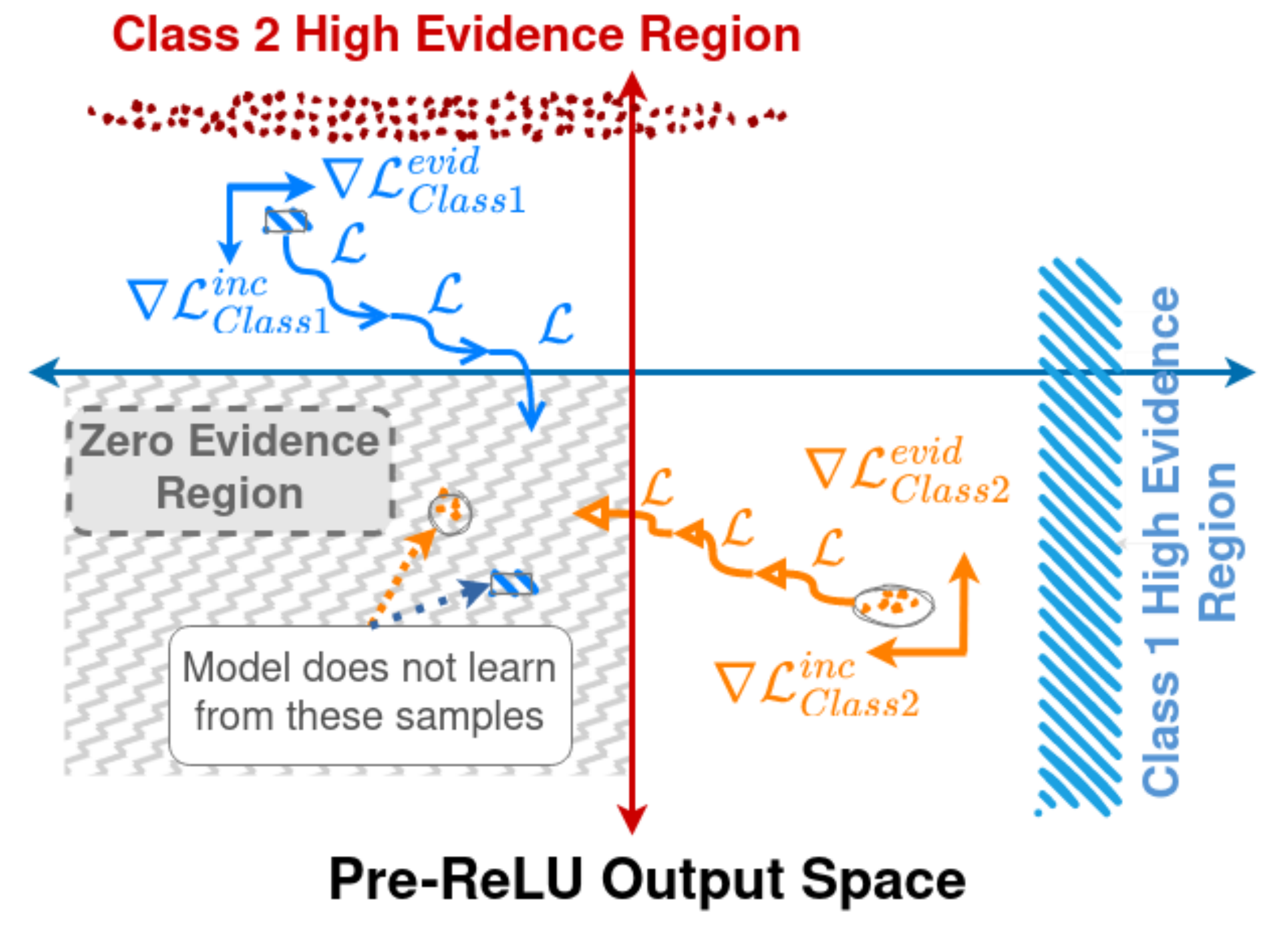}
  \vspace{-3mm}
\caption{\label{fig:intuitiveFailureOfRelu}Visualization of zero-evidence region for evidential models with $\texttt{ReLU}$ activation in a binary classification setting. Existing models fail to learn from samples that are mapped to such zero-evidence region (shared area at the bottom left quadrant).}
\vspace{-5mm}
\end{figure} 

\vspace{-2mm}\section{Related Works}\vspace{-2mm}

\paragraph{Uncertainty Quantification in Deep Learning.}
Accurate quantification of predictive uncertainty is essential for development of trustworthy Deep Learning (DL) models. Deep ensemble techniques \cite{pearce2020uncertainty,lakshminarayanan2017simple} have been developed for uncertainty quantification. An ensemble of neural networks is constructed and the agreement/disagreement across the ensemble components is used to quantify different uncertainties. Ensemble-based methods significantly increase the number of model parameters, which are computationally expensive at both training and test times. Alternatively, Bayesian neural networks  \cite{gal2016dropout}\cite{blundell2015weight}\cite{mobiny2021dropconnect} have been developed that consider a Bayesian formalism to quantify different uncertainties. For instance, \cite{blundell2015weight} use Bayes-by-backdrop to learn a distribution over neural network parameters, whereas \cite{gal2016dropout} enable dropout during inference phase to obtain predictive uncertainty. Bayesian methods resort to some form of approximation to address the intractability issue in marginalization of latent variables. Moreover, these methods are also computationally expensive as they require sampling for uncertainty quantification.
\vspace{-4mm}\paragraph{Evidential Deep Learning.}
Evidential models introduce a conjugate higher-order evidential prior for the likelihood distribution that enables the model to capture the fine-grained uncertainties. For instance, Dirichlet prior is introduced over the multinomial likelihood for evidential classification \cite{bao2021evidential,zhao2020uncertainty}, and NIG prior is introduced over the Gaussian likelihood \cite{amini2020deep, pandey2022evidential} for the evidential regression models. 
Adversarial robustness \cite{kopetzki2021evaluating} and calibration \cite{tomani2021towards} of evidential models have also been well studied. Usually, these models are trained with evidential losses in conjunction with heuristic evidence regularization to guide the uncertainty behavior \cite{Pandey_2022_CVPR,shi2020multifaceted} in addition to reasonable generalization performance. Some evidential models assume access to out-of-distribution data during training \cite{malinin2019reverse,malinin2018predictive} and use the OOD data to guide the uncertainty behavior. A recent survey  \cite{ulmer2021survey} provides a thorough review of the evidential deep learning field. 

In this work, we focus on evidential classification models and
consider settings where no OOD data is used during model training to make the proposed approach more broadly applicable to practical real-world situations.



\vspace{-2mm}\section{Learning Deficiency of Evidential Models}
\label{sec:EvDlModel}
\subsection{Preliminaries and problem setup}\vspace{-2mm}
Standard classification models use a softmax transformation on the output from the neural network $\mathcal{F}_{\Theta}$ for input $\mathbf{x}$ to obtain the class probabilities in $K$-class classification problem. 
Such models are trained with the cross-entropy based loss. For a given training sample $(\mathbf{x}, \mathbf{y})$, the loss is given by
\begin{align}
    \mathcal{L}_{cross} &= -\sum_{k = 1}^K \mathbf{y}_k \log (\texttt{sm}_k ) 
\end{align} 
where $\texttt{sm}_k$ is the softmax output.  
These models have achieved state-of-the-art performance on many benchmark problems. A detailed gradient analysis shows that they can effectively learn from all training data samples (see Appendix \ref{sec:appAnalysisStandardClassificationModels}). Nevertheless, these models lack a systematic mechanism to quantify different sources of uncertainty, a highly desired property in many real-world problems.

\begin{figure}[htpb]
\vspace{-3mm}
\centering
  \includegraphics[width=0.9\linewidth]{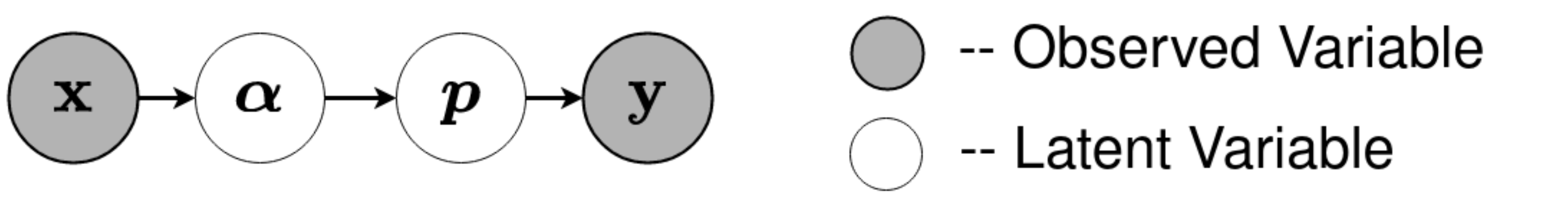}
  \caption{Graphical model for Evidential Deep Learning}
\label{fig:evidGraphicalRep}
\vspace{-2mm}
\end{figure}
Evidential classification models formulate training as an evidence acquisition process and consider a higher-order Dirichlet prior $\texttt{Dir}(\mathbf{p}|\boldsymbol{\alpha})$ over the predictive Multinomial distribution $\texttt{Mult}(\mathbf{y}|\mathbf{p})$. Different from a standard Bayesian formulation which optimizes {\em Type II Maximum Likelihood} to learn the Dirichlet hyperparameter~\cite{bishop2006pattern}, evidential models directly predict $\boldsymbol{\alpha}$ using data features ${\bf x}$ and then generate the prediction ${\bf y}$ by marginalizing the Multinomial parameter ${\bf p}$. Figure \ref{fig:evidGraphicalRep} describes this generative process. Such higher-order prior enables the model to systematically quantify different sources of uncertainty. In evidential models, the softmax layer of the standard neural networks is replaced by a non-negative activation function $\mathcal{A}$, where $\mathcal{A}({\bf x}) \geq 0 \quad \forall x \in [-\infty, \infty]$, such that for input $\mathbf{x}$, the neural network model $\mathcal{F}_{\Theta}$ with parameters  $\Theta$ can output evidence $\mathbf{e}$ for different classes. Dirichlet prior $\boldsymbol{\alpha}$ is evaluated as $\boldsymbol{\alpha} = \mathbf{e} + \boldsymbol{1}$ to ensure $\boldsymbol{\alpha} \geq 1$. The trained evidential model outputs Dirichlet parameters $\boldsymbol{\alpha}$ for input $\mathbf{x}$ that can quantify fine-grained uncertainties in addition to the prediction $\mathbf{y}$. Mathematically, for $K-$class classification problem,
\begin{align}
    &\texttt{Evidence}  (\mathbf{e}) = \mathcal{A}(\mathcal{F}_\Theta (\mathbf{x})) = \mathcal{A}(\mathbf{o})\\
    &\texttt{Dirichlet Parameter} (\boldsymbol{\alpha} ) = \mathbf{e} + \boldsymbol{1} \\
    &\texttt{Dirichlet Strength} (S) =  K + \sum_{k=1}^K \mathbf{e}_k 
\end{align}
The activation function $\mathcal{A}(\cdot)$ assumes three common forms to transform the neural network output into evidence: (1) $\texttt{ReLU}(\cdot) = \max( 0, \cdot )$, (2) $\texttt{SoftPlus}(\cdot) = \log ( 1 + \exp(\cdot) )$, and (3) $\exp(\cdot)$.

Evidential models assign input sample to that class for which the output evidence is greatest. Moreover, they quantify the confidence in the prediction for $K$ class classification problem through vacuity $\nu$ (\ie measure of lack of confidence in the prediction) computed as
\begin{align}
   \texttt{Vacuity} (\nu) = \frac{K}{S} 
\end{align}
For any training sample $(\mathbf{x}, \mathbf{y})$, the evidential models aim to maximize the evidence for the correct class, minimize the evidence for the incorrect classes, and output accurate confidence. To this end, three variants of evidential loss functions have been proposed~\cite{sensoy2018evidential}: 1) Bayes risk with sum of squares loss, 2) Bayes risk with cross-entropy loss, and 3) Type II Maximum Likelihood loss. Please refer to equations \eqref{eqn:evMSEloss}, \eqref{eqn:evDigammaloss}, and \eqref{eqn:evLogloss}  in the Appendix for the specific forms of these losses. Additionally, incorrect evidence regularization terms are introduced to guide the model to output low evidence for classes other than the ground truth class (See Appendix \ref{app:evIncReg} for discussion on the regularization). With evidential training, accurate evidential deep learning models are expected to output high evidence for the correct class, low evidence for all other classes, and output very high vacuity for unseen/out-of-distribution samples. 

\vspace{-2mm}\subsection{Theoretical Analysis of Learning Deficiency in Evidential Learning}\vspace{-2mm}

To identify the underlying reason that causes the performance gap of evidential models as described earlier, 
we consider a $K$ class classification problem and a representative evidential model trained using Bayes risk with sum of squares loss given in \eqref{eqn:evMSEloss}. We first provide an important definition that is critical for our theoretical analysis.  
\vspace{-1mm}
\begin{definition}[{\bf Zero-Evidence Region}] 
    A \textit{Zero-evidence sample} is a data sample for which the model outputs zero evidence for all classes. 
    A region in the evidence space that contains \textit{zero-evidence samples} is a \textit{zero-evidence region}.
\end{definition}
\vspace{-1mm}
For a reasonable evidential model, novel data samples not yet seen during training, difficult data samples, and out-of-distribution samples should become zero-evidence samples. 

\begin{theorem} 
\label{supotimalityTheorem}
{Given a training sample $(\mathbf{x}, \mathbf{y})$, if an evidential neural network outputs zero evidence $\mathbf{e}$, then the gradients of the evidential loss evaluated on this training sample over the network parameters reduce to zero. }
\vspace{-2mm}
\end{theorem}

\begin{proof}
Consider an input $\mathbf{x}$ with one-hot ground truth label $\mathbf{y}$. Let the ground truth class index be $gt$, \ie $y_{gt} =1, $ with corresponding Dirichlet parameter $\alpha_{gt}$, and $y_{\neq gt} = 0$. Moreover, let $\mathbf{o}, \mathbf{e}, \text{and } \boldsymbol{\alpha}$ represent the neural network output vector before applying the activation $\mathcal{A}$, the evidence vector, and the Dirichlet parameters respectively. 

In this evidential model, the loss is given by
\begin{align}
    \mathcal{L}^{\texttt{MSE}}(\mathbf{x}, \mathbf{y}) &= \sum_{j=1}^K (y_j - \frac{\alpha_j}{S})^2 + \frac{\alpha_j (S - \alpha_j)}{S^2(S+1)} 
\end{align}
Now, the gradient of the loss with respect to the  neural network output can be computed using the chain rule:
\begin{align}\label{eqn:gradMseInt}
\begin{split}
    &\frac{\partial \mathcal{L}^{\texttt{MSE}}(\mathbf{x}, \mathbf{y})}{\partial o_k}  = \frac{\partial \mathcal{L}^{\texttt{MSE}}(\mathbf{x}, \mathbf{y})}{\partial \alpha_k}\frac{\partial e_k}{\partial o_k} \\
    & = \bigg[ \frac{2\alpha_{gt}}{S^2} - 2\frac{y_k}{S} - \frac{2( S - \alpha_k)}{S(S+1)} +\\
    &\quad\quad\quad+\frac{2(2S + 1)\sum_{i} \sum_{j}\alpha_i \alpha_j}{(S^2 + S)^2}
    \bigg] \times \frac{\partial e_k}{\partial o_k}
    \end{split}
\end{align}
Based on the actual form of $\mathcal{A}$, we have three cases:

\textbf{Case I:} $\texttt{ReLU}(\cdot)$ to transform logits to evidence
\begin{align}\label{eqn:reluGradoE}
e_k &= \text{ReLU}(o_k) 
\implies \frac{\partial e_k}{\partial o_k} = \begin{cases}
1 \quad  \text{if} \quad \quad o_k > 0 \\
0 \quad  \text{otherwise}
\end{cases}    
\end{align}
For a zero-evidence sample, the logits $o_k$ satisfy the relationship $o_k \leq 0 \; \forall \; k
\implies \frac{\partial e_k}{\partial o_k} = 0
\implies\frac{\partial \mathcal{L}^{\texttt{MSE}}(\mathbf{x}, \mathbf{y})}{\partial o_k}  = 0
$

\textbf{Case II:} $\texttt{SoftPlus}(\cdot)$ to transform logits to evidence
\begin{align}\label{eqn:softplusGradoE}
    e_k &=  \log ( \exp(o_k) + 1) \implies &\frac{\partial e_k}{\partial o_k} = \text{Sigmoid}(o_k)
\end{align}
For a zero-evidence sample, the logits $o_k \rightarrow -\infty \implies \text{Sigmoid}(o_k) \rightarrow 0 \; \& \; \frac{\partial e_k}{\partial o_k} \rightarrow 0$.

\textbf{Case III:} $\exp(\cdot)$ to transform logits to evidence
\begin{align}\label{eqn:expGradoE}
    e_k &= \exp(o_k)
    \implies &\frac{\partial e_k}{\partial o_k} = \exp(o_k) = \alpha_k - 1
\end{align}
 For a zero-evidence sample, $\alpha_k \rightarrow 1 \implies \frac{\partial e_k}{\partial o_k} \rightarrow 0$. Moreover, there is no term in the first part of the loss gradient in \eqref{eqn:gradMseInt} to counterbalance these zero-approaching gradients. So, for \textit{zero-evidence training samples}, for any node $k$,
\begin{align}
\frac{\partial \mathcal{L}^{\texttt{MSE}}(\mathbf{x}, \mathbf{y})}{\partial o_k}  = 0
\end{align}
Since the gradient of the loss with respect to all the nodes is zero, there is no update to the model from such samples. {This implies that the evidential models fail to learn from a zero-evidence data sample.} 
\end{proof}
For completeness, we present the analysis of standard classification models in Appendix \ref{sec:appAnalysisStandardClassificationModels}, detailed proof of the evidential models trained using Bayes risk with sum of squares error along with other evidential lossses in Appendix \ref{apSec:evidentialProof}, and impact of incorrect evidence regularization in Appendix \ref{app:evIncReg}. 

\vspace{-2mm}\paragraph{Remark:} Evidential models can not learn from a training sample that the model has never seen and for which the model accurately outputs ``I don't know", \ie  $e_k=0 \quad \forall k \in [1, K]$.  Such samples are expected and likely to be present during model training. However, the supervised information in such training data points is completely missed by evidential models so they fail to acquire any new knowledge from all such training data samples (\ie data samples in zero-evidence region of the evidence space). 
\begin{corollary}
    Incorrect evidence regularization can not help evidential models learn from zero-evidence samples.  
\end{corollary}
Intuitively, the incorrect evidence regularization encourages the model to output zero evidence for all classes other than the ground truth class and the regularization does not have any impact on the evidence for the ground truth class. So, the regularization updates the model parameters such that the model is likely to map input samples closer to zero-evidence region in the evidence space. Thus, the regularization does not address the failure of evidential models to learn from zero evidence samples.
\begin{theorem}
\label{th:superirorityofExp}

{ For a data sample $\mathbf{x}$, if an evidential model outputs logits $\mathbf{o}_k \leq 0 $ $\forall k \in [0, K]$, the exponential activation function leads to a larger gradident update on the model parameters than \texttt{softplus} and \texttt{ReLu}.}
\end{theorem}
Limited by space, we present the proof of Theorem \ref{th:superirorityofExp} {along with additional analysis } in the Appendix \ref{sec:impactLogitsEvTranf}. {The proof follows the gradient analysis of the exponential, $\texttt{Softplus}$, and $\texttt{ReLU}$ based models. It implies that the the training of evidential models is most effective with the exponential activation function.} Intuitively, the $\texttt{ReLU}$ based activation completely destroys all the information in the negative logits, and has largest region in evidence space in which training data have zero evidence. $\texttt{Softplus}$ activation improves over the $\texttt{ReLU}$, and  compared to $\texttt{ReLU}$, has smaller region in evidence space where training data have zero evidence. However, \texttt{Softplus} based evidential models fail to correct the acquired knowledge when the model has strong wrong evidence. Moreover, these models are likely to suffer from vanishing gradients problem when the number of classes increases (\ie classification problem becomes more challenging). Finally, exponential activation has the smallest zero-evidence region in the evidence space without suffering from the issues of $\texttt{SoftPlus}$ based evidential models. 

\section{Avoiding Zero-Evidence Regions Through Correct Evidence Regularization}
We now consider an evidential model with exponential function to transform the logits into evidence. We propose a novel vacuity-guided correct evidence regularization term
\begin{align}\label{eqn:propCorEvReg}
    \mathcal{L}_{\texttt{cor}}({\bm x},{\bm y}) = - \lambda_{\texttt{cor}} \log (\alpha_{gt} - 1)
\end{align} 
where $\lambda_{\texttt{cor}} = \nu=\frac{K}{S}$ represents the regularization term whose value is given by the magnitude of the vacuity output by the evidential model and $\alpha_{gt}-1$ represents the predicted evidence for the ground truth class. The regularization term $\lambda_{\texttt{cor}}$ determines the relative importance of the correct evidence regularization term compared to the evidential loss and incorrect evidence regularization and is treated as constant during model parameter update.
\begin{theorem}
\label{th:sovlingzeroevidenceIssue}
Correct evidence regularization  $\mathcal{L}_{\texttt{cor}}({\bm x},{\bm y})$ can address the issue of learning from zero-evidence training samples. 
\end{theorem}
\begin{proof}
    The proposed regularization term $\mathcal{L}_{\texttt{cor}}({\bm x},{\bm y})$ does not contain any evidence terms other than the evidence for the ground truth node. So, the gradient of the regularization for nodes other than the ground truth node will be 0 i.e. $\frac{\partial \mathcal{L}_{\texttt{cor}}({\bm x},{\bm y})}{\partial o_k}\Big|_{k \neq gt} = 0$ and there will be no update on these nodes. 
    For the ground truth node $gt, y_{gt} = 1$, the gradient is given by
    \begin{align} 
        \frac{ \partial \mathcal{L}_{\texttt{cor}}({\bm x},{\bm y}) }{\partial o_{{gt}}} &=  \frac{ \partial \big(- \lambda_{\texttt{cor}} \log (\alpha_{gt} - 1) \big) }{\partial o_{{gt}}} \\
        &= -\lambda_{\texttt{cor}} \frac{ \partial \log (\alpha_{{gt}} - 1) }{\partial \alpha_{{gt}}} \times \frac{ \partial\alpha_{{gt}} }{\partial o_{{gt}}}\\
        &= - \frac{\lambda_{\texttt{cor}} }{(\alpha_{{gt}} - 1)}  (\alpha_{{gt}} - 1) = - \lambda_{\texttt{cor}}
    \end{align}

The gradient value equals the magnitude of the vacuity. The vacuity is bounded in the range $[0,1]$, and \textit{zero-evidence sample}, the vacuity is maximum, leading to the greatest gradient value of $\frac{ \partial \mathcal{L}_{\texttt{cor}}({\bm x},{\bm y}) }{\partial o_{{gt}}} = -1$. In other words, the regularization encourages the model to update the parameters such that the correct evidence $\alpha_{gt} - 1$ increases. As the model evidence increases, the vacuity decreases, and the contribution of the regularization $\mathcal{L}_{\texttt{cor}}({\bm x},{\bm y})$ is minimized. Thus, the proposed regularization enables the evidential model to learn from \textit{zero-evidence samples}.
\end{proof}

\vspace{-2mm}\subsection{Evidential Model Training}\label{sec:evModelTrainingLoss}\vspace{-2mm}
We formulate an overall objective used to train the proposed \textbf{R}egularized \textbf{e}vidential mo\textbf{d}el (\textbf{RED}). Essentially, the evidential model is trained to maximize the correct evidence, minimize the incorrect evidence, and avoid the \textit{zero-evidence} region during training. The overall loss is 
\begin{align}\label{eqn:proposedEvidentialModelOverallLoss}
    \mathcal{L}(\bm{x},\bm{y}) = \mathcal{L}^{\texttt{evid}}(\bm{x},\bm{y}) + \eta_1 \mathcal{L}^{\texttt{inc}}(\bm{x},\bm{y}) + \mathcal{L}^{\texttt{cor}}(\bm{x},\bm{y})
\end{align}
where $\mathcal{L}^{\texttt{evid}}(\bm{x},\bm{y})$ is the loss based on the evidential framework given by  \eqref{eqn:evMSEloss}, \eqref{eqn:evLogloss}, or \eqref{eqn:evDigammaloss} (See Appendix~\ref{apSec:evidentialProof}), $\mathcal{L}^{\texttt{inc}}(\bm{x},\bm{y})$ represents the incorrect evidence regularization (See Appendix Section \ref{app:evIncReg}), $ \mathcal{L}^{\texttt{cor}}(\bm{x},\bm{y})$ represents the proposed novel correct evidence regularization term in \eqref{eqn:propCorEvReg}, and $\eta_1 = \lambda_1 \times \min(1.0, \text{epoch index}/10) $ controls the impact of incorrect evidence regularization to the overall model training. In this work, we consider the forward-KL based incorrect evidence regularization  given in \eqref{appeq:klsensoy} based on \cite{sensoy2018evidential}.

\vspace{-2mm}\subsection{Evidence Space Visualization}\label{sec:intuitiveDescription}\vspace{-2mm}
\begin{figure}[ht!] 
\centering
  \includegraphics[width=0.90\linewidth]{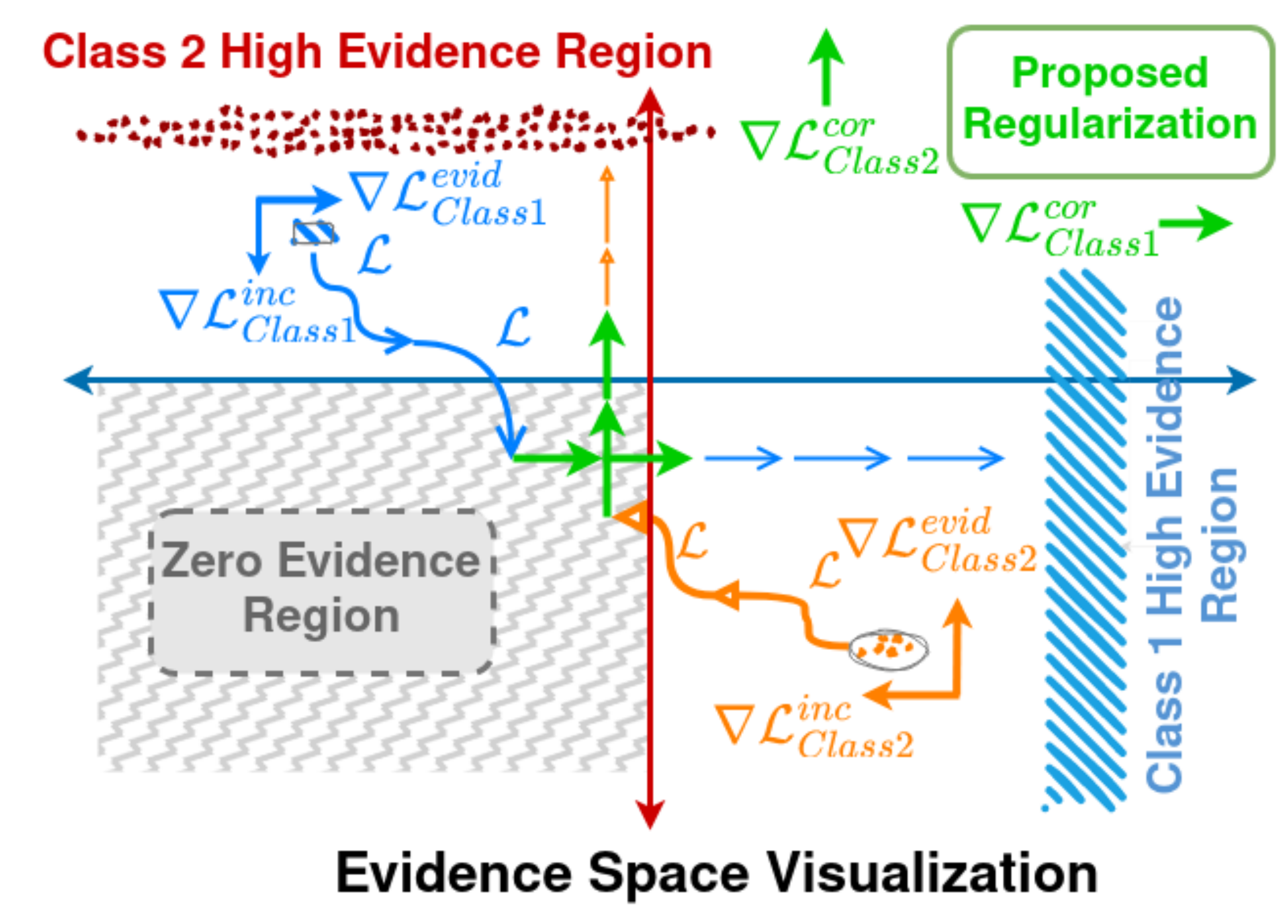}
\caption{\label{fig:fixZeroEvidenceIssue}Evidence space visualization to demonstrate the effectiveness of the proposed method. }
\vspace{-4mm}
\end{figure}

Figure \ref{fig:intuitiveFailureOfRelu} visualizes the evidence space in $\texttt{ReLU}$-based evidential models by considering the pre-ReLU output in a binary classification setting. Ideally, all samples that belong to Class 1 should be mapped to the blue region (region of high evidence for Class 1, low evidence for all other classes), all samples that belong to Class 2 should be mapped to the red region, and all out-of distribution samples should be mapped to the zero-evidence region (no evidence for all classes). To realize this goal, the models are trained using the evidential loss $\mathcal{L}^{evid}$ with incorrect evidence regularization $\mathcal{L}^{inc}$. However,  there is no update to the evidential model from such samples of \textit{zero-evidence region}. Model's prior belief of ``I don't know" for such samples does not get updated even after being exposed to the true label. For the samples with high incorrect evidence and low correct evidence, evidential model aims to correct itself. However, many such samples are likely to get mapped to the zero-evidence region (as shown by blue and orange arrows in Figure \ref{fig:intuitiveFailureOfRelu}) after which there is no update to the model. Such fundamental limitation holds true for all evidential models.

The evidence space visualization for RED is shown in Figure \ref{fig:fixZeroEvidenceIssue} to illustrate how it addresses the above limitation. Correct evidence regularization (indicated by green arrows) is weighted by the magnitude of the vacuity and is maximum in the zero-evidence region. In this problematic region, the proposed regularization fully dominates the model update as there is no update to the model from the two loss components ($\mathcal{L}^{evid}$ and $\mathcal{L}^{inc}$) in \eqref{eqn:proposedEvidentialModelOverallLoss}. As the sample gets far away from the zero evidence region, the vacuity decreases proportionally, the impact of the proposed regularization to model update becomes insignificant, and the evidential losses ($\mathcal{L}^{evid}$ \& $\mathcal{L}^{inc}$) guide the model training. In this way, RED can effectively learn from all training samples irrespective of the model's existing evidence.  

\vspace{-2mm}\section{Experiments}\vspace{-2mm}
\paragraph{Datasets and setup.} We consider the standard supervised classification problem with MNIST \cite{lecun1998mnist}, Cifar10, and Cifar100 datasets \cite{krizhevsky2009learning}, and few-shot classification with \textit{mini}-ImageNet dataset \cite{vinyals2016matching}. We employ the LeNet model for MNIST, ResNet18 model \cite{he2016deep} for Cifar10/Cifar100, and ResNet12 model \cite{he2016deep} for \textit{mini}-ImageNet. We first conduct experiments to demonstrate the learning deficiency of existing evidential models to confirm our theoretical findings. We then evaluate the proposed correct evidence regularization to show its effectiveness. We finally conduct ablation studies to investigate the impact of evidential losses on model generalization and the uncertainty quantification of the proposed evidential model. {Limited by space, additional clarifications, experiment results including few-shot classification experiments, experiments over challenging tiny-Imagenet datasett with Swin Transformer, }  hyperparameter details, and discussions are presented in the Appendix. 

\vspace{-2mm}\subsection{Learning Deficiency of Evidential Models}

\paragraph{Sensitivity to the change of the architecture.} 
We first consider a toy illustrative experiment with  two frameworks: 1) standard softmax, 2) evidential learning, and experiment with the LeNet \cite{lecun1999object} model considered in EDL~\cite{sensoy2018evidential} with a minor modification to the architecture: no dropout in the model. To construct the toy dataset, we randomly select 4 labeled data points from the MNIST training dataset as shown in the Figure \ref{fig:toyDatasetMnist}. For the evidential model, we use ReLU to transform the network outputs to evidence, and train the model with MSE-based evidential loss~\cite{sensoy2018evidential} given in \eqref{eqn:evMSEloss} without incorrect evidence regularization. We train both models using only these 4 training data points. 

Figure \ref{fig:compToyMnistEVAccLossTrend} compares the training accuracy and training loss trends of the evidential model with the standard softmax model (trained with the cross-entropy loss). Before any training, both models have $0\%$ accuracy and the loss is high as expected. For the evidential model, in the first few iterations, the model learns from the training dataset, and the model's accuracy increases to $50\%$. Afterward, the  evidential model fails to learn as the evidential model maps two of the training data samples to the \textit{zero-evidence region}. Even in such a trivial setting, the evidential model fails to fit the 4 training data points showing their learning deficiency that empirically verifies the conclusion in Theorem \ref{supotimalityTheorem}. It is also worth noting that the range of the evidential model's loss is significantly smaller than the standard model. This is mainly due to the bounded nature of the evidential MSE loss(\ie it is bounded in the range $[0,2]$) (a detailed theoretical analysis of the evidential losses is provided in the Appendix). In contrast, the standard model trained with cross-entropy loss easily fits the trivial dataset, obtains near $0$ loss, and perfect accuracy of $100\%$ after a few iterations of training. 

\begin{figure}[!t]
    \centering
    \includegraphics[width=.96\linewidth]{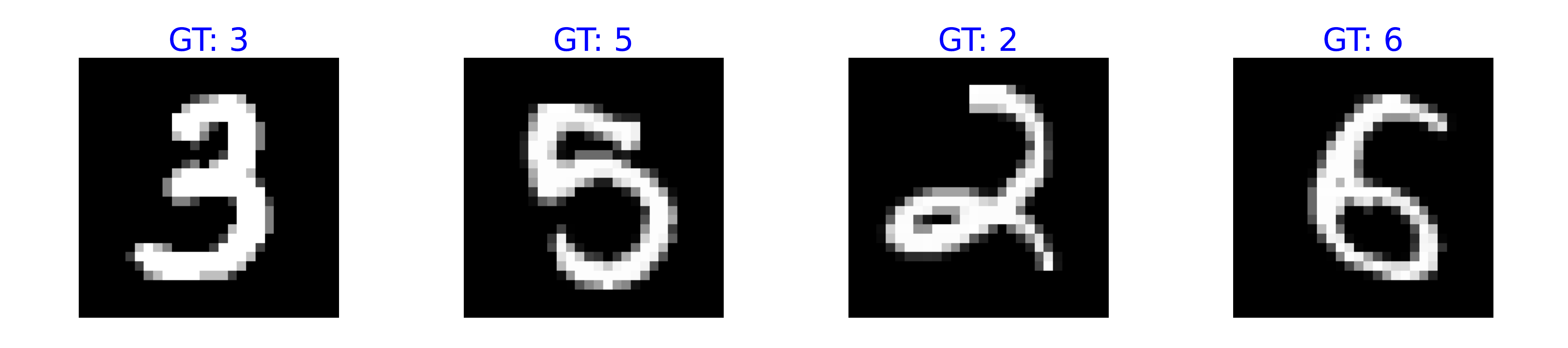}  
    \vspace{-2mm}
    \caption{Toy dataset with 4 data points. } 
    \label{fig:toyDatasetMnist}
    \vspace{-5mm}
\end{figure}

\begin{figure}
\centering
\subfigure[Training accuracy trend]{
  \includegraphics[width=0.46\linewidth]{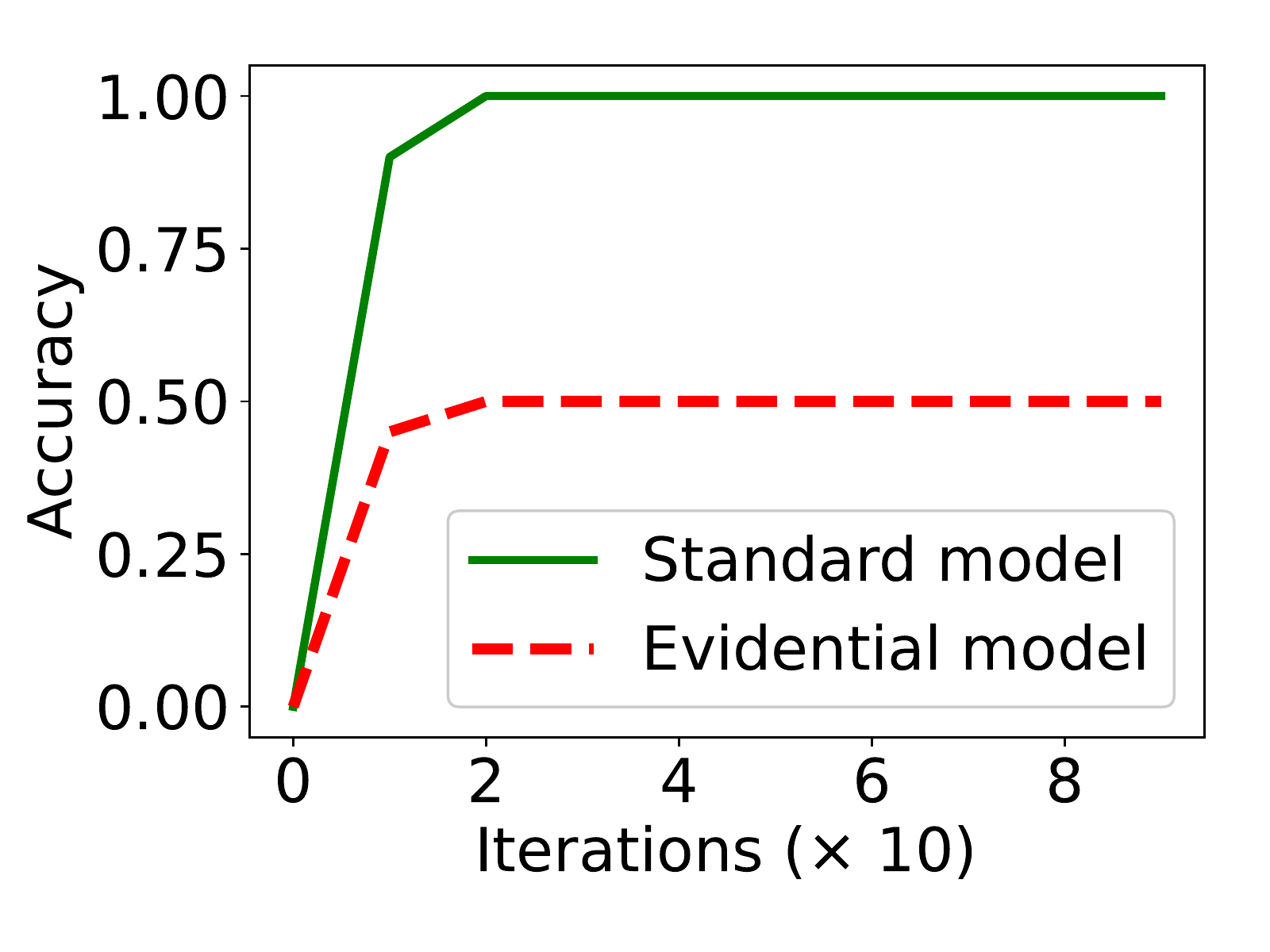} 
}
\subfigure[Training loss trend]{
  \includegraphics[width=0.46\linewidth]{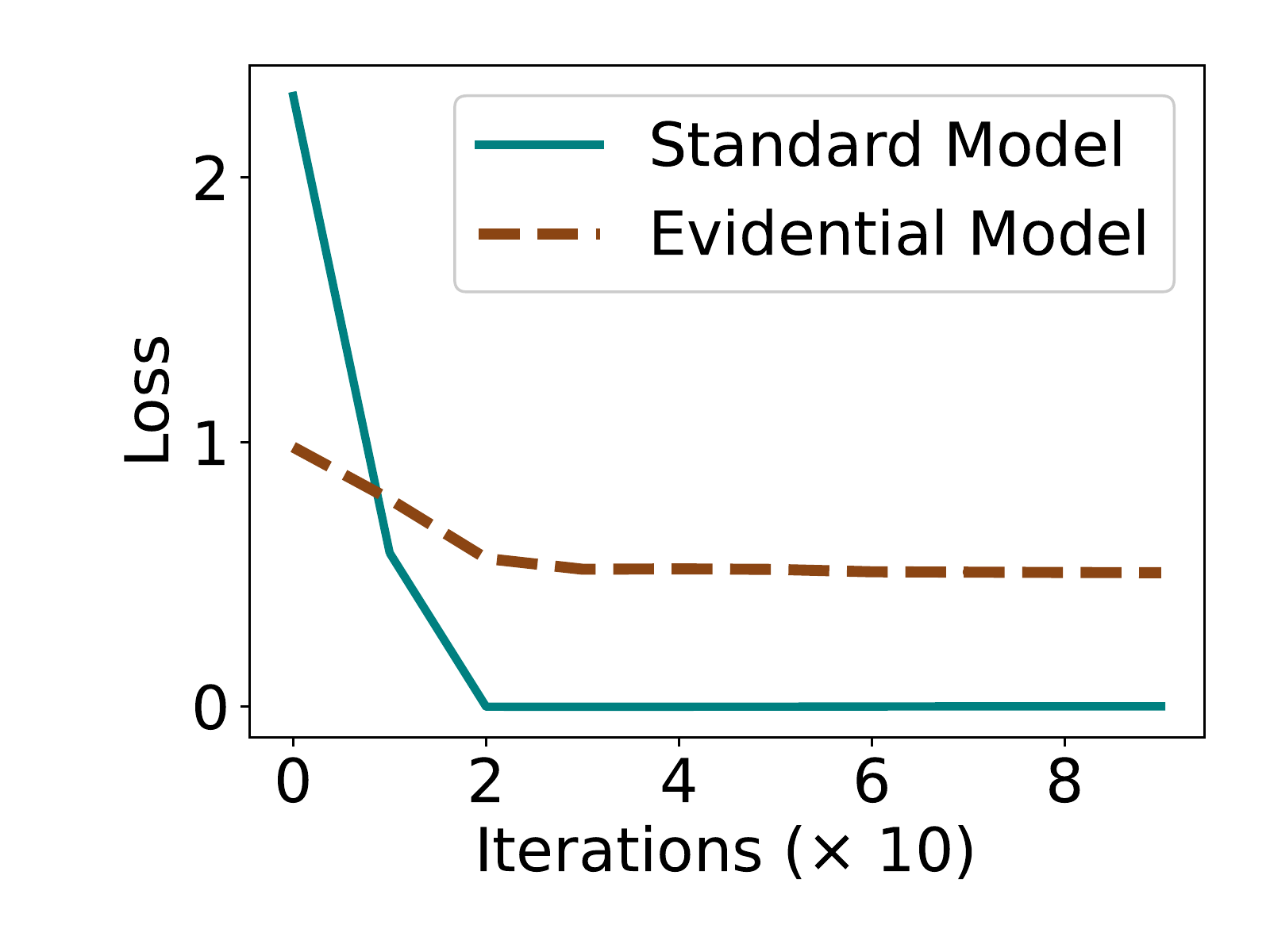} 
}
\vspace{-2mm}
\caption{Training of standard and evidential models } 
\label{fig:compToyMnistEVAccLossTrend}
\vspace{-5mm}
\end{figure} 

\begin{figure}[htpb]
    \centering
    \includegraphics[width=.8\linewidth]{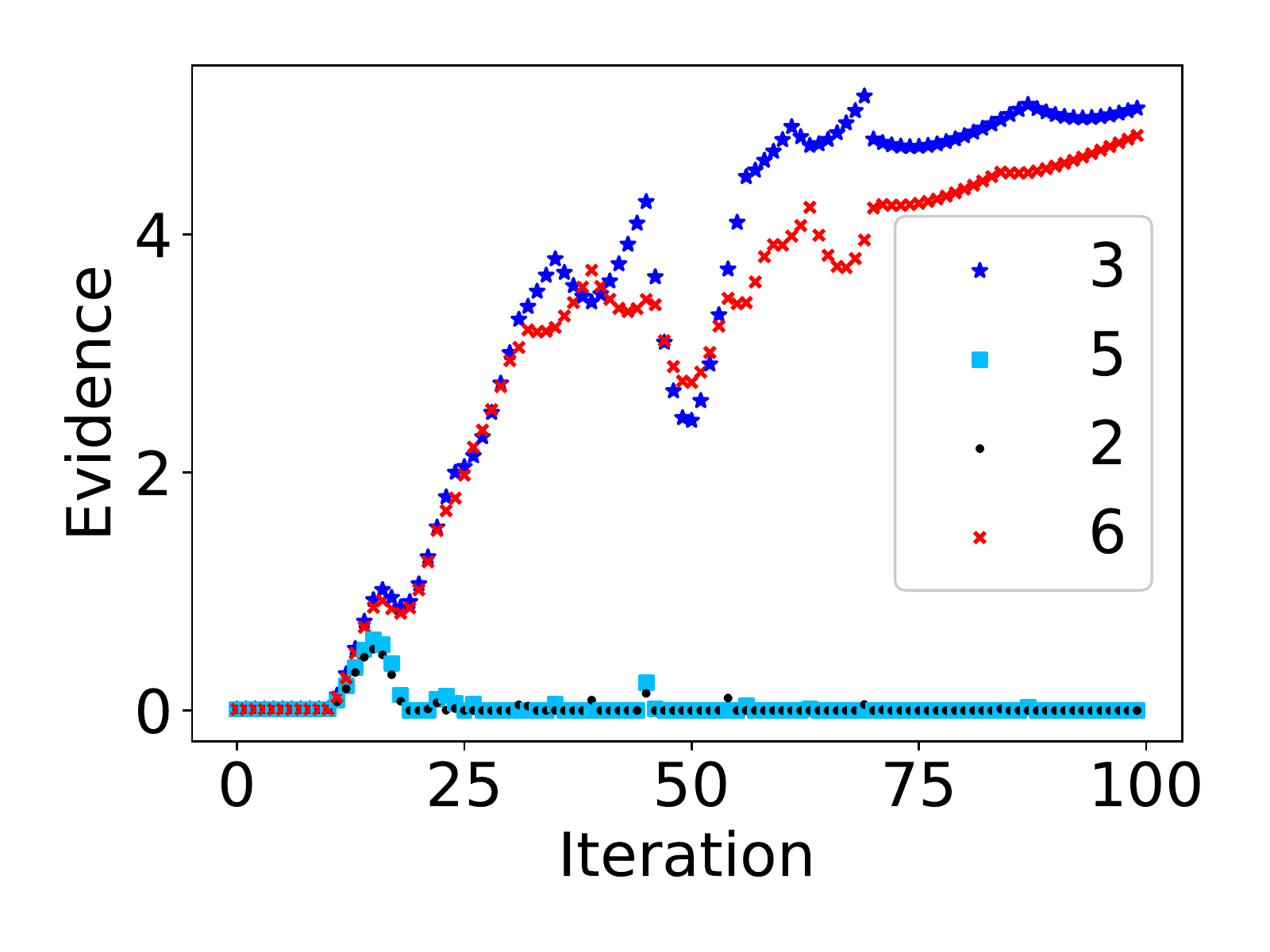}
    \vspace{-6mm}
    \caption{Zero-evidence trend during model training} 
    \label{fig:zeroEvidenceTrendVisualization}
    \vspace{-2mm}
\end{figure}

{Additionally, we visualize the zero-evidence data samples for the toy dataset setting. 
We plot the total evidence for each training sample as training progresses for the first 100 iterations. 
The total evidence trend as training progresses for the first 100 iterations is shown in Figure \ref{fig:zeroEvidenceTrendVisualization}. The evidential model’s predictions are correct for data samples with ground truth labels of 3 and 6, and incorrect for the remaining two data samples. After few iterations of training, the remaining two samples
have zero total evidence (i.e. samples are mapped to zero evidence region), the model never learns from them, and the model only achieves overall 50\% training accuracy even after 100 iterations. Clearly, the evidential model continues to output zero evidence for two of the training examples and fails to learn from them. Such learning deficiency of evidential models limits their extension to challenging settings. In contrast, the standard model easily overfits the 4 training examples and achieves 100\% accuracy.}


\begin{figure}[t!] 
\centering
  \includegraphics[width=0.88\linewidth]{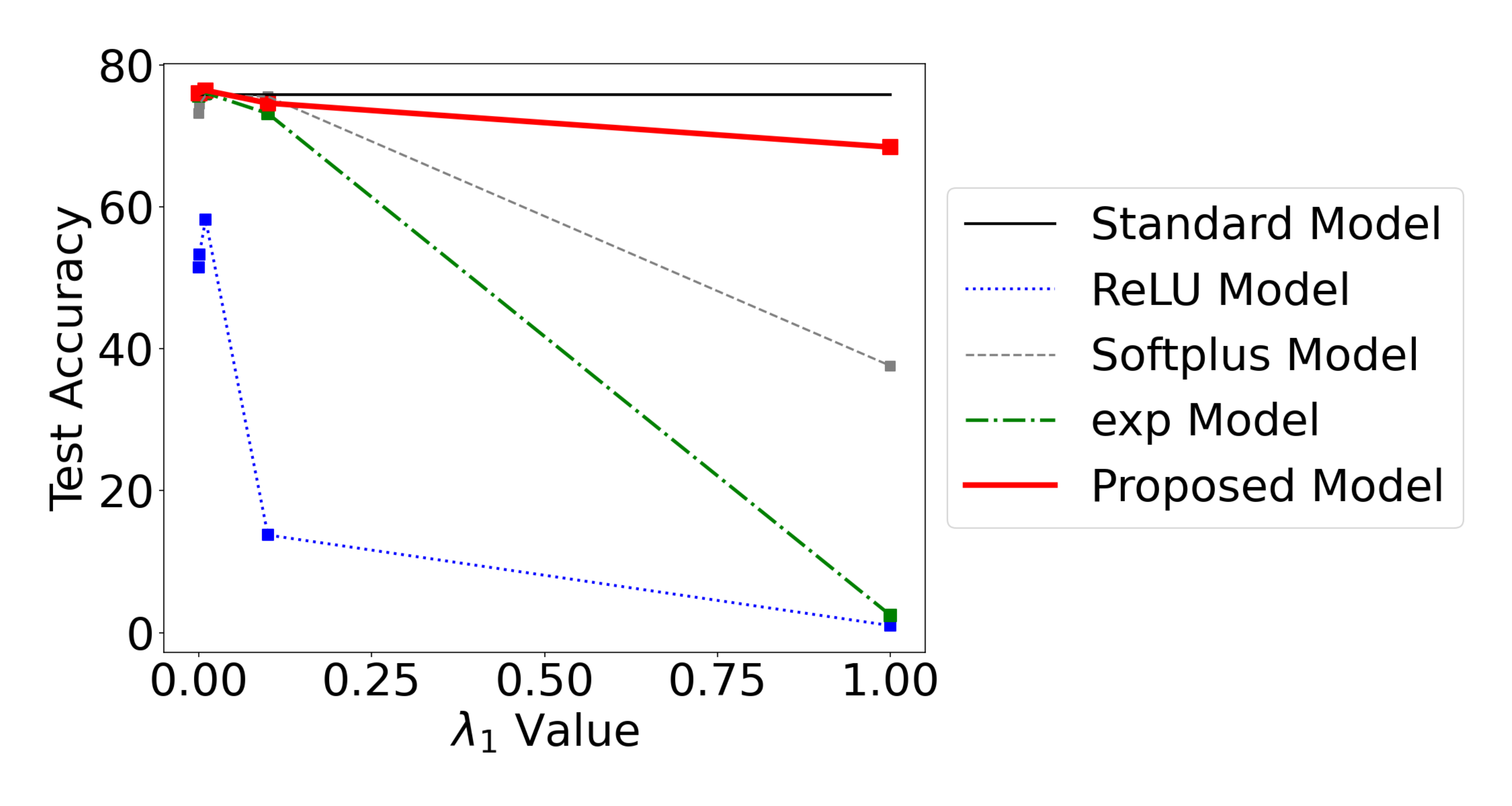}
  \vspace{-4mm}
\caption{\label{fig:IncorrectEvRegImpact}Impact of different incorrect evidence regularization strengths to the test set accuracy on Cifar100 dataset}
\vspace{-5mm}
\end{figure} 

\vspace{-2mm}\paragraph{Sensitivity to hyperparameter tuning.}
In this experiment, evidential models are trained using evidential losses given in \eqref{eqn:evMSEloss}, \eqref{eqn:evDigammaloss}, or \eqref{eqn:evLogloss}  with incorrect evidence regularization to guide the model for accurate uncertainty quantification. We study the impact of the incorrect evidence regularization $\lambda_1$ to the evidential model's performance using Cifar100. The result shows that the generalization performance of evidential models is highly sensitive to $\lambda_1$ values. To illustrate, we consider the Type II Maximum Likelihood loss in \eqref{eqn:evLogloss} with different $\lambda_1$ to control KL regularization (results on other loss functions are presented in the Appendix).  As shown in Figure \ref{fig:IncorrectEvRegImpact}, when some regularization is introduced, evidential model's test performance improves slightly. However, when strong regularization is used, the model focuses strongly on minimizing the incorrect evidence. Such regularization causes the model to push many training samples into or close to the zero-evidence regions, which hurts the model's learning capabilities. In contrast, the proposed model can continue to learn from samples in zero-evidence regions,  which shows its robustness to incorrect evidence regularization. Moreover, our model has stable performance across all hyperparameter settings as it can effectively learn from all training samples. 

\vspace{-2mm}\paragraph{Challenging datasets and settings.}
We next consider standard classification models for the Cifar100 dataset and 1-shot classification with the \textit{mini}-ImageNet dataset. We develop evidential extensions of the classification models using Type II Maximum Likelihood loss given in \eqref{eqn:evLogloss} without any incorrect evidence regularization and use \texttt{ReLU} to transform logits to evidence. As shown in Figure \ref{fig:ChallengingFailureEvidential}, compared to the standard classification model, the evidential model's predictive performance is sub-optimal (almost $20\%$ lower for both classification problems). This is mainly due to the fact that evidential model maps many of the training data points to $\textit{zero-evidence region}$, which is equivalent to the model saying ``I don't know to which class this sample belongs" and stopping to learn from them. Consequently,  the model fails to acquire new knowledge (\ie update itself), even after being exposed to correct supervision (the label information). In these cases, instead of learning, the evidential model chooses to ignore the training data on which it does not have any evidence and remains to be ignorant. 

\vspace{-2mm}\paragraph{Visualization of zero-evidence samples.} {We next show the 2-dimensional visualization of the latent representation for the randomly selected 500 training examples based on the tSNE plot for ReLU based evidential model trained on the Cifar100 dataset with $\lambda_1 = 0.1$. Figure \ref{fig:tsneChallenge} plot visualizes the
latent embedding of zero evidence (Zero E) training samples with non-zero evidence (Non-Zero E) training samples. As can be seen, both zero and non-zero evidence samples appear to be dispersed, overlap at different regions, and cover a large area in the embedding space. This further confirms the challenge of effectively learning from these samples}
\begin{figure}[!t]
    \centering
    \includegraphics[width=.76\linewidth]{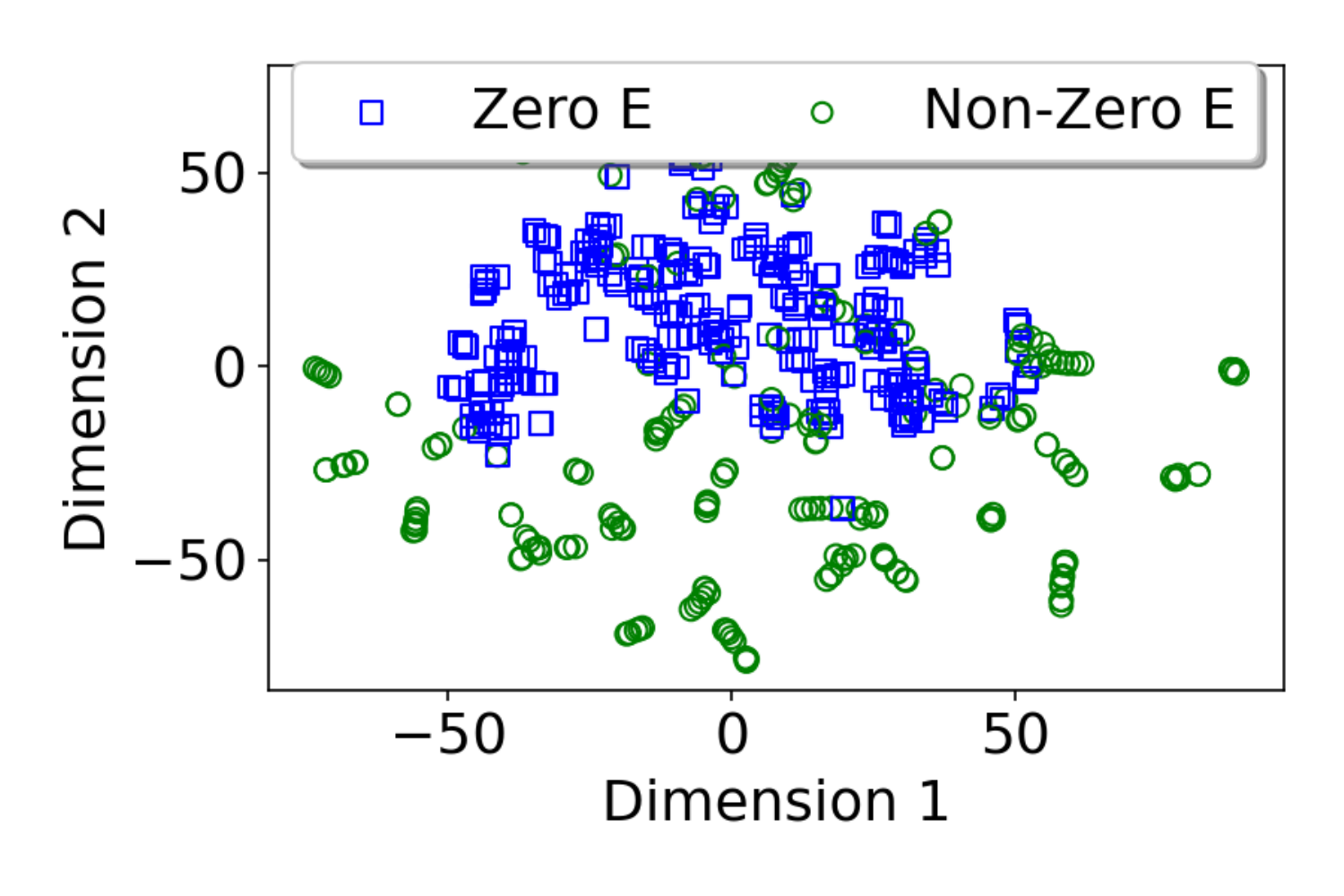}
    \vspace{-6mm}
    \caption{Zero-Evidence Sample Visualization}
    \label{fig:tsneChallenge}
    \vspace{-2mm}
\end{figure}

\begin{figure}[t!] 
\centering
\subfigure[Cifar100 Results]{
  \includegraphics[width=0.46\linewidth]{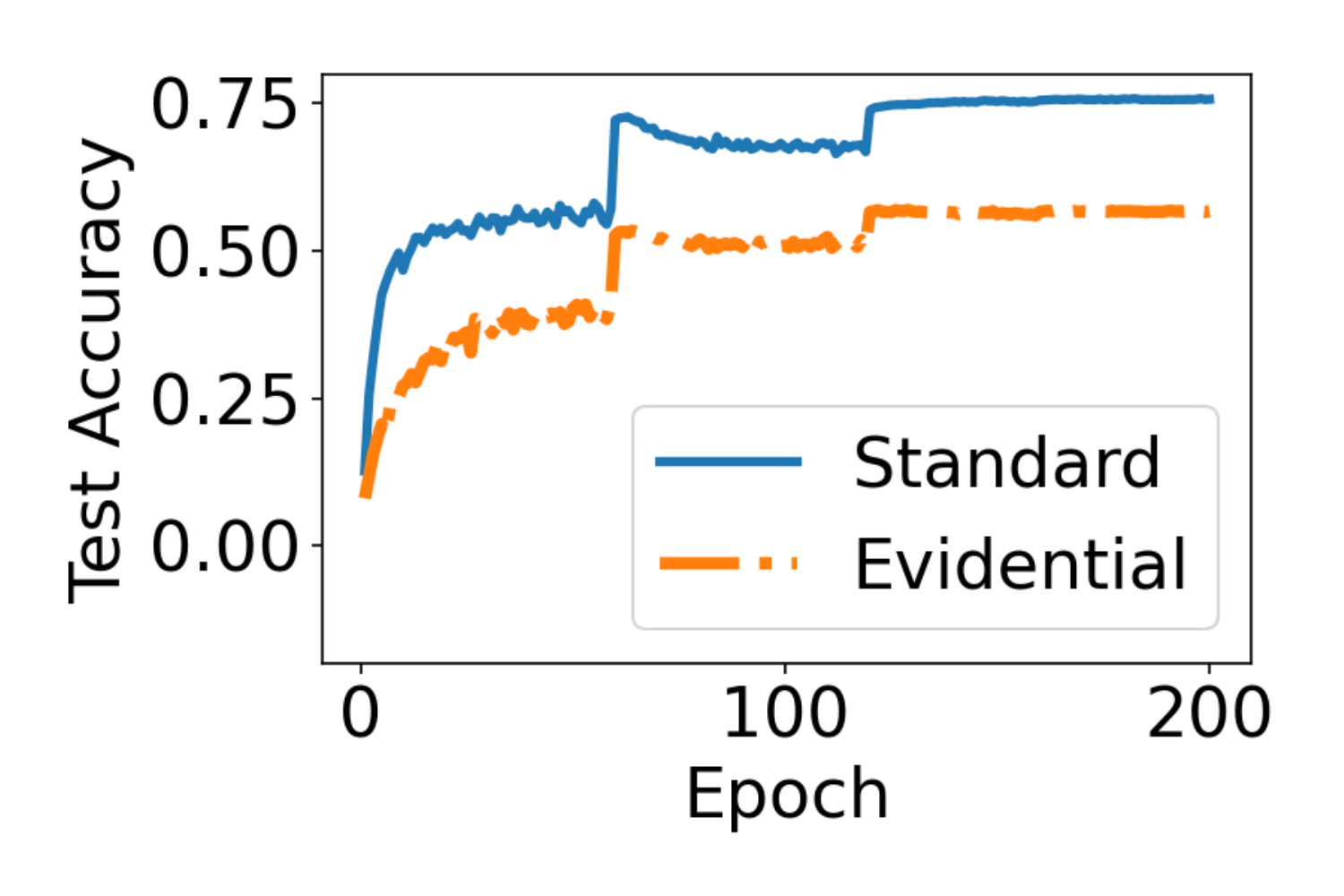}
}
\subfigure[$1$-Shot Results]{
  \includegraphics[width=0.46\linewidth]{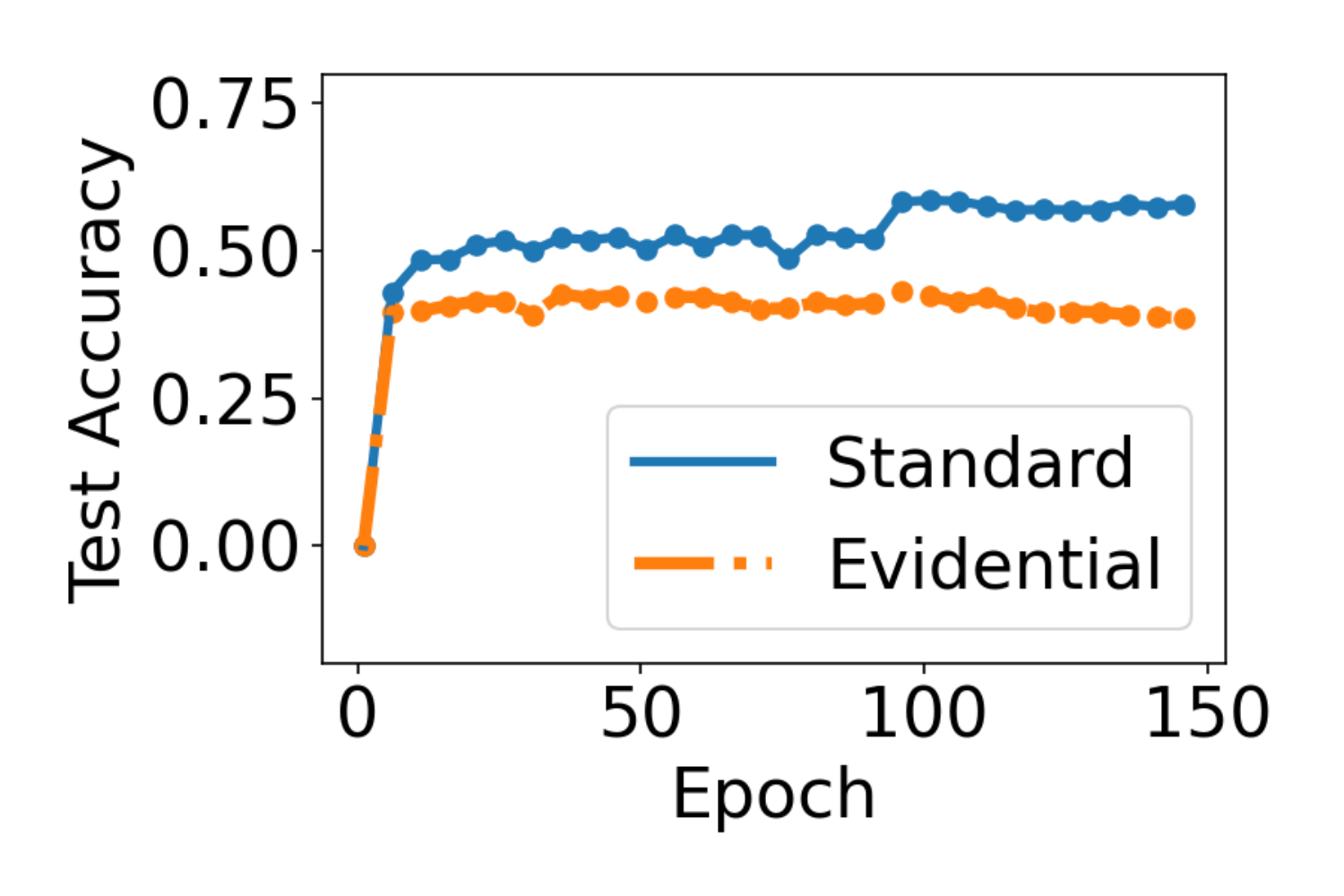}
}
\vspace{-3mm}
\caption{\label{fig:challengingDatasetsWeakness}Learning trends in complex classification problems }
\vspace{-3mm}
\label{fig:ChallengingFailureEvidential}
\end{figure} 

\subsection{Effectiveness of the RED}
\paragraph{Evidential activation function.}
\label{subsec:impactActivationFunction}
We first experiment with different activation functions for the evidential models to show the superior predictive performance and generalization capability of $\exp$ activation validating our Theorem \ref{th:superirorityofExp}. 
We consider evidential models trained with evidential log loss given by \eqref{eqn:evLogloss} in Table \ref{tab:ClassificationPerformanceCOmparison} (Additional results along with hyperparameter details are presented in Appendix Section \ref{ap:AdditionalExpResults}). As can be seen, $\exp$ activation to transform network outputs into evidence leads to superior performance compared to $\texttt{ReLU}$ and $\texttt{Softplus}$ based transformations. 
Furthermore, our proposed model with correct evidence regularization further improves over the $\exp$-based evidential models as it enables the evidential model to continue learning from \textit{zero-evidence} samples. 

\begin{table}[ht]
\vspace{-2mm}
\centering
\small
    \caption{Classification performance comparison 
    }
    \label{tab:impactEvidentialActivation}
\begin{tabular}{|p{0.09\textwidth}|p{0.09\textwidth}|p{0.09\textwidth}|p{0.09\textwidth}|}
\hline
Model & MNIST & Cifar10 &Cifar100 \\ 
\hline
\texttt{ReLU}     &$98.19_{\pm0.08}$&$41.43_{\pm19.60}$&$61.27_{\pm3.79}$\\
\texttt{SoftPlus} &$98.21_{\pm0.05}$&$95.18_{\pm0.11}$&$74.48_{\pm0.17}$\\
$\exp$            &$98.79_{\pm0.02}$&$95.11_{\pm0.10}$&$76.12_{\pm0.04}$\\
\textbf{RED(Ours)}     &$\bf{99.10_{\pm0.02}}$&$\bf{95.24_{\pm0.06}}$&$\bf{76.43_{\pm0.21}}$ \\
\hline
\end{tabular}
\label{tab:ClassificationPerformanceCOmparison}
\end{table} 

We next present the test set performance change as training progresses with MNIST dataset and two different evidential losses in Figure \ref{fig:new_Evid_actEvModel} where we observe similar results. The $\exp$ activation shows superior performance, as it has smallest \textit{zero-evidence region}, and does not suffer from many learning issues present in other activation functions.

\begin{figure}[ht!] 
\vspace{-2mm}
\centering
\subfigure[Evidential MSE loss]{
  \includegraphics[width=0.46\linewidth]{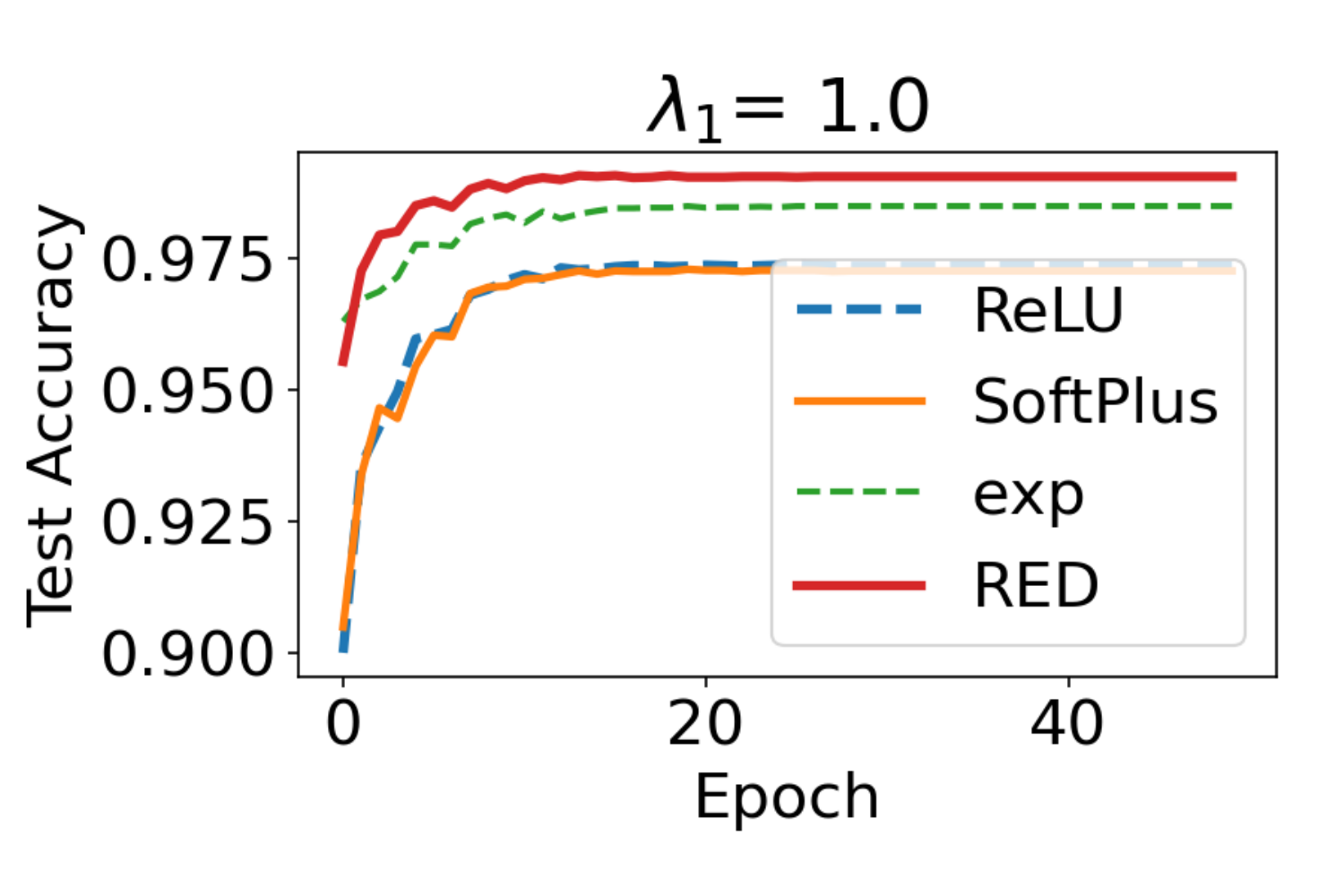}
}
\subfigure[Evidential Log loss]{
  \includegraphics[width=0.46\linewidth]{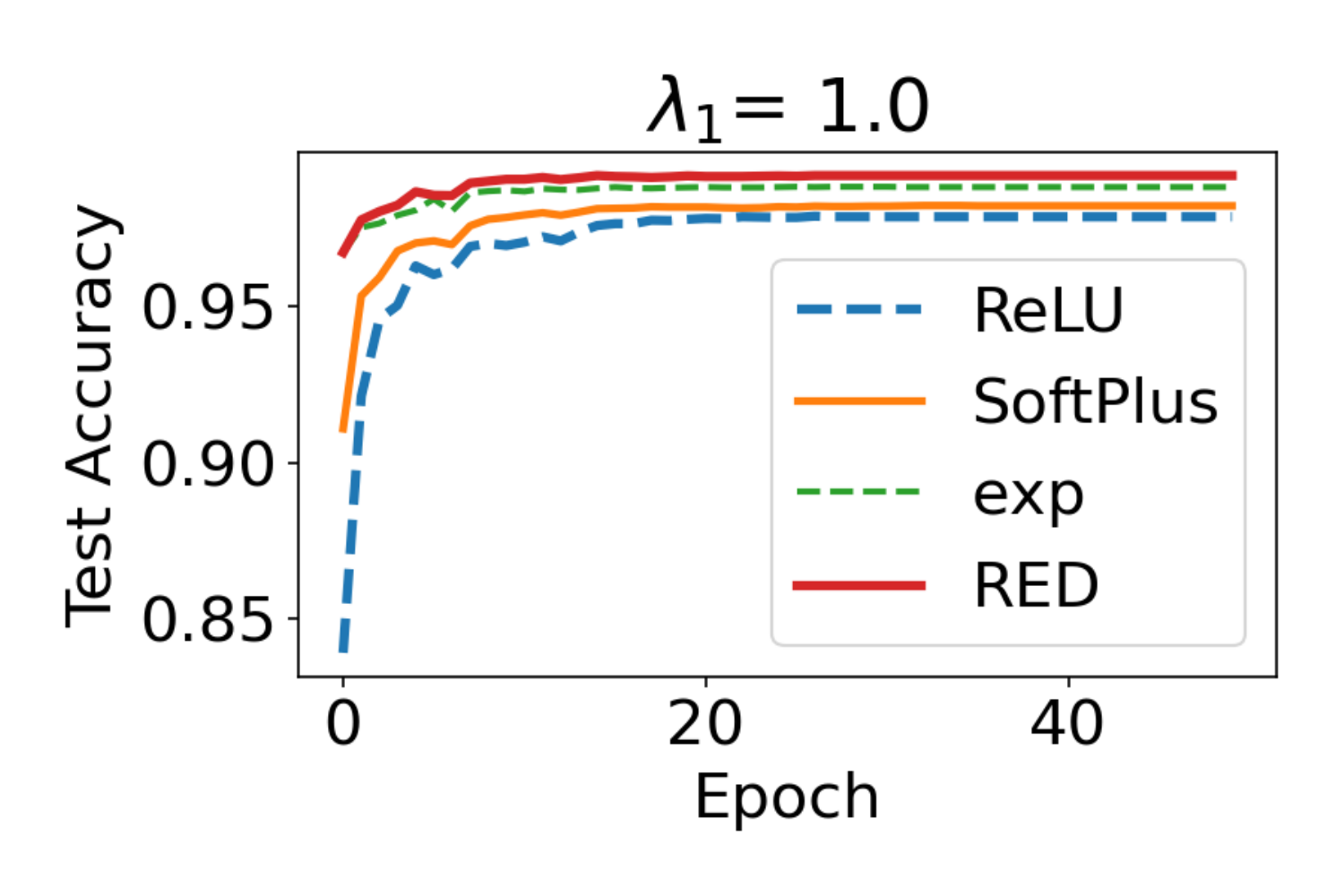}
}
\caption{\label{fig:new_Evid_actEvModel}Impact of evidential activation functions to the Test Accuracy }
\vspace{-5mm}
\end{figure}

\paragraph{Correct evidence regularization.}
We now study the impact of the proposed correct evidence regularization using the MNIST and Cifar100 classification problems. We consider the evidential baseline model that uses $\exp$ activation to acquire evidence, and is trained with Type II Maximum Likelihood based loss with different incorrect evidence regularization strengths. We introduce the proposed novel correct evidence regularization to the model. As can be seen in Figure \ref{fig:CorEvRegImpact}, the model with correct-evidence regularization has superior generalization performance compared to the baseline evidential model. This is mainly due to the fact that with proposed correct evidence regularization, the evidential model can also learn from the zero-evidence training samples to acquire new knowledge instead of ignoring them. Our proposed model considers knowledge from all the training data and aims to acquire new knowledge to improve its generalization instead of ignoring the samples on which it has no knowledge. Finally, even though strong incorrect evidence regularization hurts the model's generalization, the proposed model is robust and generalizes better, empirically validating our Theorem \ref{th:sovlingzeroevidenceIssue}. Limited by space, we present additional results in Appendix \ref{app:secImpactOfCorrectEvReg}.
\begin{figure}[ht!] 
\vspace{-3mm}
\centering
\subfigure[Trend for $\lambda_1 = 1.0$]{
  \includegraphics[width=0.46\linewidth]{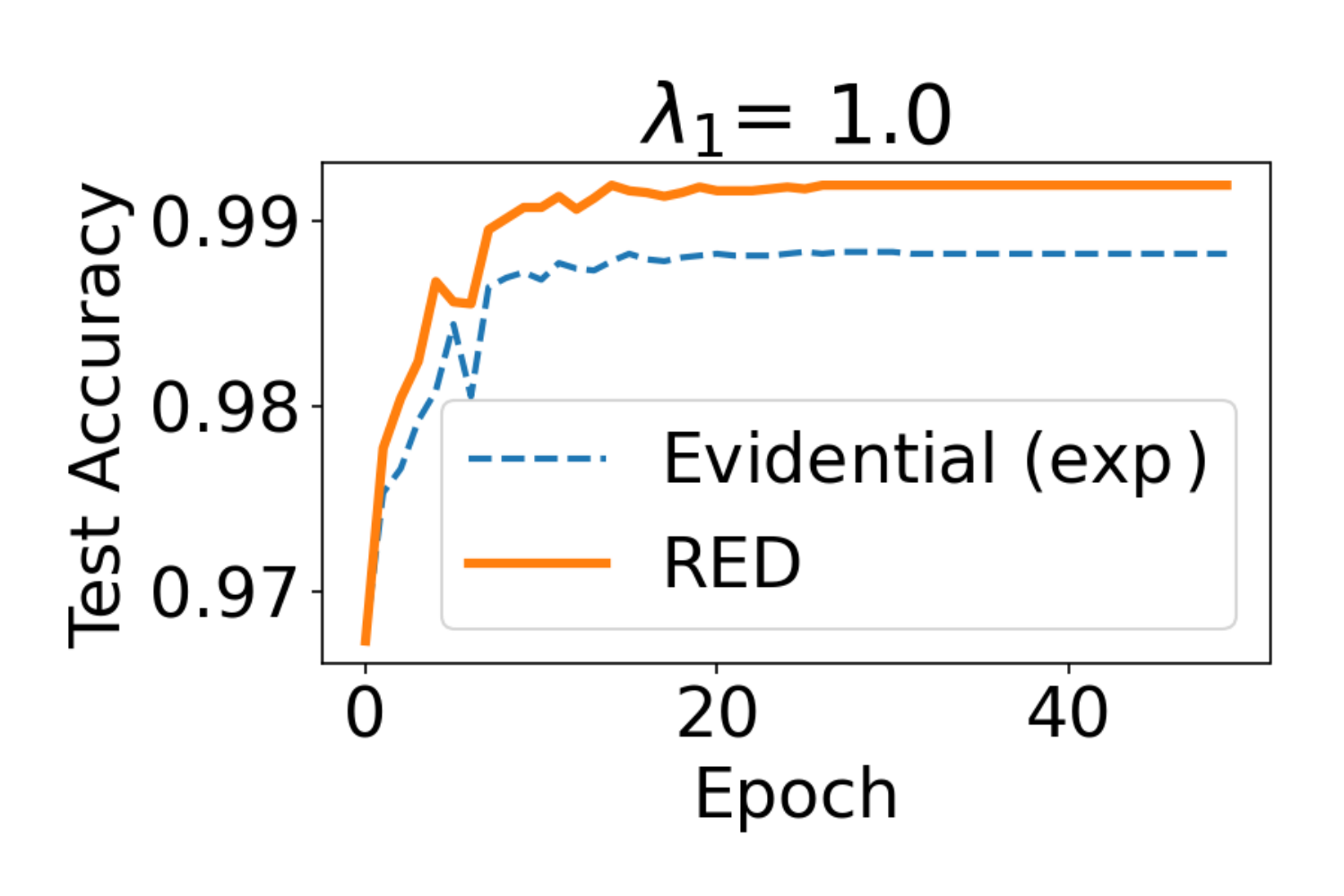}
}
\subfigure[Trend for $\lambda_1 = 10.0$]{
  \includegraphics[width=0.46\linewidth]{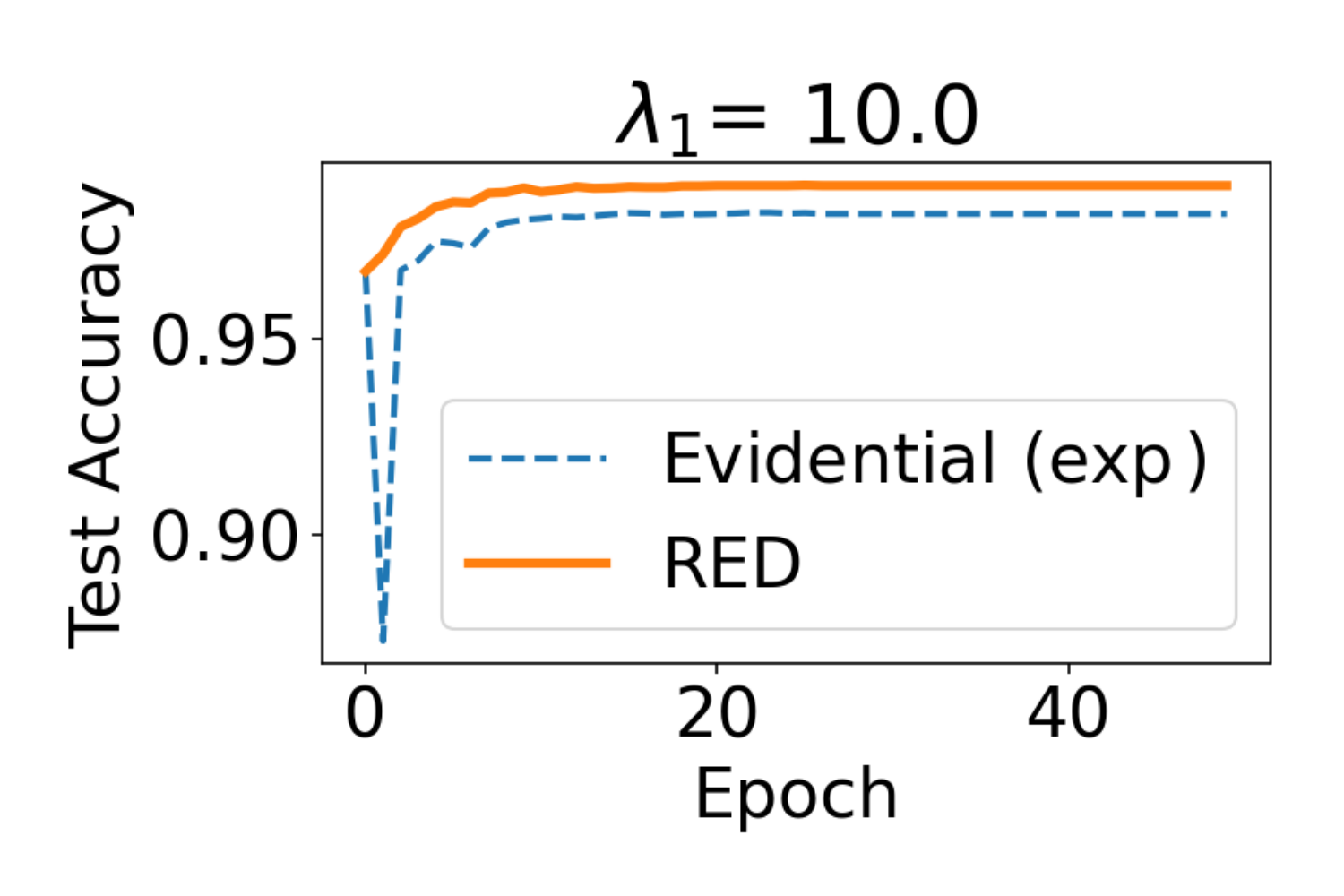}
}
\subfigure[Trend for $\lambda_1 = 0.1$]{
  \includegraphics[width=0.46\linewidth]{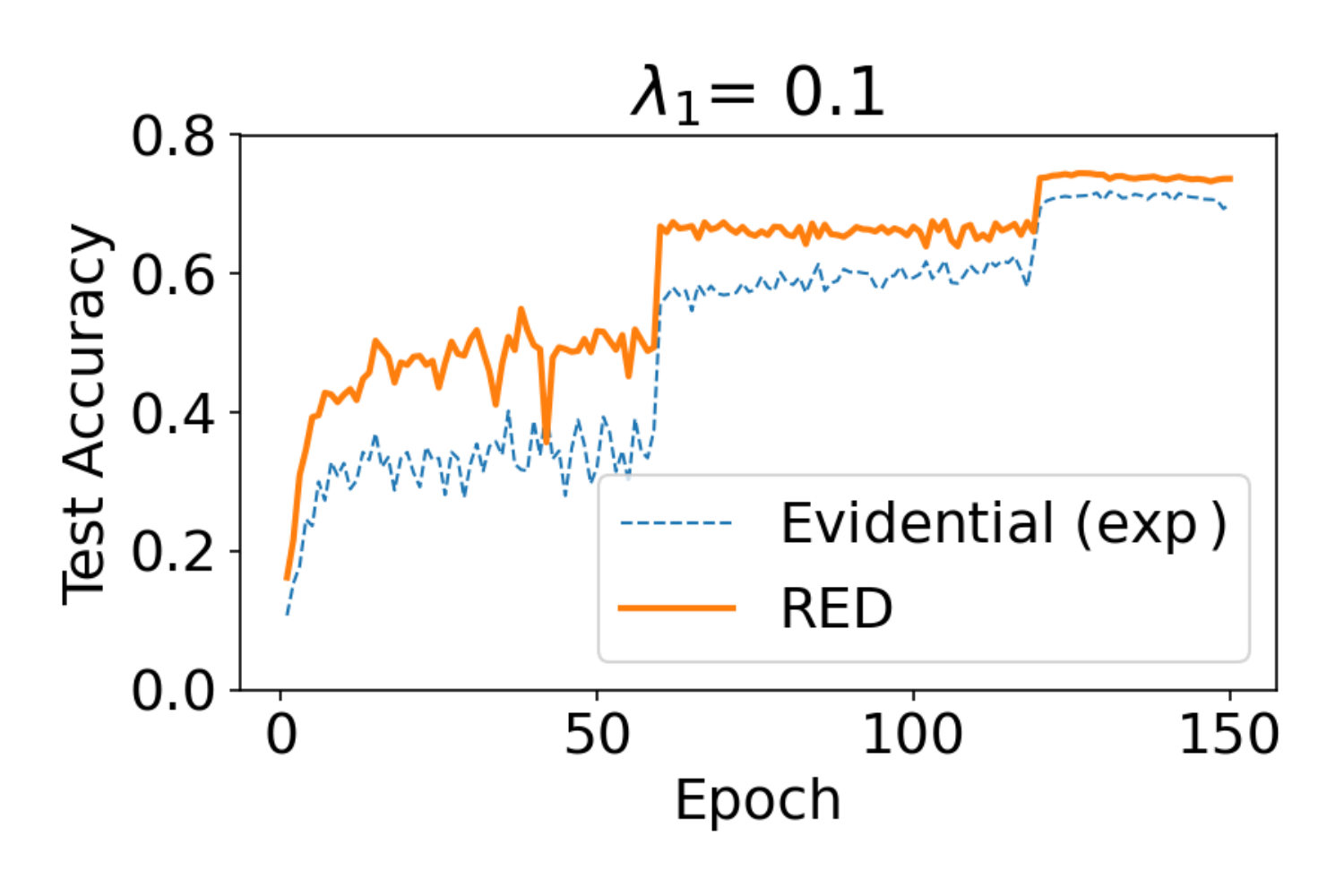}
}
\subfigure[Trend for $\lambda_1 = 1.0$]{
  \includegraphics[width=0.46\linewidth]{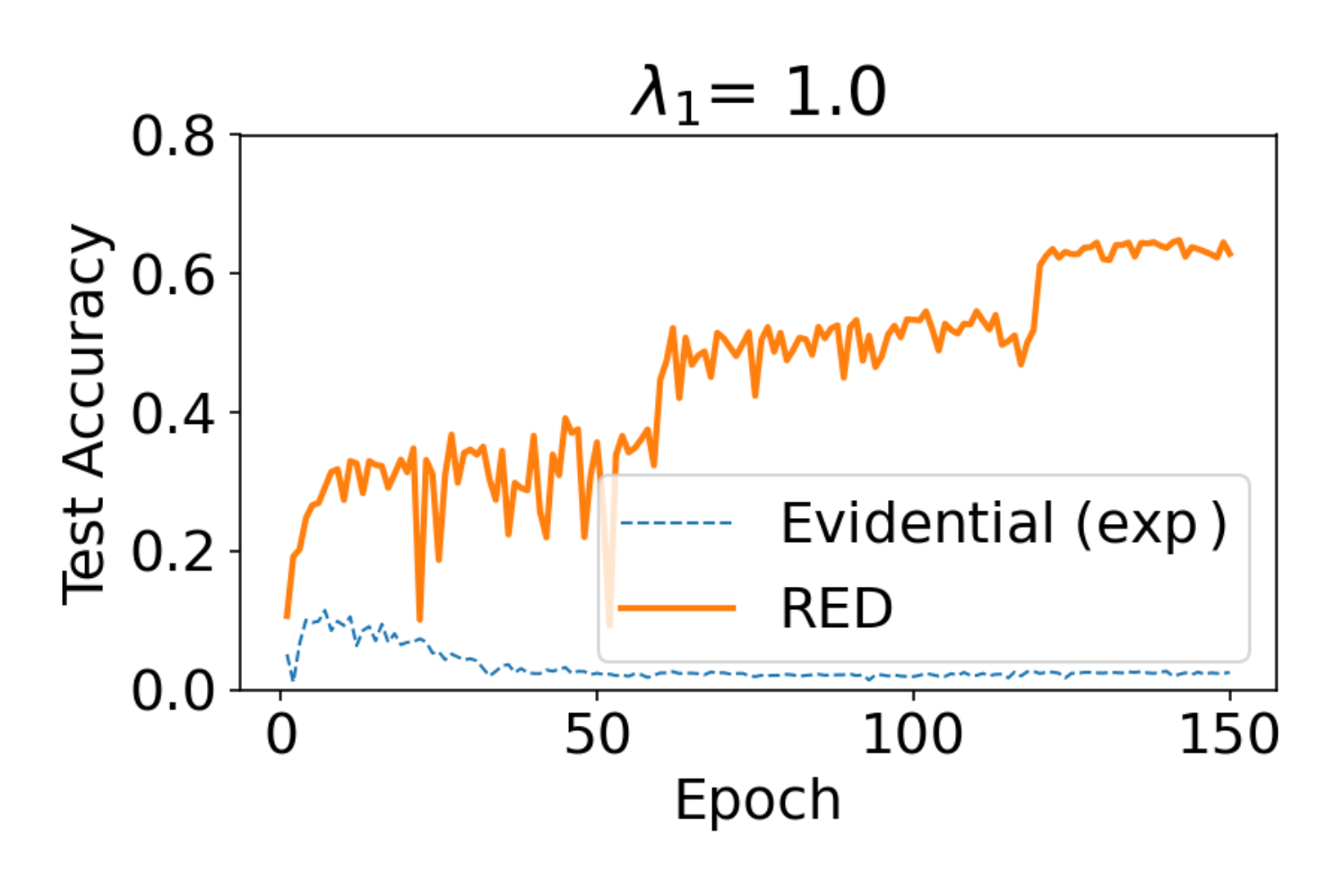}
}
\vspace{-3mm}
\caption{\label{fig:CorEvRegImpact}Impact of correct evidence regularization to test accuracy: (a), (b) - MNIST Results; (c), (d) - Cifar100 Results}
\vspace{-3mm}
\end{figure} 

\paragraph{Zero-evidence Sample Anaysis.}
{ Similar to the toy MNIST zero-evidence analysis,  we consider the Cifar100 dataset, and carry out the analysis for this complex dataset/setting. Instead of focusing on a few training examples, we present the average statistics of the evidence ($\mathcal{E}$) for the 50,000 training samples in the 100 class classification problem for a model trained for 200 epochs using a log-based evidential loss in \eqref{eqn:evLogloss} with $\lambda_1 = 1.0$. For reference, the samples with less than 0.01 average evidence (\ie $\mathcal{E} \leq 0.01$) are samples on which the model is not confident (\ie having a high vacuity of $\nu \geq 0.99$), and are close to the ideal zero-evidence region. Our proposed RED model effectively avoids such zero evidence regions, and has the lowest number of samples (i.e. only $0.06\%$ of total training dataset compared to 58.96\% of \texttt{SoftPlus} based, and 100\% of \texttt{ReLU} based evidential models) in very low evidence regions.}
\begin{table}[ht]
\vspace{-2mm}
\centering
\small
    \caption{Zero-Evidence Analysis for Complex Dataset-Setting
    }
    \label{tab:ZeroEvidenceComplexDatasetSetting}
\begin{tabular}{|p{0.06\textwidth}|p{0.07\textwidth}|p{0.07\textwidth}|p{0.07\textwidth}|p{0.08\textwidth}|}
\hline
Model & $\mathcal{E}\leq.01$ & $\mathcal{E} \leq 0.1$ & $\mathcal{E} \leq 1.0$ & $\mathcal{E} > 1.0$\\
\hline
\texttt{ReLU}&50000&50000&50000&0 \\
SoftPlus&29483&32006&49938&62 \\
Exp&48318&49881&49949&51 \\
\textbf{RED}&30&16322&25154&24846 \\
\hline
\end{tabular}
\vspace{-2mm}
\end{table} 

\vspace{-2mm}\subsection{Ablation Study}
\paragraph{Impact of loss function.}
We next study the impact of the evidential loss function on the model's performance using MNIST and CIFAR100 classification problems. We consider all three activations: $\texttt{ReLU}, \texttt{SoftPlus}, $ and $\exp$ to transform neural network outputs to evidence and carry out experiments over CIFAR100 with identical model and settings. As seen in Table \ref{tab:impactActivationLossesCifar100}, the generalization performance of evidential model is consistently sub-optimal when trained with evidential MSE loss given by  \eqref{eqn:evMSEloss} compared to the two other evidential losses \eqref{eqn:evDigammaloss} $\&$ \eqref{eqn:evLogloss}. This is consistent across all three evidence activation functions. This is mainly due to the bounded nature of the evidential MSE loss \eqref{eqn:evMSEloss}: for all training samples, evidential MSE loss is bounded in the range of $[0,2]$. Type II Maximum Likelihood loss given in \eqref{eqn:evLogloss} and cross-entropy based evidential loss given in \eqref{eqn:evDigammaloss} show comparable empirical results. 

Next, we consider $\exp$ activation and conduct experiments over the MNIST dataset for incorrect evidence regularization strengths of $\lambda_1 = 0 \& 1$. We again observe similar results where the training with the Evidential MSE loss in \eqref{eqn:evMSEloss} leads to sub-optimal test performance. Additional results, along with theoretical analysis are presented in the Appendix. In the subsequent experiments, we consider the Type II Maximum Likelihood loss \eqref{eqn:evLogloss} for evidential model training due to its simplicity and some theoretical advantages (see Appendix \ref{ap:evLossAnalysis}). We leave a thorough investigation of these two evidential losses (\eqref{eqn:evDigammaloss} \& \eqref{eqn:evLogloss}) as future work. 

\begin{table}[ht]
\vspace{-2mm}
\centering
\small
    \caption{Impact of evidential losses on classification performance
    }
    \label{tab:impactActivationLossesCifar100}
\begin{tabular}{|p{0.06\textwidth}|p{0.07\textwidth}|p{0.07\textwidth}|p{0.07\textwidth}|p{0.08\textwidth}|}
\hline
Loss & ReLU & SoftPlus & $\exp$ & \textbf{RED(Ours)}\\
\hline
MSE(\ref{eqn:evMSEloss})&$31.49_{\pm 0.3}$&$15.74_{\pm 0.5}$&$42.95_{\pm 0.7}$&$\bf{75.73_{\pm 0.3}}$\\
CE (\ref{eqn:evDigammaloss})&$68.62_{\pm2.4}$&$74.44_{\pm0.1}$&$76.23_{\pm0.1}$&$\bf{76.35_{\pm0.1}}$\\
Log(\ref{eqn:evLogloss})&$61.27_{\pm3.8}$&$74.48_{\pm0.1}$&$76.12_{\pm0.1}$&$\bf{76.43_{\pm0.2}}$\\
\hline
\end{tabular}
\vspace{-5mm}
\end{table} 

\begin{figure}[ht!] \label{fig:LossImpactEvidentialModel}
\centering
\subfigure[Trend for $\lambda_1 = 0.0$]{
  \includegraphics[width=0.46\linewidth]{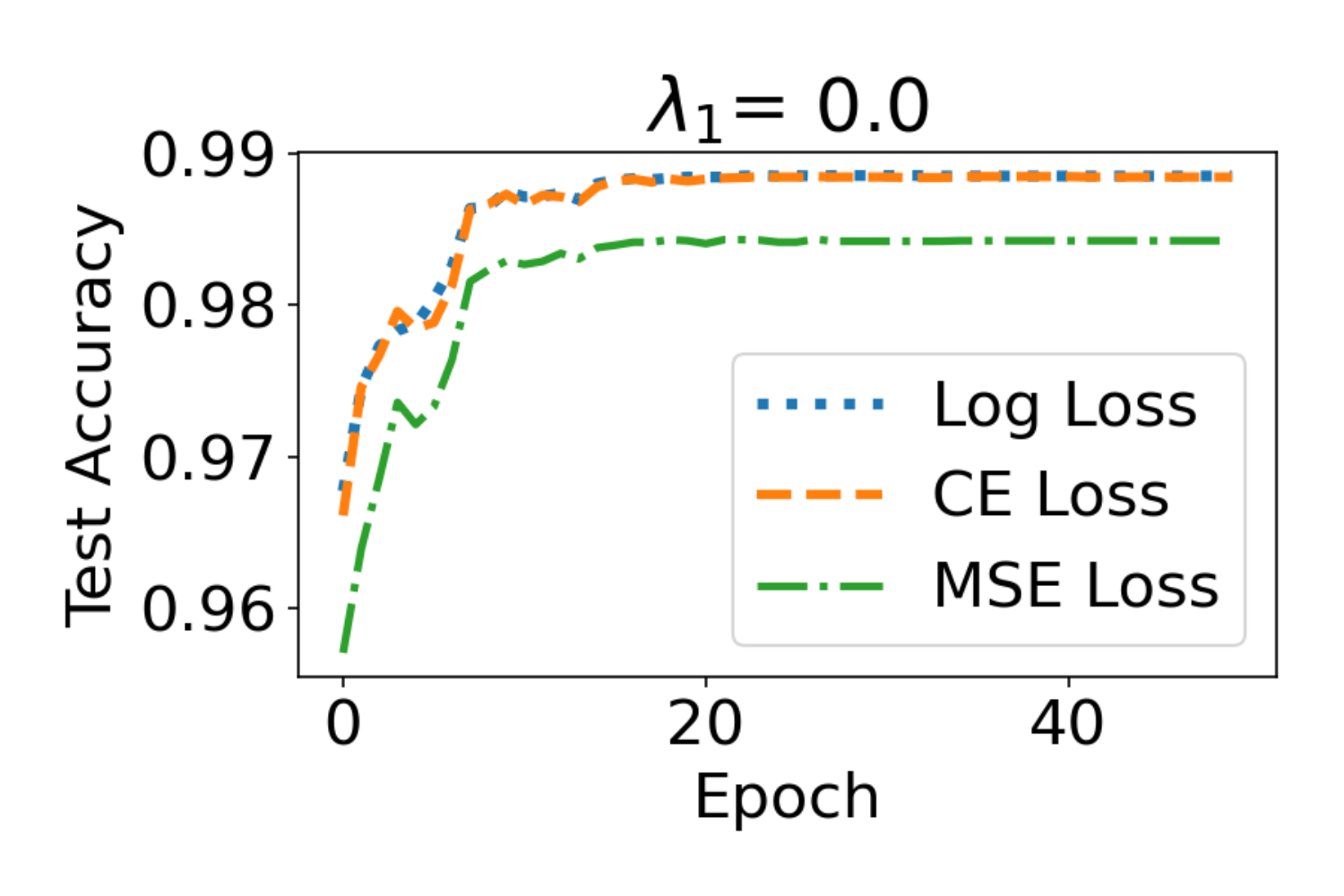}
}
\subfigure[Trend for $\lambda_1 = 1.0$]{
  \includegraphics[width=0.46\linewidth]{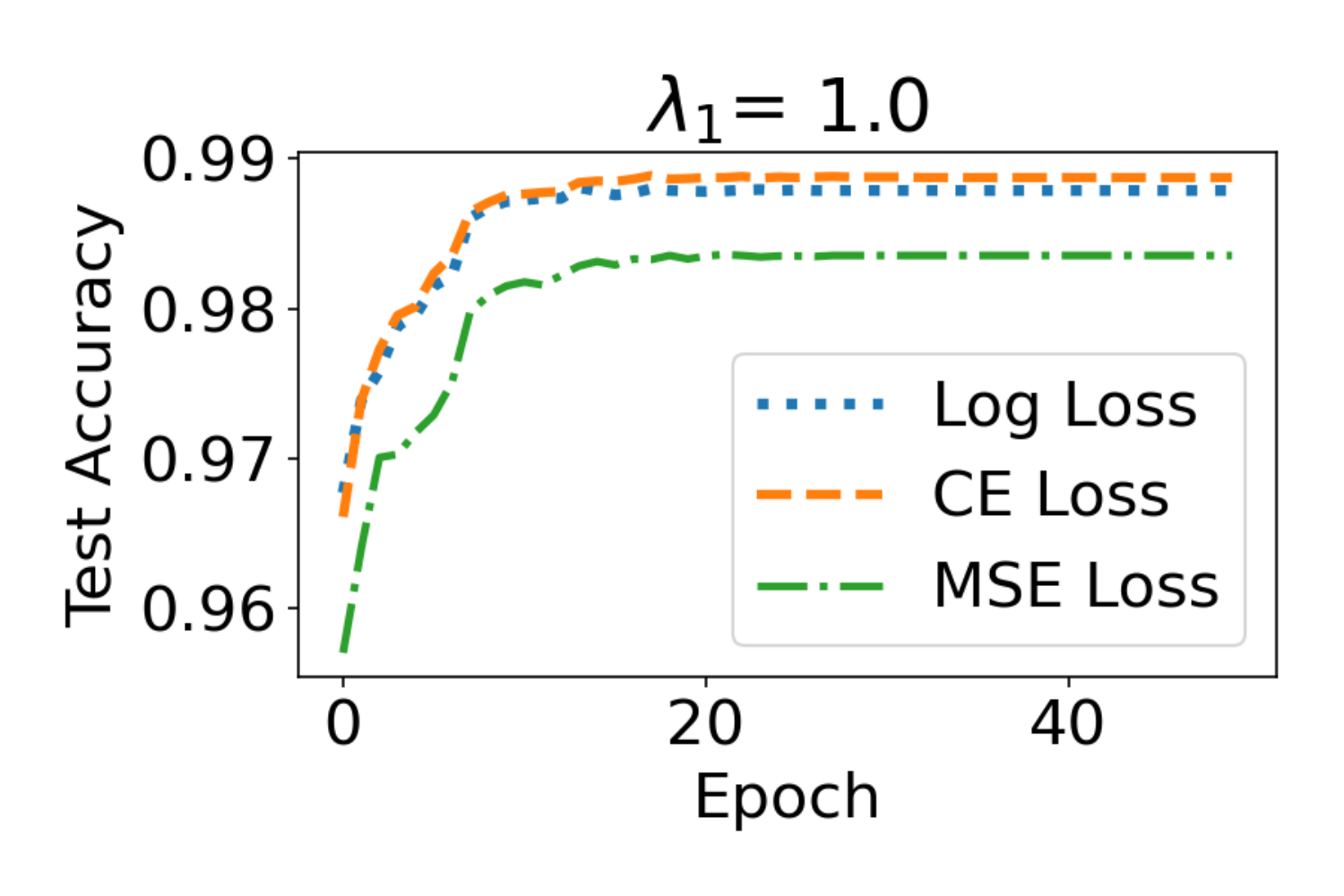}
}
\vspace{-3mm}
\caption{Impact of evidential losses on test set accuracy}
\vspace{-2mm}
\end{figure} 



\begin{figure}[ht!] 
\vspace{-1mm}
\centering
  \includegraphics[width=0.7\linewidth]{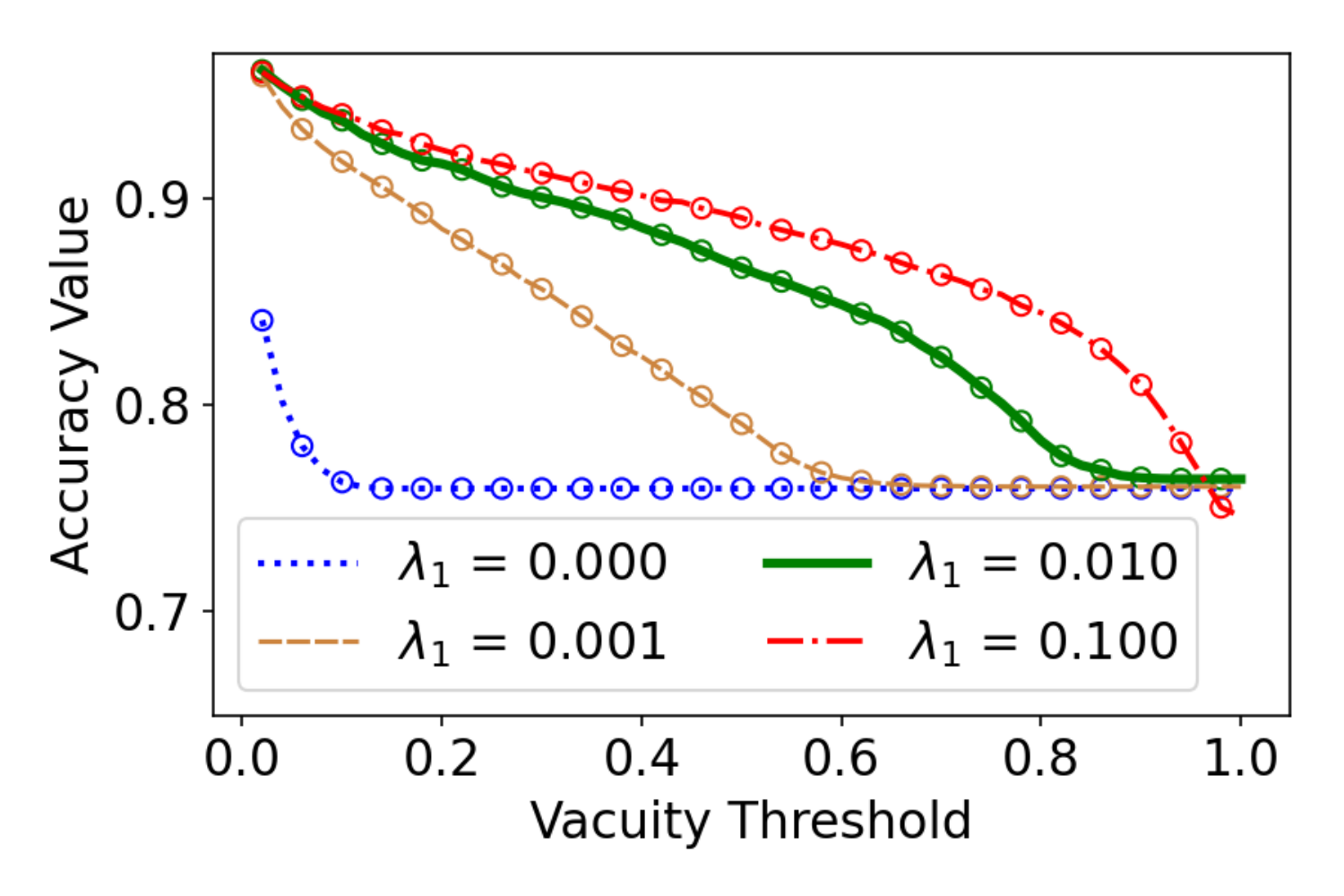}
\vspace{-5mm}
\caption{\label{fig:accVacuityCurve}Accuracy-Vacuity curve}
\vspace{-2mm}
\end{figure}

\paragraph{Study of uncertainty information.}\vspace{-2mm}

We now investigate the uncertainty behavior of the proposed evidential model with Cifar100 experiments. 
We present the Accuracy-Vacuity curve for different incorrect evidence regularization strengths ($\lambda_1$) in Figure \ref{fig:accVacuityCurve}. Vacuity reflects the lack of confidence in the predictions, and the accuracy of effective evidential model should increase with lower vacuity threshold. Without any incorrect evidence regularization (\ie $\lambda_1 = 0$), the evidential model is highly confident on its predictions and all test samples are concentrated on the low vacuity region. As the incorrect evidence regularization strength is increased, the model outputs more accurate confidence in the predictions. Strong incorrect evidence regularization hurts the generalization over the test set as indicated by low accuracy when all test samples are considered (\ie vacuity threshold of 1.0). In all cases, the evidential model shows reasonable uncertainty behavior: the model's test set accuracy increases as the vacuity threshold is decreased. 

Next, we look at the accuracy of the evidential models on their top-$K$ $\%$ most confident predictions over the test set. Table \ref{tab:topKConfidentAccuracyTrend} shows the accuracy trend of Top-$K$ (\%) confident samples. Consider the most confident 20\% samples (corresponding to 2000 test samples of Cifar100 dataset). The proposed model leads to highest accuracy (of 99.35\%) compared to all the models. Similar trend is seen for different $K$ values where the proposed model shows comparable to superior results demonstrating its accurate uncertainty quantification capability. 

\begin{table}[ht]
\vspace{-4mm}
\centering
\small
    \caption{Accuracy on Top-K$\%$ confident samples (\%)
    }
    \label{tab:topKConfidentAccuracyTrend}
\begin{tabular}{|p{0.08\textwidth}|p{0.035\textwidth}|p{0.035\textwidth}|p{0.035\textwidth}|p{0.035\textwidth}|p{0.035\textwidth}|p{0.035\textwidth}|}
\hline
Model & 10\% & 20\% & 30\%  & 50\% & 80\% & 100\% \\
\hline
\texttt{ReLU} &98.50&98.30&97.27&90.60&71.54&61.27 \\
\texttt{SoftPlus} &99.10&98.75&98.30&95.86&85.56&74.48 \\
$\exp$ &99.40&98.95&98.50&96.52&86.46&76.12 \\
\textbf{RED} &99.60&99.35&98.83&96.24&86.38&76.43 \\
\hline
\end{tabular}
\end{table} 

We next consider out-of-distribution (OOD) detection experiments for the Cifar100-trained evidential model using SVHN dataset (as OOD) \cite{netzer2011reading}. As seen in Table \ref {tab:OODSVHNCifar100}, the evidential models, on average, output very high vacuity for the OOD samples, showing the potential for OOD detection. 

\begin{table}[ht]
\vspace{-4mm}
\centering
\small
    \caption{Out-of-Distribution sample detection
    }
    \label{tab:OODSVHNCifar100}
\begin{tabular}{|p{0.1\textwidth}|p{0.12\textwidth}|p{0.17\textwidth}|}
\hline
Model & InD Vacuity& OOD Vacuity (SVHN) \\
\hline
$\exp$ &0.3227&0.7681 \\
\textbf{RED (Ours)} &0.2729&0.7552 \\
\hline
\end{tabular}
\vspace{-2mm}
\end{table} 

We present the AUROC score for Cifar100 trained models with SVHN dataset test set as the OOD samples in Table \ref{tab:aurocComparison}. In AUROC calculation, we use the maximum softmax score for the standard model, and predicted vacuity score for all the evidential models. As can be seen, the \texttt{exp}-based model outperforms all other activation functions, and the proposed model RED can learn from all the training samples that leads to the best performance. 
\begin{table}[ht]
\vspace{-2mm}
\centering
\small
    \caption{AUROC for Cifar100-SVHN experiment
    }
    \label{tab:aurocComparison}
\begin{tabular}{|p{0.05\textwidth}|p{0.05\textwidth}|p{0.06\textwidth}|p{0.06\textwidth}|p{0.05\textwidth}|p{0.05\textwidth}|}
\hline
Model & ReLU&SoftPlus & Standard & $\exp$ &\textbf{RED}\\ 
\hline
AUROC &$0.7430$&$0.8058$&$0.8669$&$0.8804$&$\bf{0.8833}$ \\
\hline
\end{tabular}
\vspace{-4mm}
\end{table}

\vspace{-1mm}
\section{Conclusion}
\vspace{-1mm}
In this paper, we theoretically investigate the evidential models to identify their learning deficiency, which makes them fail to learn from zero-evidence regions. We then show the superiority of the evidential model with $\exp$ evidential activation over the $\texttt{ReLU}$ and $\texttt{SoftPlus}$ based models. We further analyze the evidential losses, and introduce a novel correct evidence regularization over the $\exp$-based evidential model. The proposed model effectively pushes the training samples out of the zero-evidence regions, leading to superior learning capabilities. We conduct extensive experiments that empirically validate all theoretical claims while demonstrating the effectiveness of the proposed approach.    

\vspace{-1mm}
\section*{Acknowledgements}
\vspace{-1mm}
This research was supported in part by an NSF IIS award IIS-1814450 and an ONR award N00014-18-1-2875. The views and conclusions
contained in this paper are those of the authors and should not
be interpreted as representing any funding agency.

\clearpage
\newpage
\bibliography{main}
\bibliographystyle{icml2023}
\newpage
\appendix
\onecolumn
\begin{center}
\large{\bf Appendix}
\end{center}
\section*{Organization of the Appendix} 
\begin{itemize}
    \item In Section \ref{sec:appAnalysisStandardClassificationModels}, we present an analysis of standard classification models trained with cross-entropy loss to show their learning capabilities.
    \item In Section \ref{apSec:evidentialProof}, we present a complete proof of Theorem \ref{supotimalityTheorem} for different evidential losses that demonstrates the inability of evidential models to learn from zero-evidence samples.
    \item In Section \ref{app:evIncReg}, we describe different incorrect evidence regularizations used in the existing literature and carry out a gradient analysis to study their impact on evidential model learning.
    \item In Section \ref{sec:impactLogitsEvTranf}, we present the proof for Theorem \ref{th:superirorityofExp} that shows the superiority of $\exp$ activation over the $\texttt{SoftPlus}$ and $\texttt{ReLU}$ functions to transform logits to evidence.
    \item In Section \ref{ap:evLossAnalysis}, we analyze the evidential losses that reveals the theoretical limitation of evidential models trained using Bayes risk with sum of squares loss. 
    \item In Section \ref{ap:AdditionalExpResults}, we present additional experiment results, {clarifications}, hyperparameter details, and discuss some limitations along with possible future works.
\end{itemize}
     The source code for the experiments carried out in this work is attached in the supplementary materials and is available at the link: \href{https://github.com/pandeydeep9/EvidentialResearch2023}{https://github.com/pandeydeep9/EvidentialResearch2023} 
\section{Standard Classification Model}\label{sec:appAnalysisStandardClassificationModels}
Consider a standard cross-entropy based model for $K-$class classification. Let the overall network be represented by $f_{\Theta}(.)$, and let $\mathbf{o} = f_{\Theta}(\mathbf{x})$ be the output from this network before the softmax layer for input $\mathbf{x}$ and one-hot ground truth label of $\mathbf{y}$. The output after the softmax layer is given by 
\begin{align}
\texttt{sm}_i = \frac{\exp({o}_i)}{\sum_{k=1}^{K} \exp({o}_k)} = \frac{\exp({o}_i)}{S^{\text{ce}}}    
\end{align}
Where $S^{\text{ce}} =  \sum_{i=1}^K \exp(o_i)$. The model is trained with cross-entropy loss. For a given sample $(\mathbf{x}, \mathbf{y})$, the loss is given by
\begin{align}
    \mathcal{L}_{\texttt{cross-entropy}} &= -\sum_{k = 1}^K y_k \log (\texttt{sm}_k ) 
    = -\sum_{k = 1}^K \Big[y_k {o}_k - y_k \log \Big( \sum_{i=1}^K \exp(o_i) \Big) \Big] \\
    &= \log S^{\text{ce}} - \sum_{k=1}^K y_k o_k
\end{align}
Now, looking at the gradient of this loss with respect to the pre-softmax values $\mathbf{o}$
\begin{align}
    \text{grad}_k &= \frac{\partial \mathcal{L}_{\texttt{cross-entropy}}}{\partial o_k} = \Big( \frac{1}{S^{\text{ce}}}\frac{\partial S^{\text{ce}}}{\partial o_k} - y_k \Big) 
    = \Big( \frac{\exp (o_k)}{S^{\text{ce}}} - y_k \Big) = \texttt{sm}_k - y_k
\end{align}
\subsection*{Analysis of the gradients For Standard Classification Model.} \label{par:gradInterpCross} The gradient measures the error signal, and for standard classification models, it is bounded in the range [-1, 1] as $0 \leq \texttt{sm}_k \leq 1$ and $y_k \in \{0, 1\}$. The model is updated using gradient descent based optimization objectives. For input $\mathbf{x}$, the neural network outputs K values $o_1$ to  $o_K$, and the corresponding ground truth is $\mathbf{y}$, $y_{gt} = 1, y_{\neq gt} = 0$.

When $y_i$ = 0, the gradient signal is $\text{grad}_i = \texttt{sm}_i$ and the model optimizes the parameters to minimize this value. Only when $\texttt{sm}_i = 0$, the gradient is zero, and the model is not updated. In all other cases when $\texttt{sm}_i \neq 0$, there is a non-zero gradient dependent on $\texttt{sm}_i$, and the model is updated to minimize the $\texttt{sm}_i$ as expected.

When $y_i$ = 1, the gradient signal is $\text{grad}_i = \texttt{sm}_i - 1$ and the model optimizes the parameters to minimize this value. As $ \texttt{sm}_i \in [0, 1]$, only when the model outputs a large logit on $i$ (corresponding to the ground truth class) and small logit for all other nodes, $\texttt{sm}_i = 1$, the gradient is zero, and the model is not updated. In all other cases when $\texttt{sm}_i < 1$, there is a non-zero gradient dependent on $\texttt{sm}_i$ and the model is updated to maximize the $\texttt{sm}_i$ and minimize all other $\texttt{sm}_{\neq i}$ as expected. The gradient signal in standard classification models trained with standard cross-entropy loss is reasonable and enables learning from all the training data samples. 

\section{Evidential Classification Models} \label{apSec:evidentialProof}
\textbf{Theorem \ref{supotimalityTheorem}:}
{Given a training sample $(\mathbf{x}, \mathbf{y})$, if an evidential neural network outputs zero evidence $\mathbf{e}$, then the gradients of the evidential loss evaluated on this training sample over the network parameters reduce to zero. }
\begin{proof}
In the main paper, we  considered a $K-$class classification problem and a representative evidential model trained using Bayes risk with sum of squares loss (Eqn. \ref{eqn:evMSEloss}) in the proof. Following 3 variants of evidential losses (\cite{sensoy2018evidential}) have been commonly used in evidential classification works: 
\begin{enumerate}[noitemsep,topsep=0pt,leftmargin=*]
    \item Bayes risk with sum of squares loss (\ie Evidential MSE loss) \cite{zhao2020uncertainty}
    \begin{align} 
    \label{eqn:evMSEloss} 
    \mathcal{L}^{\texttt{MSE}}(\mathbf{x}, \mathbf{y}) &= \sum_{j=1}^K (y_j - \frac{\alpha_j}{S})^2 + \frac{\alpha_j (S - \alpha_j)}{S^2(S+1)}
\end{align}
    \item Bayes risk with cross-entropy loss (\ie Evidential CE loss)\cite{charpentier2020posterior}
\begin{align}\label{eqn:evDigammaloss}
    \mathcal{L}^{\texttt{CE}}(\mathbf{x}, \mathbf{y}) &= \sum_{j=1}^K y_k \Big( \Psi(S) - \Psi(\alpha_k) \Big) 
\end{align}
    \item Type II Maximum Likelihood loss (\ie Evidential log loss)\cite{Pandey_2022_CVPR}
\begin{align}\label{eqn:evLogloss}
    \mathcal{L}^{\texttt{Log}}(\mathbf{x}, \mathbf{y}) = \sum_{k=1}^K y_k \Big( \log(S) - \log (\alpha_k) \Big)
\end{align}
\end{enumerate}

For completeness, we consider all three loss functions used in evidential classification models and carry out their analysis.
\subsection{Gradient of Evidential Activation Functions $\mathcal{A}(.)$}
Three non-linear functions are proposed and commonly used in the existing literature to transform the neural network output to evidence: 1) $\texttt{ReLU}$ function, 2) $\texttt{SoftPlus}$ function, and 3) Exponential function. In this section, we compute the gradients of the evidence output $e_i$ from these non-linear activation functions with respect to the logit input $o_i$
\begin{enumerate} \label{eqn:Activation Functions}
    \item $\mathcal{A}(.) = \texttt{ReLU}(.) = \max( 0, . )$
    \begin{align}
e_k &= \texttt{ReLU}(o_k) = \max(0, o_k) \implies \frac{\partial e_k}{\partial o_k} = \begin{cases}
0 \quad \quad \text{if} \quad \quad o_k \leq 0 \\
1 \quad \quad \text{otherwise}
\end{cases}    
\end{align}
    \item $\mathcal{A}(.) = \texttt{SoftPlus}(.) = \log ( 1 + \exp(.))$
    \begin{align} \label{eqn:gradSoftplusActivation}
    e_k =  \log ( \exp(o_k) + 1)  \implies \frac{\partial e_k}{\partial o_k} = \frac{1}{1 + \exp(-o_k)} = \text{Sigmoid}(o_k)
\end{align}
    \item $\mathcal{A}(.) = \exp(.)$
    \begin{align}\label{eqn:gradExpActivation}
    e_k &= \exp(o_k)
    \implies \frac{\partial e_k}{\partial o_k} = \exp(o_k) = e_k = \alpha_k - 1
\end{align}
\end{enumerate}

\subsection{Evidential Model Trained using Bayes risk with sum of squares loss (\ie Eqn. \ref{eqn:evMSEloss})}
\begin{proof}
Consider an input $\mathbf{x}$ with one-hot ground truth label of $\mathbf{y}$. Let the ground truth class be $g$ i.e. $y_{gt} =1, $ with corresponding Dirichlet parameter $\alpha_{gt}$, and $y_{\neq gt} = 0$. Moreover, let $\mathbf{o}, \mathbf{e}, \text{and } \boldsymbol{\alpha}$ represent the neural network output vector before applying the activation $\mathcal{A}$, the evidence vector, and the Dirichlet parameters respectively. 

In this evidential framework, the loss is given by
\begin{align}
    \mathcal{L}^{\texttt{MSE}}(\mathbf{x}, \mathbf{y}) &= \sum_{j=1}^K (y_j - \frac{\alpha_j}{S})^2 + \frac{\alpha_j (S - \alpha_j)}{S^2(S+1)} 
    = 1 - \frac{2 \alpha_{gt}}{S} + \frac{\sum_k \alpha_k^2}{S^2} + \frac{2\sum_i \sum_j \alpha_i\alpha_j}{S^2(S+1)} \\
    &= 2 - \frac{2 \alpha_{gt}}{S} - \frac{2\sum_i \sum_j \alpha_i\alpha_j}{S(S+1)} 
\end{align}
Now, consider different components of the loss and compute the gradients of the components with respect to Dirichlet parameters $\alpha$,
\begin{align*}
    \frac{\partial \frac{\alpha_{gt}}{S}}{\partial \alpha_{gt}} &= \frac{1}{S} - \frac{\alpha_{gt}}{S^2} \quad \& \quad
    \frac{\partial \frac{\alpha_{gt}}{S}}{\partial \alpha_{\neq gt}} =  - \frac{\alpha_{gt}}{S^2} 
       \implies \frac{\partial \frac{\alpha_{gt}}{S}}{\partial \alpha_{k}} =  \frac{y_k}{S}- \frac{\alpha_{gt}}{S^2}
\end{align*}
The gradient of the variance term is the same for all the $K$ Dirichlet parameters and is given by 
\begin{align*}
    \frac{\partial  \frac{\sum_i \sum_j \alpha_i\alpha_j}{S(S+1)}}{\partial \alpha_k} = \frac{( S - \alpha_k)}{S(S+1)} - \frac{(2S + 1)\sum_{i} \sum_{j}\alpha_i \alpha_j}{(S^2 + S)^2}
\end{align*}
Now, the gradient of the loss with respect to the  neural network output can be computed using the chain rule as
\begin{align*}\label{eqn:apgradMseInt}
    &\frac{\partial \mathcal{L}^{\texttt{MSE}}(\mathbf{x}, \mathbf{y})}{\partial o_k}  = \frac{\partial \mathcal{L}^{\texttt{MSE}}(\mathbf{x}, \mathbf{y})}{\partial \alpha_k}\frac{\partial e_k}{\partial o_k} 
    = -\bigg[2 \frac{\partial \frac{\alpha_k}{S}}{\partial \alpha_{k}} -2 \frac{\partial  \frac{\sum_i \sum_j \alpha_i\alpha_j}{S(S+1)}}{\partial \alpha_k}\bigg] \times \frac{\partial e_k}{\partial o_k} \\
    & = \bigg[ \frac{2\alpha_{gt}}{S^2} - 2\frac{y_k}{S} - \frac{2( S - \alpha_k)}{S(S+1)} +
    \frac{2(2S + 1)\sum_{i} \sum_{j}\alpha_i \alpha_j}{(S^2 + S)^2}
    \bigg] \times \frac{\partial e_k}{\partial o_k}
\end{align*}
\textbf{Case I:} $\texttt{ReLU}(.)$ to transform logits to evidence
\begin{align}
e_k &= \text{ReLU}(o_k) = \max(0, o_k) 
\implies \frac{\partial e_k}{\partial o_k} = \begin{cases}
1 \quad \quad \text{if} \quad \quad o_k > 0 \\
o \quad \quad \text{otherwise}
\end{cases}    
\end{align}
For zero-evidence sample with $\texttt{ReLU}(.)$ used to transform the logits to evidence, the logits $o_k$ satisfy the relationship $o_k \leq 0 \; \forall \; k
\implies \frac{\partial e_k}{\partial o_k} = 0
\implies\frac{\partial \mathcal{L}^{\texttt{MSE}}(\mathbf{x}, \mathbf{y})}{\partial o_k}  = 0
$

\textbf{Case II:} $\texttt{SoftPlus}(.)$ to transform logits to evidence
\begin{align}
    e_k &=  \log ( \exp(o_k) + 1) \implies 
    \frac{\partial e_k}{\partial o_k} = \text{Sigmoid}(o_k)
\end{align}
\textbf{Case II:} $\exp(.)$ to transform logits to evidence
\begin{align}
    e_k &= \exp(o_k)
    \implies 
    \frac{\partial e_k}{\partial o_k} = \exp(o_k) = \alpha_k - 1
\end{align}
For zero-evidence sample with $\texttt{SoftPlus}(.)$ used to transform the logits to evidence, the logits $o_k \rightarrow -\infty \implies \text{Sigmoid}(o_k) \rightarrow 0 \; \& \; \frac{\partial e_k}{\partial o_k} \rightarrow 0$. For zero-evidence sample with $\exp(.)$ used to transform the logits to evidence, $\alpha_k \rightarrow 1 \implies \frac{\partial e_k}{\partial o_k} \rightarrow 0$. Moreover, there is no term in the first part of the loss gradient (see Eqn. \ref{eqn:apgradMseInt}) to counterbalance these zero-approaching gradients. So, for zero-evidence samples,

\begin{align}
\frac{\partial \mathcal{L}^{\texttt{MSE}}(\mathbf{x}, \mathbf{y})}{\partial o_k}  = 0
\end{align}
Since the gradient of the loss with respect to all the nodes is zero, there is no update to the model from such samples. Thus, the evidential models fail to learn from such zero-evidence samples. 
\end{proof}

\subsection{Evidential Model Trained using Type II Maximum Likelihood formulation of Evidential loss (\ie Eqn. \ref{eqn:evLogloss})}
Consider a $K-$class evidential classification model that trains the model using Type II Maximum Likelihood formulation of the evidential loss. Consider an input $\mathbf{x}$ with one-hot ground truth label of $\mathbf{y}$, $\sum_{k=1}^K y_k = 1$. For this evidential framework, the Type II Maximum Likelihood loss is given by
\begin{align}
     \mathcal{L}^{\texttt{Log}}(\mathbf{x}, \mathbf{y}) = \sum_{k=1}^K y_k \Big( \log(S) - \log (\alpha_k) \Big) 
     =  \log S - \sum_{k=1}^K y_k \log \alpha_k
\end{align} 
Taking the gradient of the loss with the logits $\mathbf{o}$, we get
\begin{align} \label{eq:gradtype2}
    \text{grad}_k &= \frac{\partial \mathcal{L}^{\texttt{Log}}(\mathbf{x}, \mathbf{y})}{\partial o_k} 
    = \frac{1}{S}\frac{\partial S}{\partial o_k} - y_k \frac{1}{\alpha_k} \frac{\partial \alpha_k}{\partial o_k} 
    = \Big(\frac{1}{S} - \frac{y_k}{\alpha_k} \Big)\frac{\partial e_k}{\partial o_k}
\end{align}
\textbf{Case I:} $\texttt{ReLU}(.)$ to transform logits to evidence

For any zero-evidence sample with $\texttt{ReLU}(.)$ used to transform the logits to evidence, the logits $o_k$ satisfy the relationship $o_k \leq 0 \; \forall \; k \;
\implies \frac{\partial e_k}{\partial o_k} = 0 
\implies\frac{\partial \mathcal{L}^{\texttt{Log}}(\mathbf{x}, \mathbf{y})}{\partial o_k}  = 0 \; \forall k \in [1,K]$

\textbf{Case II:} $\texttt{SoftPlus}(.)$ to transform logits to evidence.
Considering Eqn. \ref{eq:gradtype2} and Eqn \ref{eqn:gradSoftplusActivation}, the gradient of the loss with respect to the logits becomes 
\begin{align} \label{eq:gradtype2SoftplusGrad}
    \text{grad}_k = \frac{\partial \mathcal{L}^{\texttt{Log}}(\mathbf{x}, \mathbf{y})}{\partial o_k} = \Big(\frac{1}{S} - \frac{y_k}{\alpha_k} \Big)\text{Sigmoid}(o_k)
\end{align}

\textbf{Case III:} $\exp(.)$ to transform logits to evidence. Considering Eqn. \ref{eq:gradtype2} and Eqn \ref{eqn:gradExpActivation}, the gradient of the loss with respect to the logits becomes 
\begin{align} \label{eq:gradtype2ExpGrad}
    \text{grad}_k = \frac{\partial \mathcal{L}^{\texttt{Log}}(\mathbf{x}, \mathbf{y})}{\partial o_k} = \Big(\frac{1}{S} - \frac{y_k}{\alpha_k} \Big)(e_k) = \Big(\frac{1}{S} - \frac{y_k}{\alpha_k}\Big) (\alpha_k - 1)
\end{align}

For zero-evidence sample with $\texttt{SoftPlus}(.)$ used to transform the logits to evidence, the logits $o_k \rightarrow -\infty \implies \text{Sigmoid}(o_k) \rightarrow 0 \; \& \; \frac{\partial e_k}{\partial o_k} \rightarrow 0$. Similarly, for zero-evidence sample with $\exp(.)$ used to transform the logits to evidence, $\alpha_k \rightarrow 1 \implies \frac{\partial e_k}{\partial o_k} \rightarrow 0$. Moreover, there is no term in the first part of the loss gradient (see Eqn. \ref{eq:gradtype2SoftplusGrad} and Eqn. \ref{eq:gradtype2ExpGrad} ) to counterbalance these zero-approaching gradient terms. 

Since the gradient of the loss with respect to all the nodes is zero, there is no update to the model from such samples. Thus, the evidential models trained with Type II Maximum Likelihood formulation of the evidential loss fail to learn from such zero-evidence samples. 

\subsection{Evidential Model Trained using Bayes risk with cross-entropy formulation of Evidential loss (\ie Eqn. \ref{eqn:evDigammaloss})}
Consider a $K-$class evidential classification model that trains model using Bayes risk with cross-entropy loss for evidential learning (Eqn. \ref{eqn:evDigammaloss}). Consider an input $\mathbf{x}$ with one-hot ground truth label of $\mathbf{y}$, $\sum_{k=1}^K y_k = 1$. For this evidential framework, the loss is given by
\begin{align}
    \mathcal{L}^{\texttt{CE}}(\mathbf{x}, \mathbf{y}) &= \sum_{j=1}^K y_k \Big( \Psi(S) - \Psi(\alpha_k) \Big)  =  \Psi(S) - \Psi(\alpha_{gt})
\end{align}
Where $\alpha_{gt}$ represents the output Dirichlet parameter for the ground truth class i.e. $y_{gt} = 1$, $y_{\neq gt} = 0$, and $\Psi(.)$ represents the Digamma function, and for $z\geq 1$, is given by
\begin{align*}
    \Psi(z) &= \frac{d}{dz} \log \Gamma(z) = \frac{d}{dz} \bigg( -\gamma z - \log z + \sum_{n=1}^{\infty}\Big( \frac{z}{n} - \log \big(1 + \frac{z}{n}\big)\Big) \bigg) 
    = -\gamma -\frac{1}{z} + z\sum_{n=1}^{\infty} \frac{1}{n(n+z)} \\
\end{align*}
Here, $\gamma$ is the Euler–Mascheroni constant, and $\Gamma(.)$ is the gamma function,  Using Weierstass's definition of gamma function \cite{knopp1996weierstrass} for values outside negative integers that is given by
\begin{align*}
    \Gamma(z) = \frac{e^{-\gamma z}}{z} \prod_{n=1}^{\infty} \Big(1 + \frac{z}{n}\Big)^{-1} e^{\frac{z}{n}}
\end{align*}
Using the definition of the digamma functions, the loss updates as 
\begin{align}
    \mathcal{L}^{\texttt{CE}}(\mathbf{x}, \mathbf{y})  =  \Psi(S) - \Psi(\alpha_{gt}) = \frac{1}{\alpha_{gt}} - \frac{1}{S} + 
    S\sum_{n=1}^{\infty} \frac{1}{n(n+S)} - \alpha_{gt} \sum_{n=1}^{\infty} \frac{1}{n(n+\alpha_{gt})}
\end{align}
The derivative of the digamma function is bounded and is given by
\begin{align*}
    \frac{\partial \Psi(z)}{\partial z} &= \frac{\partial }{\partial z} \bigg( -\gamma -\frac{1}{z} + \sum_{n=1}^{\infty} \frac{1}{n} - \frac{1}{n+z} \bigg) = \frac{1}{z^2} + \sum_{n=1}^{\infty}\frac{1}{(n+z)^2} \\
    \frac{1}{z^2} <& \frac{\partial \Psi(z)}{\partial z} < \frac{1}{z^2} + \frac{\pi^2}{6} , \quad z \geq 1
\end{align*}
With this, we can compute the gradients of the loss with respect to the logits as 
\begin{align} \label{eq:gradDigap}
    \text{grad}_k &= \frac{\partial \mathcal{L}^{\texttt{CE}}(\mathbf{x}, \mathbf{y})}{\partial o_k} 
    = \frac{\partial   }{\partial \alpha_k} \big(\Psi(S) - \Psi(\alpha_{gt}) \big)\frac{\partial  \alpha_k }{\partial o_k} \
    = \Big(\frac{1}{S^2} + \sum_{i=1}^{\infty}\frac{1}{(n + S)^2} - \frac{y_k}{\alpha_{gt}^2} - \sum_{i=1}^{\infty}\frac{y_k}{(n + \alpha_{gt})^2}\Big)\frac{\partial  e_k }{\partial o_k}
\end{align}
\textbf{Case I:} $\texttt{ReLU}(.)$ to transform logits to evidence

For any zero-evidence sample with $\texttt{ReLU}(.)$ used to transform the logits to evidence, the logits $o_k$ satisfy the relationship $o_k \leq 0 \; \forall \; k \;
\implies \frac{\partial e_k}{\partial o_k} = 0 
\implies\frac{\partial \mathcal{L}^{\texttt{CE}}(\mathbf{x}, \mathbf{y})}{\partial o_k}  = 0 \; \forall k \in [1,K]$

\textbf{Case II:} $\texttt{SoftPlus}(.)$ to transform logits to evidence.
Considering Eqn. \ref{eqn:gradSoftplusActivation} and Eqn \ref{eq:gradDigap}, the gradient of the loss with respect to the logits becomes 
\begin{align}
    \text{grad}_k = \frac{\partial \mathcal{L}^{\texttt{CE}}(\mathbf{x}, \mathbf{y})}{\partial o_k} = \Big(\frac{1}{S^2} + \sum_{i=1}^{\infty}\frac{1}{(n + S)^2} - \frac{y_k}{\alpha_{gt}^2} - \sum_{i=1}^{\infty}\frac{y_k}{(n + \alpha_{gt})^2}\Big)\text{Sigmoid}(o_k)
\end{align}

\textbf{Case III:} $\exp(.)$ to transform logits to evidence. Considering Eqn. \ref{eqn:gradExpActivation} and Eqn \ref{eq:gradDigap}, the gradient of the loss with respect to the logits becomes 
\begin{align} \label{eq:digammaExpGrad}
    \text{grad}_k = \frac{\partial \mathcal{L}^{\texttt{CE}}(\mathbf{x}, \mathbf{y})}{\partial o_k} = \Big(\frac{1}{S^2} + \sum_{i=1}^{\infty}\frac{1}{(n + S)^2} - \frac{y_k}{\alpha_{gt}^2} - \sum_{i=1}^{\infty}\frac{y_k}{(n + \alpha_{gt})^2}\Big) (\alpha_k - 1)
\end{align}
For zero-evidence sample with $\texttt{SoftPlus}(.)$ used to transform the logits to evidence, the logits $o_k \rightarrow -\infty \implies \text{Sigmoid}(o_k) \rightarrow 0 \; \& \; \frac{\partial e_k}{\partial o_k} \rightarrow 0$. Similarly, for zero-evidence sample with $\exp(.)$ used to transform the logits to evidence, $\alpha_k \rightarrow 1 \implies \frac{\partial e_k}{\partial o_k} \rightarrow 0$. Moreover, there is no term in the first part of the loss gradient (see Eqn. \ref{eqn:apgradMseInt}) to counterbalance these zero-approaching gradient terms. 

 {The gradient of the loss with respect to all the nodes is zero for all the considered cases.} Since the gradient of the loss with respect to all the nodes is zero for all three cases, there is no update to the model from such samples. Thus, the evidential models fail to learn from such zero-evidence samples in all cases. 
\end{proof}

\section{Regularization in the Evidential Classification Models}\label{app:evIncReg}
Based on the evidence $\mathbf{e}$, beliefs $\mathbf{b}$, and the Dirichlet parameters $\bm{\alpha}$, various regularization terms have been introduced that aim to penalize the incorrect evidence/incorrect belief of the model, leading to the model with accurate uncertainty estimates. Here, we briefly summarize the key regurlaizations:
\begin{enumerate}

    \item  Introduce a forward KL regularization term as in EDL \cite{sensoy2018evidential} that regularizes the model to output no incorrect evidence. 
    \begin{align}
    \label{appeq:klsensoy}
    \begin{split}       \mathcal{L}_{\texttt{reg}}^{\texttt{EDL}}({\bm x},{\bm y}) &=  \text{KL}\big(\texttt{Dir}({\bm p}|\boldsymbol{\tilde\alpha}) ||\texttt{Dir}({\bm p}|\bm{1})\big) = \log \Big( \frac{\Gamma \sum_{k=1}^K \tilde \alpha_{k}}{\Gamma(K) \prod_{k=1}^K \Gamma \tilde \alpha_{k}} \Big)  + 
       \sum_{k =1}^K ({ \tilde \alpha_{k} -1 }) \bigg[ \psi(\tilde \alpha_{k}) - \psi \Big( \sum_{j = 1}^K \tilde\alpha_{j} \Big) \bigg] 
    \end{split}
    \end{align}
    Where $\boldsymbol{\tilde{\alpha}} = {\bm y} + (\bm{1} - {\bm y}) \odot \boldsymbol{\alpha} = (\tilde{\alpha}_1, \tilde{\alpha}_2,...\tilde{\alpha}_N)$ parameterize a dirichlet distribution, $\tilde{\alpha}_{i = gt} = 1, \tilde{\alpha}_{i} = \alpha_{i} \forall i \neq gt $. Here, the KL regularization term encourages the Dirichlet distribution based on the incorrect evidence i.e., $\text{Dir}({\bm p}|\boldsymbol{\tilde\alpha})$ to be flat which is possible when there is no incorrect evidence. From Eqn. \ref{appeq:klsensoy}, we can see that the regularization term, introduces digamma functions for the loss and may require evaluation of higher-order polygamma functions for challenging problems (e.g. involving bi-level optimizations as in MAML \cite{finn2017model}).
    
    
    \item Introduce an incorrect evidence regularization term as in ADL~\cite{shi2020multifaceted} that is the sum of the incorrect evidence for a sample
    \begin{align}
    \label{eq:incevreg_ap}
    \mathcal{L}_{\texttt{reg}}^{\texttt{ADL}}({\bm x},{\bm y}) =  \sum_{k=1}^{K} \big(\mathbf{e} \odot (\mathbf{1}-\mathbf{y})\big)_k = \sum_{k=1}^K e_k \times (1 - y_k) 
    \end{align}
    Here, $\odot$ represents element-wise product. The evidence for a class $e_k$ is only restricted to be non-negative and can take large positive values leading to large variation in the overall loss. 
    
    \item Introduce incorrect belief-based regularization as in Units-ML~\cite{Pandey_2022_CVPR} 
    \begin{align}
    \label{eq:inbelreg_apGradAnalysis}
    \mathcal{L}_{\texttt{reg}}^{\texttt{Units}}({\bm x},{\bm y}) &=  \sum_{k=1}^{K} \big(\frac{\mathbf{e}}{S} \odot (\mathbf{1}-\mathbf{y})\big)_k = \sum_{k=1}^K \frac{e_k}{S} \times (1 - y_k) 
    \end{align}
    The regularization value is bounded to be in a range of $[0, 1]$ for all the data samples, no matter how severe the mistake is.
\end{enumerate}

All three regularizations aim to guide the model such that the incorrect evidence is minimized (ideally close to zero). These regularizations help the evidential model acquire desired uncertainty quantification capabilities in evidential models. Such guidance is expected to update the model such that it maps input samples near zero-evidence regions in the evidence space. Thus, the regularization does not help address the issue of learning from zero-evidence samples and is likely to hurt the model's learning capabilities.

\subsection{Gradient Analysis of the Incorrect Evidence Regularizations}
The regularization terms use ground truth information to consider only the incorrect evidence. Thus, the gradient of the regularization loss with respect to the ground truth node $\alpha_{gt}$ is $0$. In this analysis, we consider the gradient with respect to non-ground truth nodes i.e. $\alpha_k$, and $o_k, k \neq gt$.
\begin{enumerate}
    \item Gradient for EDL regularization (Eqn. \ref{appeq:klsensoy} )
    \begin{align}
    \begin{split}
       \mathcal{L}_{\texttt{reg}}^{\texttt{EDL}}({\bm x},{\bm y}) &=  \text{KL}\big(\texttt{Dir}({\bm p}|\boldsymbol{\tilde\alpha}) ||\texttt{Dir}({\bm p}|\bm{1})\big) = \log \Big( \frac{\Gamma \sum_{k=1}^K \tilde \alpha_{k}}{\Gamma(K) \prod_{k=1}^K \Gamma \tilde \alpha_{k}} \Big)  + 
       \sum_{k =1}^K ({ \tilde \alpha_{k} -1 }) \bigg[ \psi(\tilde \alpha_{k}) - \psi \Big( \sum_{j = 1}^K \tilde\alpha_{j} \Big) \bigg] \\
 &= \log \Gamma (S - \alpha_{gt}) - \log \Gamma(K) -\sum_{k=1}^K \log{\Gamma \tilde \alpha_{k}}   + 
 \sum_{k =1}^K ({ \tilde \alpha_{k} -1 }) \bigg[ \psi(\tilde \alpha_{k}) - \psi (S - \alpha_{gt}) \bigg] 
\end{split}
\end{align}

\begin{align*}
    &\frac{\partial \mathcal{L}_{\texttt{reg}}^{\texttt{EDL}}({\bm x},{\bm y})}{\partial \alpha_k} = \frac{\partial }{\partial \alpha_k}\bigg(\log \Gamma (S - \alpha_{gt}) - \log \Gamma(K) -
    \sum_{k=1}^K \log{\Gamma \tilde \alpha_{k}}   + \sum_{k =1}^K ({ \tilde \alpha_{k} -1 }) \bigg[ \psi(\tilde \alpha_{k}) - \psi (S - \alpha_{gt}) \bigg]\bigg) \\
    &=\psi (S - \alpha_{gt}) - \psi(\alpha_{k}) + 
    \frac{\partial}{\partial \alpha_k}\bigg(\sum_{k =1}^K ({ \tilde \alpha_{k} -1 }) \bigg[ \psi(\tilde \alpha_{k}) - \psi (S - \alpha_{gt}) \bigg]\bigg)\\
    &=\psi (S - \alpha_{gt}) - \psi(\alpha_{k}) + \psi(\alpha_{k}) -\psi (S - \alpha_{gt}) +
    ({ \alpha_{k} -1 }) \frac{\partial}{\partial \alpha_k}\bigg( \psi(\tilde \alpha_{k}) - \psi (S - \alpha_{gt}) \bigg) \\
    &=({ \alpha_{k} -1 }) \frac{\partial}{\partial \alpha_k}\bigg( \psi( \alpha_{k}) - \psi (S - \alpha_{gt}) \bigg)
    =({ \alpha_{k} -1 })\big( \psi_1( \alpha_{k}) - \psi_1 (S - \alpha_{gt}) \big)
\end{align*}
Where $\psi_1$ is the trigamma function. Further, using the definition of trigamma function,
\begin{align}
    &\frac{\partial \mathcal{L}_{\texttt{reg}}^{\texttt{EDL}}({\bm x},{\bm y})}{\partial \alpha_k} =({ \alpha_{k} -1 })\big( \psi_1( \alpha_{k}) - \psi_1 (S - \alpha_{gt}) \big)
    = ({ \alpha_{k} -1 })\bigg( \sum_{n=0}^{\infty}\frac{1}{(n+\alpha_k)^2} -  \frac{1}{(n+S - \alpha_{gt})^2} \bigg)
\end{align}
Now, the gradients with respect to the logits $o_k$ becomes
\begin{align}
    &\frac{\partial \mathcal{L}_{\texttt{reg}}^{\texttt{EDL}}({\bm x},{\bm y})}{\partial o_k} =\frac{\partial \mathcal{L}_{\texttt{reg}}^{\texttt{EDL}}({\bm x},{\bm y})}{\partial \alpha_k} \frac{\partial \alpha_k}{\partial o_k}
    = ({ \alpha_{k} -1 })\bigg( \sum_{n=0}^{\infty}\frac{1}{(n+\alpha_k)^2} -  \frac{1}{(n+S - \alpha_{gt})^2} \bigg) \times \frac{\partial e_k}{\partial o_k}
\end{align}
\textbf{Case I:} $\texttt{ReLU}(.)$ to transform logits to evidence.
The gradients with respect to the logits $o_k$ for zero evidence is zero. For all non-zero evidence, the gradient updates as $\frac{\partial e_k}{\partial o_k} = 1 \forall e_k > 0$ and 
\begin{align}
    &\frac{\partial \mathcal{L}_{\texttt{reg}}^{\texttt{EDL}}({\bm x},{\bm y})}{\partial o_k}
    = ({ \alpha_{k} -1 })\bigg( \sum_{n=0}^{\infty}\frac{1}{(n+\alpha_k)^2} -  \frac{1}{(n+S - \alpha_{gt})^2} \bigg)
\end{align}
Now, when $\alpha_k \rightarrow \infty$, the value of the gradient $\frac{\partial \mathcal{L}_{\texttt{reg}}^{\texttt{EDL}}({\bm x},{\bm y})}{\partial o_k} \rightarrow 0$. There is close to zero model update from regularization for very large incorrect evidence.

\textbf{Case II:} $\texttt{SoftPlus}(.)$ to transform logits to evidence. The gradients with respect to the logits $o_k$ is given by the sigmoid i.e. $\frac{\partial e_k}{\partial o_k} = \text{sigmoid}(o_k) \; , \;  \lim_{o_k \rightarrow \infty} \frac{\partial e_k}{\partial o_k} = 1$, and
\begin{align}
    &\frac{\partial \mathcal{L}_{\texttt{reg}}^{\texttt{EDL}}({\bm x},{\bm y})}{\partial o_k}
    = ({ \alpha_{k} -1 })\bigg( \sum_{n=0}^{\infty}\frac{1}{(n+\alpha_k)^2} -  \frac{1}{(n+S - \alpha_{gt})^2} \bigg) \sigma(\alpha_k - 1)
\end{align}
Now, similar to $\texttt{ReLU}$, when $\alpha_k \rightarrow \infty$, the value of the gradient $\frac{\partial \mathcal{L}_{\texttt{reg}}^{\texttt{EDL}}({\bm x},{\bm y})}{\partial o_k} \rightarrow 0$. There is close to zero model update from regularization for very large incorrect evidence.

\textbf{Case III:} $\exp(.)$ to transform logits to evidence. 
When using exponential non-linearity to transform the neural network output to evidence, the $\alpha_k $ is given by $\alpha_k = \exp(o_k) + 1, \frac{\partial \alpha_k}{\partial o_k} = \alpha_k - 1$. Now the gradients with respect to the neural network output $o_k$ becomes:
\begin{align}
    &\frac{\partial \mathcal{L}_{reg}^{2}({\bm x},{\bm y})}{\partial o_k} =\frac{\partial \mathcal{L}_{reg}^{2}({\bm x},{\bm y})}{\partial \alpha_k} \times \frac{\partial \alpha_k }{\partial o_k} 
    = ({ \alpha_{k} -1 })^2\bigg( \sum_{n=0}^{\infty}\frac{1}{(n+\alpha_k)^2} -  \frac{1}{(n+S - \alpha_{gt})^2} \bigg)
\end{align}
Here, the gradient values increase as $\alpha_k \rightarrow \infty$, and the gradient values do not vanish. Simply, as the incorrect evidence becomes very large, the model updates also become large in the accurate direction.

Thus, considering Case I, II, and II, we see that the incorrect evidence-based regularization with forward KL divergence is not effective in regions of incorrect evidence when using $\texttt{ReLu}$ and $\texttt{SoftPlus}$ functions to transform logits to evidence. This issue of correcting very large incorrect evidence does not appear when using $\exp$ function to transform the logits into evidence. 

\item  Gradient for ADL regularization (\cite{shi2020multifaceted} )
\begin{align}
\label{eq:incevreg_apGradAnalysis}
\mathcal{L}_{\texttt{reg}}^{\texttt{ADL}}({\bm x},{\bm y}) =  \sum_{k=1}^{K} \big(\mathbf{e} \odot (\mathbf{1}-\mathbf{y})\big)_k = \sum_{k=1}^K e_k \times (1 - y_k) = S - K - \alpha_{gt} + 1
\end{align}
Considering the gradient of the regularization with respect to the parameters $\alpha_{k}, k \neq gt$, and corresponding logits $o_k$, we get
\begin{align}
    &\frac{\partial \mathcal{L}_{\texttt{reg}}^{\texttt{ADL}}({\bm x},{\bm y}) }{\partial \alpha_k} =1 \quad \implies \frac{\partial \mathcal{L}_{\texttt{reg}}^{\texttt{ADL}}({\bm x},{\bm y}) }{\partial o_k} = \frac{ \partial e_k }{o_k}
\end{align}
When considering the $\exp$ function to transform logits to evidence,  $\frac{ \partial e_k }{o_k} = e_k = \exp(o_k)$ and the gradient value becomes very large when the model's predicted incorrect evidence value is large. This may lead to exploding gradients and stability issues in the model training. For \texttt{ReLU} and \texttt{SoftPlus} functions, the gradients in positive evidence regions are $\frac{ \partial e_k }{o_k} = 1$, and $\frac{ \partial e_k }{o_k} = \sigma(o_k)$ respectably. Thus, the gradient and corresponding model updates for high incorrect evidence are as desired. 

\item{Gradient analysis of incorrect belief regularization term as in Units-ML\cite{Pandey_2022_CVPR} }
    \begin{align}
    \mathcal{L}_{\texttt{reg}}^{\texttt{Units}}({\bm x},{\bm y}) &=  \sum_{k=1}^{K} \big(\frac{\mathbf{e}}{S} \odot (\mathbf{1}-\mathbf{y})\big)_k = \sum_{k=1}^K \frac{e_k}{S} \times (1 - y_k) = \frac{1}{S} \big(S - K - \alpha_{gt} + 1 \big)
    \end{align}

The regularization value is bounded to be in a range of $[0, 1]$ for all the data samples, no matter how severe  the mistake which may limit its effectiveness. Next, the gradient of the regularization with respect to the parameters $\alpha_k$, and logits $o_k$ is given by
\begin{align}
    \frac{\partial \mathcal{L}_{\texttt{reg}}^{\texttt{Units}}({\bm x},{\bm y}) }{\partial \alpha_k} &=\frac{\partial\Big( \frac{1}{S} \big(S - K - \alpha_{gt} + 1 \big)\Big) }{\partial \alpha_k} 
    = \frac{\alpha_{gt} + K  - 1}{S^2} 
    = \frac{e_{gt} + K  }{(K + \sum_{k=1}^K e_k) ^2}
\end{align}
\begin{align}
    \frac{\partial \mathcal{L}_{reg}^3({\bm x},{\bm y})}{\partial o_k} &=\frac{\partial \mathcal{L}_{\texttt{reg}}^{\texttt{Units}}({\bm x},{\bm y}) }{\partial \alpha_k} \times \frac{\partial \alpha_k}{\partial o_k} = \frac{e_{gt} + K }{S^2} \times \frac{\partial e_k}{\partial o_k}
\end{align}
The gradient value decreases as the number of classes $K$ in the classification problem increases. For all three transformations:  $\texttt{ReLU}$, $\texttt{SoftPlus}$, and $\exp$ to transform logits to evidence, the gradients will go to zero as the incorrect evidence increases i.e. $e_k \rightarrow \infty$ and $S\rightarrow\infty \implies \frac{\partial \mathcal{L}_{reg}^3({\bm x},{\bm y})}{\partial o_k} \rightarrow 0$. So, the regularization may be ineffective when the incorrect evidence is very high.
\end{enumerate}

\section{Impact of Non-linear Transformation}\label{sec:impactLogitsEvTranf}
\textbf{Theorem \ref{th:superirorityofExp}: }
{ For a data sample $\mathbf{x}$, if an evidential model outputs logits $\mathbf{o}_k \leq 0 $ $\forall k \in [0, K]$, the exponential activation function leads to a larger gradident update on the model parameters than \texttt{softplus} and \texttt{ReLu}.}
\begin{proof}
{
Consider an evidential loss $\mathcal{L}$, which is formally defined in Eqns. \eqref{eqn:evMSEloss}, \eqref{eqn:evDigammaloss}, and \eqref{eqn:evLogloss}, is used to train the evidential model, let ${\bf o, e} \in \mathbb{R}^K$ denote the neural network output vector before applying the activation $\mathcal{A}$, and the evidence vector, respectively, for a network with weight $w$. For a data sample {$\bf x$}, if the network outputs $o_k<0, \forall k \in [K]$, we have: }

1. \texttt{ReLu}: 
$$\frac{\partial \mathcal{L}_1}{\partial w}=\sum_k\frac{\partial \mathcal{L}_1}{\partial e_k}\frac{\partial e_k}{\partial o_k}\frac{\partial o_k}{\partial w} = 0  \quad \quad \text{(see Eqn. \ref{eqn:reluGradoE}),}
$$
2. \texttt{SoftPlus}: 
$$\frac{\partial \mathcal{L}_2}{\partial w}=\sum_k\frac{\partial \mathcal{L}_2}{\partial e_k}\frac{\partial e_k}{\partial o_k}\frac{\partial o_k}{\partial w} = \sum_k\frac{\partial \mathcal{L}_2}{\partial e_k} \frac{\partial o_k}{\partial w} \text{Sigmoid}(o_k)\quad \quad \text{( see Eqn. \ref{eqn:softplusGradoE})},
$$

3. Exponential: 
$$\frac{\partial \mathcal{L}_3}{\partial w}=\sum_k\frac{\partial \mathcal{L}_3}{\partial e_k}\frac{\partial e_k}{\partial o_k}\frac{\partial o_k}{\partial w} = \sum_k\frac{\partial \mathcal{L}_3}{\partial e_k} \frac{\partial o_k}{\partial w} \exp(o_k)=\sum_k\frac{\partial \mathcal{L}_3}{\partial e_k} \frac{\partial o_k}{\partial w} \{[1 + \exp(o_k)]\text{Sigmoid}(o_k)\} \quad \quad \text{(see Eqn. \ref{eqn:expGradoE})}
$$
{Thus, we have $\frac{\partial \mathcal{L}_3}{\partial w}\geq\frac{\partial \mathcal{L}_2}{\partial w}\geq\frac{\partial \mathcal{L}_1}{\partial w}$, which implies that $\mathcal{A}=\exp$ leads to a larger update to the network than both Softplus and ReLu. This completes the proof. Now we carry out an analysis of the three activations.}
\end{proof}

\textbf{Analysis:}

    Consider a representative $K-$class evidential classification model that trains using Type II Maximum Likelihood evidential loss. Consider an input $\mathbf{x}$ with one-hot label of $\mathbf{y}$, $\sum_{k=1}^K y_k = 1$. For this evidential framework, the Type II Maximum Likelihood loss ($\mathcal{L}^{\texttt{Log}}(\mathbf{x}, \mathbf{y})$) and its gradient with the logits $\mathbf{o}$ ( Eqn. \ref{eq:gradtype2}) are given by
\begin{align}
     \mathcal{L}^{\texttt{Log}}(\mathbf{x}, \mathbf{y}) 
     =  \log S - \sum_{k=1}^K y_k \log \alpha_k \quad \& \quad 
    \text{grad}_k = \frac{\partial \mathcal{L}^{\texttt{Log}}(\mathbf{x}, \mathbf{y}) }{\partial o_k}
    = \Big(\frac{1}{S} - \frac{y_k}{\alpha_k} \Big)\frac{\partial e_k}{\partial o_k}
\end{align}
\textbf{Case I and II:} $\texttt{ReLU(.)}$ and $\texttt{SoftPlus(.)}$ to transform logits to evidence.
\begin{itemize}
    \item \textbf{Zero evidence region:} For $\texttt{ReLU(.)}$ based evidential models, if the logits value for class $k$ i.e. $o_k$ is negative, then the corresponding evidence for class $k$ i.e. $e_k = 0$, $\frac{\partial e_k}{\partial o_k} = 0 
    \; \& \; \text{grad}_k = \frac{\partial \mathcal{L}^{\texttt{Log}}(\mathbf{x}, \mathbf{y})}{\partial o_k}  = 0$. So, there is no update to the model through the nodes that output negative logits value.  
    In the case of $\texttt{SoftPlus}(.)$ based evidential models, there is no update to the model when training samples lie in zero-evidence regions. This is possible in the condition of $o_k \rightarrow - \infty$. In other cases, there will be some small finite small update in the accurate direction from the gradient. 
    \item \textbf{Range of gradients:} The range of gradients for both $\texttt{ReLU(.)}$ and $\texttt{SoftPlus}(.)$ based evidential models are identical. Considering the gradient for the ground truth node $ i.e. y_k = 1$, the range of gradients is $[\frac{1}{K}-1, 0]$. For all other nodes other than the ground truth node i.e. $y_k = 0$, the range of gradients is $[0, \frac{1}{K}]$. So, for classification problems with a large number of classes, the gradient updates to the nodes that do not correspond to the ground truth class will be bounded in a small range and is likely to be very small.
    \item \textbf{High incorrect evidence region:} If the evidence for class $k$ is very large i.e. $e_k \rightarrow \infty$, then for $\texttt{ReLU(.)}$, $\frac{\partial e_k}{o_k} = 1$,  and for $\texttt{SoftPlus(.)}$, $\frac{\partial e_k}{o_k} = \text{Sigmoid}(o_k) \rightarrow 1, \frac{1}{\alpha_k} = \frac{1}{e_k + 1} \rightarrow 0, \frac{1}{S} \rightarrow 0, \; \& \; \text{grad}_k = \frac{\partial \mathcal{L}^{\texttt{Log}}(\mathbf{x}, \mathbf{y})}{\partial o_k}  \rightarrow 0$. For large positive model evidence, there is no update to the corresponding node of the neural network. The evidence can be further broken down into correct evidence (corresponding to the evidence for the ground truth class), and incorrect evidence (corresponding to the evidence for any other class other than the ground truth class). When the correct class evidence is large, the corresponding gradient is close to zero and there is no update to the model parameters which is desired. When the incorrect evidence is large, the model should be updated to minimize such incorrect evidence. However, the evidential models with $\texttt{ReLU}$ and $\texttt{Softplus}$ fail to minimize incorrect evidence when the incorrect evidence value is large. These necessities the need for incorrect evidence regularization terms. 
\end{itemize}

\textbf{Case III:} $\exp(.)$ to transform logits to evidence. Considering Eqn. Eqn. \ref{eq:gradtype2} and Eqn \ref{eqn:gradExpActivation}, the gradient of the loss with respect to the logits becomes 
\begin{align} 
    \text{grad}_k = \frac{\partial \mathcal{L}^{\texttt{Log}}(\mathbf{x}, \mathbf{y})}{\partial o_k} = \Big(\frac{1}{S} - \frac{y_k}{\alpha_k} \Big)(e_k) = \Big(\frac{1}{S} - \frac{y_k}{\alpha_k}\Big) (\alpha_k - 1)
\end{align}
 
\begin{itemize}
    \item \textbf{Zero evidence region:} In case of $\exp(.)$ based evidential models, except in the extreme cases of $\alpha_k \rightarrow \infty$, there will be some signal to guide the model.
    In cases outside the zero-evidence region (i.e. outside $\alpha_k \rightarrow \infty$), there will be some finite small update in the accurate direction from the gradient. Moreover, for same evidence values, the gradient of $\exp$ based model is larger than the $\texttt{SoftPlus}$ based evidential model by a factor of $1 + \exp(o_k)$. Compared to $\texttt{SoftPlus}$ models, the larger gradient is expected to help the model learn faster in low-evidence regions. 
    \item \textbf{Range of gradients:} For the ground truth node $ i.e. y_k = 1$, the range of gradients is $[-1, 0]$. For all nodes other than the ground truth node i.e. $y_k = 0$, the range of gradients is $[0, 1]$. Thus, the gradients are expected to be more expressive and accurate in guiding the evidential model compared to $\texttt{ReLU}$ and $\texttt{SoftPlus}$ based evidential models.
    \item \textbf{High evidence region:} If the evidence for class $k$ is very large i.e. $e_k \rightarrow \infty$, then $\alpha_k - 1 \thickapprox \alpha_k $ and $\text{grad}_k = \texttt{sm}_k - y_k$. In other words, the model's gradient updates become identical to the standard classification model (see Section \ref{sec:appAnalysisStandardClassificationModels}) without any learning issues.
\end{itemize}
Due to smaller zero-evidence region, more expressive gradients, and no issue of learning in high incorrect evidence region, the exponential-based evidential models { are expected to be more effective }compared to $\texttt{ReLU}$ and $\texttt{SoftPlus}$ based evidential models. 
\section{Analysis of Evidential Losses}\label{ap:evLossAnalysis}
Here, we analyze the three variants of evidential loss. As seen in Section \ref{sec:impactLogitsEvTranf}, $\exp$ function is { expected to be} superior to $\texttt{ReLU}$ and $\texttt{SoftPlus}$ functions to transform the logits to evidence. Thus, in this section, we consider $\exp$ function to transform the logits into evidence. However, the analysis holds true for all three functions.
\begin{enumerate}
    \item Bayes risk with the sum of squares loss (Eqn. \ref{eqn:evMSEloss})
    \begin{align}
    \mathcal{L}^{\texttt{MSE}}(\mathbf{x}, \mathbf{y}) &= \sum_{j=1}^K (y_j - \frac{\alpha_j}{S})^2 + \frac{\alpha_j (S - \alpha_j)}{S^2(S+1)}
\end{align}
The loss can be simplified as 
\begin{align}
    \mathcal{L}^{\texttt{MSE}}(\mathbf{x}, \mathbf{y})  &= \sum_{j=1}^K (y_j - \frac{\alpha_j}{S})^2 + \frac{\alpha_j (S - \alpha_j)}{S^2(S+1)}\\
    &= 1 - \frac{2 \alpha_{gt}}{S} + \frac{\sum_k \alpha_k^2}{S^2} + \frac{2\sum_i \sum_j \alpha_i\alpha_j}{S^2(S+1)} \\
    &= 1 - \frac{2 \alpha_{gt}}{S} + \frac{\sum_k \alpha_k^2 + 2\sum_i \sum_j \alpha_i\alpha_j}{S^2} 
    +\frac{2\sum_i \sum_j \alpha_i\alpha_j}{S^2(S+1)} -\frac{2\sum_i \sum_j \alpha_i\alpha_j}{S^2} \\
    &= 2 - \frac{2 \alpha_{gt}}{S} + \frac{2\sum_i \sum_j \alpha_i\alpha_j}{S^2} \Big[\frac{1}{(S+1)} -1 \Big] \\
    &= 2 - \frac{2 \alpha_{gt}}{S} - \frac{2\sum_i \sum_j \alpha_i\alpha_j}{S(S+1)} 
\end{align}
The range of the two components in the loss is $0 \leq \frac{2\alpha_{gt}}{S} + \frac{2\sum_i\sum_j\alpha_i\alpha_j}{S(S+1)} \leq 2$ and the loss is bounded in the range $[0, 2]$. In other words, the loss for any sample in the entire sample space is bounded in the range of $[0, 2]$ no matter how severe the mistake is. Such bounded loss is expected to restrict the model's learning capacity.
    \item Bayes risk with cross-entropy loss (Eqn. \ref{eqn:evDigammaloss})
\begin{align}
    \mathcal{L}^{\texttt{CE}}(\mathbf{x}, \mathbf{y}) &= \sum_{j=1}^K y_k \Big( \Psi(S) - \Psi(\alpha_k) \Big) =  \Psi(S) - \Psi(\alpha_{gt})
\end{align}
Where $\Psi(.)$ is the Digamma function, and $\Gamma$ is the gamma function. The functions and their gradients are defined as 
\begin{align}
    \Gamma(z) &= \frac{e^{-\gamma z}}{z} \prod_{n=1}^{\infty} \Big(1 + \frac{z}{n}\Big)^{-1} e^{\frac{z}{n}}\\
    \Psi(z) &= \frac{d}{dz} \log \Gamma(z) = \frac{d}{dz} \bigg( -\gamma z - \log z + \sum_{n=1}^{\infty}\Big( \frac{z}{n} - \log \big(1 + \frac{z}{n}\big)\Big) \bigg) \\
    &= -\gamma -\frac{1}{z} + \sum_{n=1}^{\infty} \frac{1}{n} - \frac{1}{n+z} \\
    \frac{\partial \Psi(z)}{\partial z} &= \frac{\partial }{\partial z} \bigg( -\gamma -\frac{1}{z} + \sum_{n=1}^{\infty} \frac{1}{n} - \frac{1}{n+z} \bigg) = \frac{1}{z^2} + \sum_{n=1}^{\infty}\frac{1}{(n+z)^2}
\end{align}
    Now, the Bayes risk with cross-entropy loss becomes
    \begin{align}
    \mathcal{L}^{\texttt{CE}}(\mathbf{x}, \mathbf{y}) &=  \Psi(S) - \Psi(\alpha_{gt}) \\
    &= \frac{1}{\alpha_{gt}} - \frac{1}{S} + 
    S\sum_{n=1}^{\infty} \frac{1}{n(n+S)} - \alpha_{gt} \sum_{n=1}^{\infty} \frac{1}{n(n+\alpha_{gt})}
\end{align}
Both the infinite sums ($\sum_{n=1}^{\infty} \frac{1}{n(n+S)}$ and $\sum_{n=1}^{\infty} \frac{1}{n(n+\alpha_{gt})}$) converge and lie in the range of $0$ to $\frac{\pi^2}{6}$. The minimum possible value of this loss is 0 when $\alpha_{gt} \rightarrow \infty \& S \thickapprox \alpha_{gt}$. The maximum possible value is $\infty$ when only $S \rightarrow \infty$. The loss lies in the range $[0, \infty]$ and is more expressive compared to MSE-based evidential loss. 

Considering the gradient of the loss with respect to the ground truth node (i.e. $\alpha_{gt}, y_{gt}= 1$),

\begin{align}
    \frac{\partial  \mathcal{L}^{\texttt{CE}}(\mathbf{x}, \mathbf{y}) }{\partial \alpha_{gt}} &= \frac{\partial   }{\partial \alpha_{gt}} \Psi(S) - \Psi(\alpha_{gt}) = \frac{1}{S^2} + \sum_{n=1}^{\infty}\frac{1}{(n + S)^2} - \frac{1}{\alpha_{gt}^2} - \sum_{n=1}^{\infty}\frac{1}{(n + \alpha_{gt})^2}
\end{align}
As $\alpha_{gt} < S$, the gradient is always negative. Thus, the model aims to maximize the correct evidence $\alpha_{gt}$.
Considering the gradient of the loss with respect to nodes not corresponding to the ground truth (i.e. $\alpha_k, k \neq gt, y_k = 0$), \begin{align}
    \frac{\partial  \mathcal{L}^{\texttt{CE}}\mathbf{x}, \mathbf{y}) }{\partial \alpha_k} &= \frac{\partial   }{\partial \alpha_k} \Psi(S) - \Psi(\alpha_{gt}) =  \frac{\partial   \Psi(S)  }{\partial S} \frac{\partial   S  }{\partial \alpha_k} = \frac{1}{S^2} + \sum_{n=1}^{\infty}\frac{1}{(n + S)^2} \\
    \frac{\partial  \mathcal{L}^{\texttt{CE}}\mathbf{x}, \mathbf{y}) }{\partial o_k} &= \frac{\partial  \mathcal{L}^{\texttt{CE}}\mathbf{x}, \mathbf{y}) }{\partial \alpha_k} \times \frac{\alpha_k}{o_k} = \Big(\frac{1}{S^2} + \sum_{n=1}^{\infty}\frac{1}{(n + S)^2} \Big) (\alpha_k - 1)
\end{align}
The gradient at nodes that do not correspond to ground truth is always non-negative. However, this gradient is also minimum and 0 when $S  \rightarrow \infty\, \& \, \alpha_k \rightarrow \infty$. This is an undesired behavior as the model may be encouraged to always increase the evidence for all the classes. Moreover, the gradient is zero and there is no update to the nodes when $S \rightarrow \infty, \& \, \alpha_k \rightarrow \infty$. So, the incorrect evidence regularization to penalize the incorrect evidence is essential for the evidential model trained with this loss. 
    \item Type II Maximum Likelihood loss (Eqn. \ref{eqn:evLogloss})
\begin{align}
    \mathcal{L}^{\texttt{Log}}(\mathbf{x}, \mathbf{y}) = \sum_{k=1}^K y_k \Big( \log(S) - \log (\alpha_k) \Big) =\log(S) - \log (\alpha_{gt}) 
\end{align}
The loss is bounded in the range of $[0,\infty]$ as the loss is minimum and $0$ when $\alpha_{gt} \rightarrow S \rightarrow \infty$, and maximum loss when $\alpha_{gt} << S \, \& \, S \rightarrow \infty$. Thus, the loss is more expressive compared to MSE based evidential loss.
Now, the gradient of the loss is given by 
\begin{align} 
\frac{\partial \mathcal{L}^{\texttt{Log}}(\mathbf{x}, \mathbf{y})}{\partial o_k} 
    = \frac{1}{S}\frac{\partial S}{\partial o_k} - y_k \frac{1}{\alpha_k} \frac{\partial \alpha_k}{\partial o_k} 
    = \Big(\frac{1}{S} - \frac{y_k}{\alpha_k} \Big)\frac{\partial e_k}{\partial o_k} = \Big(\frac{1}{S} - \frac{y_k}{\alpha_k} \Big)(\alpha_k - 1)
\end{align}
Here, when $S  \rightarrow \infty\, \& \, \alpha_k \rightarrow \infty$, the gradient becomes $\frac{\partial \mathcal{L}^{\texttt{Log}}(\mathbf{x}, \mathbf{y})}{\partial o_k} \rightarrow (1 - y_k)$. This is highly desirable behavior for the model as it aims to minimize the evidence for the incorrect class and there will be no update to the node corresponding to the ground truth class if $\alpha_{k} = \alpha_{gt}, y_{gt} = 1$. Thus, the Type II based issue is expected to be superior to the other two losses as the range of loss is optimal (i.e. in the range $[0, \infty]$), and no learning issue arises for samples with high incorrect evidence. 
\end{enumerate}

\section{Additional Experiments and Results}\label{ap:AdditionalExpResults}
We first present the details of the models, hyperparameter settings, {clarification regarding dead neuron issue}, and experiments used in the work in Section \ref {apsubsec:hyperparameterdetails}. We then present additional results and discussions, including Few-shot classification, {and 200-class tiny-ImageNet Classification} results, that show the effectiveness of the proposed model RED in Section \ref{ap:EffectivenessOfRed}. Finally discuss some limitations and potential future works in Section \ref{ap:limitationsAndFutureWorks}.

\subsection{Hyperparameter details} 
\label{apsubsec:hyperparameterdetails}
For Table \ref{tab:ClassificationPerformanceCOmparison} results, $\lambda_1 = 1.0$ was used for MNIST experiments, $\lambda_1 = 0.1$ was used for Cifar10 experiments, and $\lambda_1 = 0.001$ was used for Cifar100 experiments. Table \ref{ap:CompleteMNISTResults}, \ref{ap:CompleteCifar10Results}, and \ref{ap:CompleteCifar100Results} present complete results across the hyperparameter values and experiment settings. MNIST model was trained on the LeNet model \cite{sensoy2018evidential} for $50$ epochs, and Cifar10/Cifar100 models were trained on Resnet-18 based classifier \cite{he2016deep} for 200 epochs. Few-shot classification experiments were carried out with $\lambda_1 = 0.1$ using Resnet-12 based classifier \cite{chen2021meta}. All results presented in this work are from local reproduction. MNIST models were trained with learning rate of 0.0001 and Adam optimizer \cite{kingma2014adam}, and all remaining models were trained with learning rate of 0.1 and Stochastic Gradient Descent optimizer with momentum. Tabular results represent the mean and standard deviation from 3 independent runs of the model. In the proposed model RED, correct evidence regularization is weighted by the parameter $\lambda_{\texttt{cor}}$ whose value is given by the predicted vacuity $\nu$. $\lambda_{\texttt{cor}}$ is treated as hyperparameter, \ie constant weighting term in the loss during model update.

\subsection{Dead Neuron Issue Clarification}
{
Instead of using ReLU as an activation function in a standard deep neural network, evidential models introduce ReLU as non-negative transformation function in the output layer to ensure that the predicted evidence is non-negative to satisfy the requirement of evidential theory. This non-negative evidence vector parameterizes a Dirichlet prior for fine-grained uncertainty quantification that covers second-order uncertainty, including vacuity and dissonance. We theoretically and empirically show the learning deficiency of ReLU based evidential models and justify the advantage of using an exponential function to output (non-negative) evidence. We further introduce a correct evidence regularization term in the loss that addresses the learning deficiency from zero-evidence samples. The ``dead neuron” issue in the activation functions has been studied, and ReLU variations such as Exponential Linear Unit, Parametric ReLU, and Leaky ReLU have been developed to address the issue. But, these activation functions will not be theoretically sound in the evidential framework as
they are can lead to negative evidences. In this case, they can not serve as Dirichlet parameters that are interpreted as pseudo counts.}

\subsection{Effectiveness of Regularized Evidential Model  (RED)}\label{ap:EffectivenessOfRed}
\subsubsection{Evidential Activation Function.}
In this section, we present additional results (for section \ref{subsec:impactActivationFunction}) with the MNIST classification problem using the LeNet model to empirically validate Theorem \ref{th:superirorityofExp}. We carry out experiments for evidential models trained using all three evidential losses: Evidential MSE loss in \eqref{eqn:evMSEloss}, Evidential cross-entropy loss in \eqref{eqn:evDigammaloss}, and Evidential Log loss in \eqref{eqn:evLogloss} with $\lambda_1 = \{0.0, 1.0, \& 10.0\}$.  As can be seen in Figure \ref{fig:appCorEvRegImpactMSELoss}, \ref{fig:appCorEvRegImpactDigammaLoss}, and \ref{fig:appCorEvRegImpactLogLoss}, using $\exp$ activation for transforming logits to evidence leads to superior performance in all settings compared to $\texttt{ReLU}$ and $\texttt{Softplus}$ based evidential models that empirically validates Theorem \ref{th:superirorityofExp}. 
\begin{figure}[ht!] 
\centering
\subfigure[Trend for $\lambda_1 = 0.0$]{
  \includegraphics[width=0.30\linewidth]{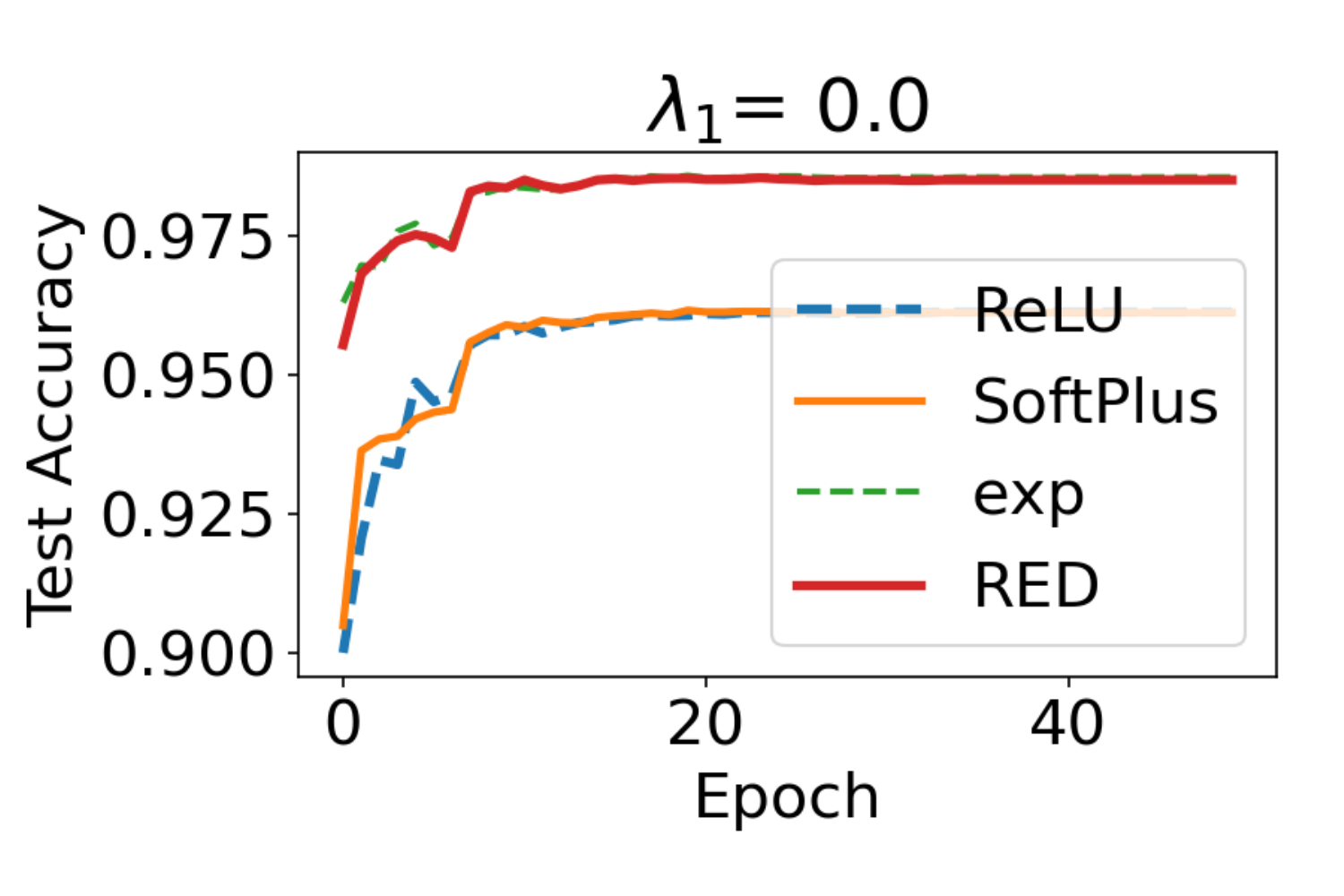}
}
\subfigure[Trend for $\lambda_1 = 1.0$]{
  \includegraphics[width=0.30\linewidth]{images/Mnist/impactActivationNew/new_Evid_act_mnist_exp_act_mse_lss_True_drp_1.0_kl.pdf}
}
\subfigure[Trend for $\lambda_1 = 10.0$]{
  \includegraphics[width=0.30\linewidth]{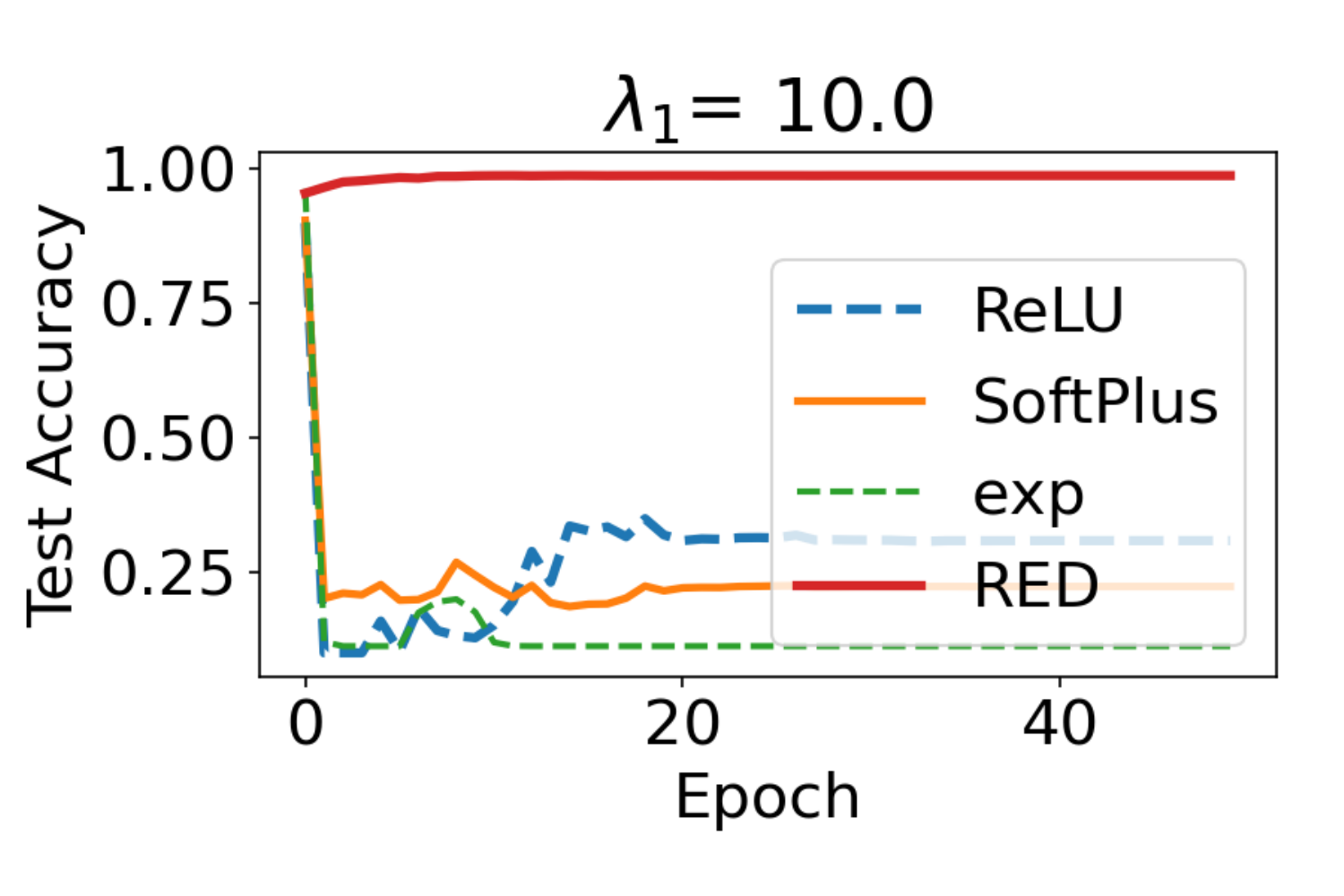}
}
\caption{Impact of Evidential Activation to the test set accuracy of the model trained with MSE based evidential loss (Eqn. \ref{eqn:evMSEloss})}
\label{fig:appCorEvRegImpactMSELoss}
\end{figure} 

\begin{figure}[ht!] 
\centering
\subfigure[Trend for $\lambda_1 = 0.0$]{
  \includegraphics[width=0.30\linewidth]{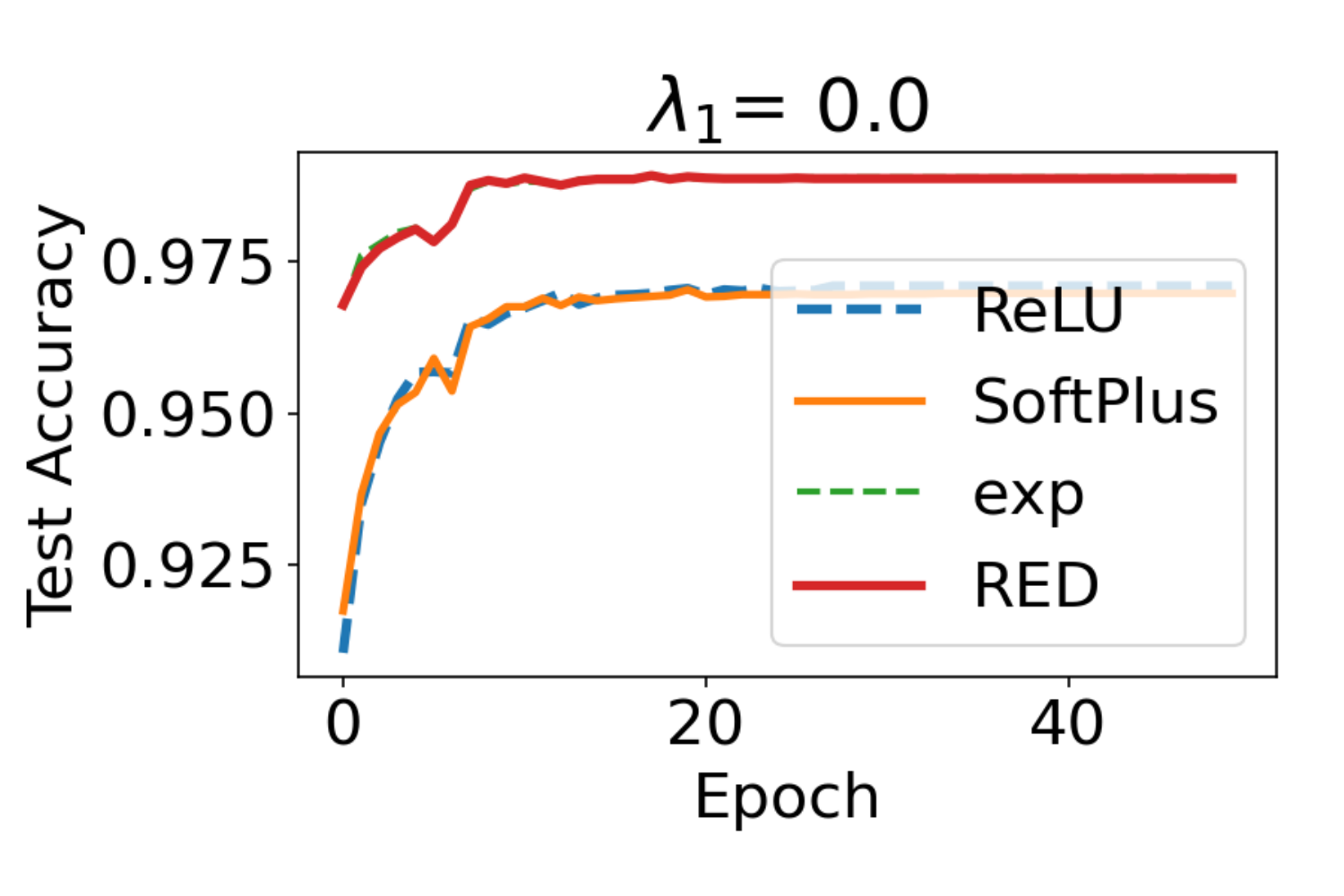}
}
\subfigure[Trend for $\lambda_1 = 1.0$]{
  \includegraphics[width=0.30\linewidth]{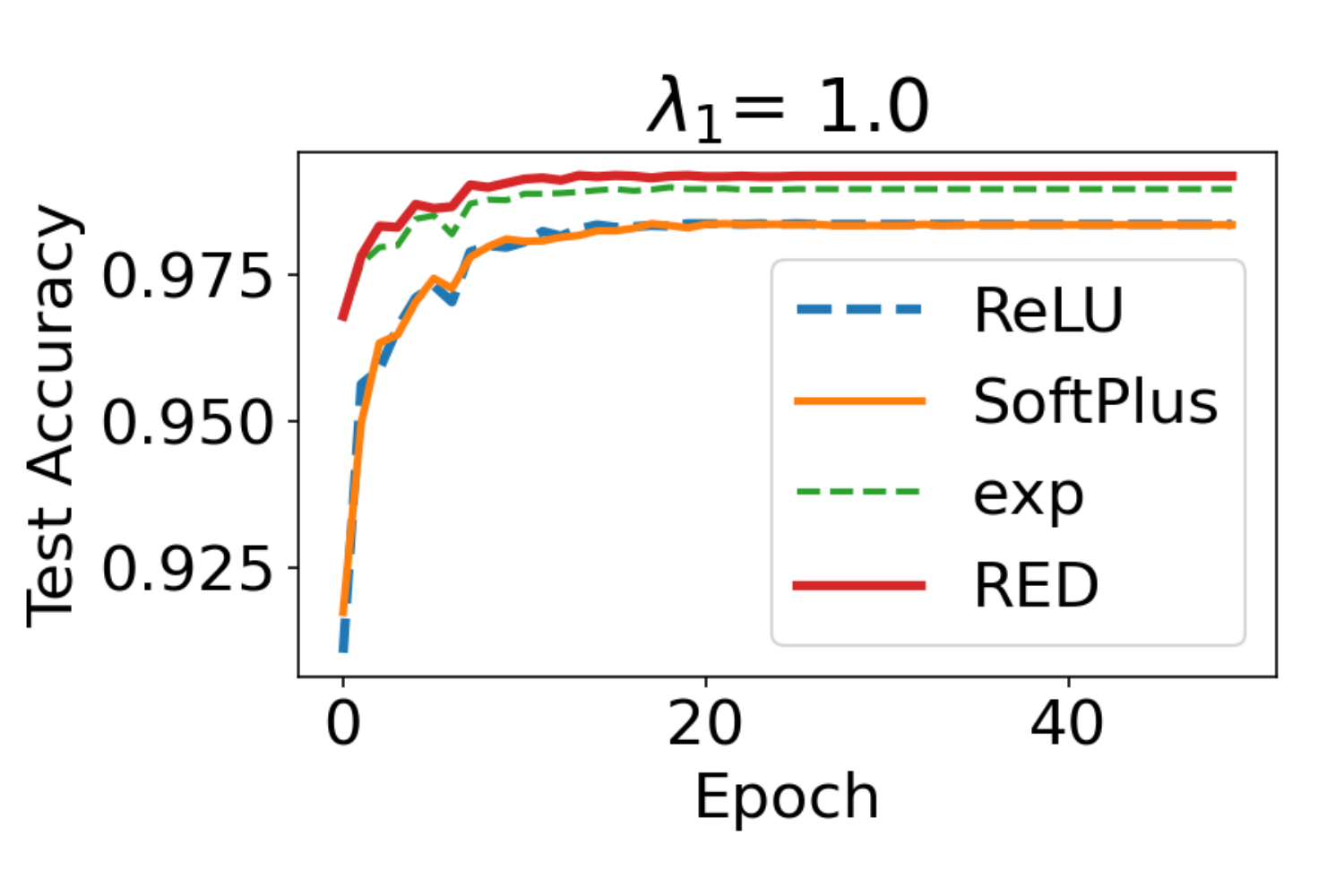}
}
\subfigure[Trend for $\lambda_1 = 10.0$]{
  \includegraphics[width=0.30\linewidth]{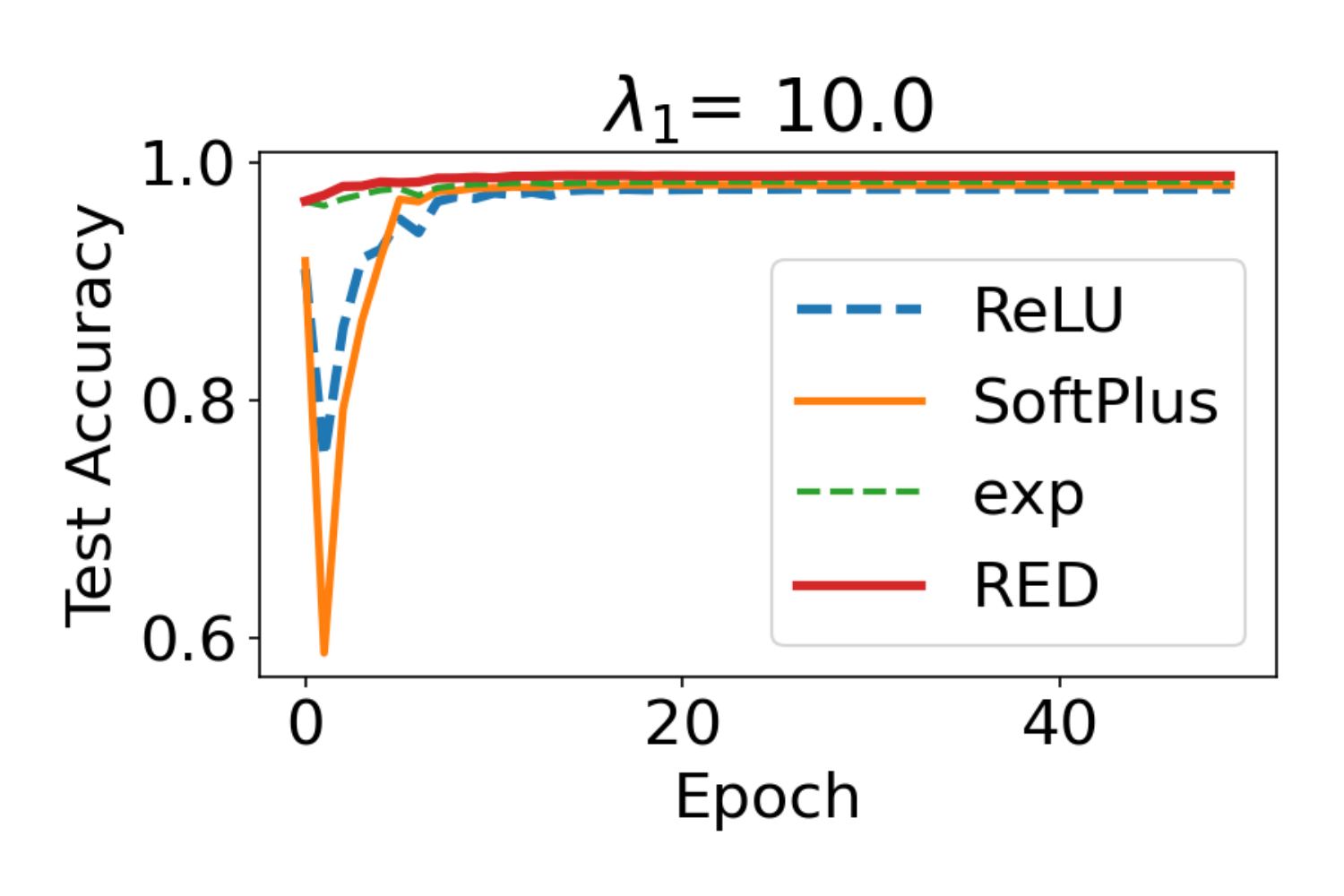}
}
\caption{Impact of Evidential Activation to test set accuracy of the model trained with cross-entropy based evidential loss (Eqn. \ref{eqn:evDigammaloss})}
\label{fig:appCorEvRegImpactDigammaLoss}
\vspace{-5mm}
\end{figure} 
\begin{figure}[ht!] 
\centering
\subfigure[Trend for $\lambda_1 = 0.0$]{
  \includegraphics[width=0.30\linewidth]{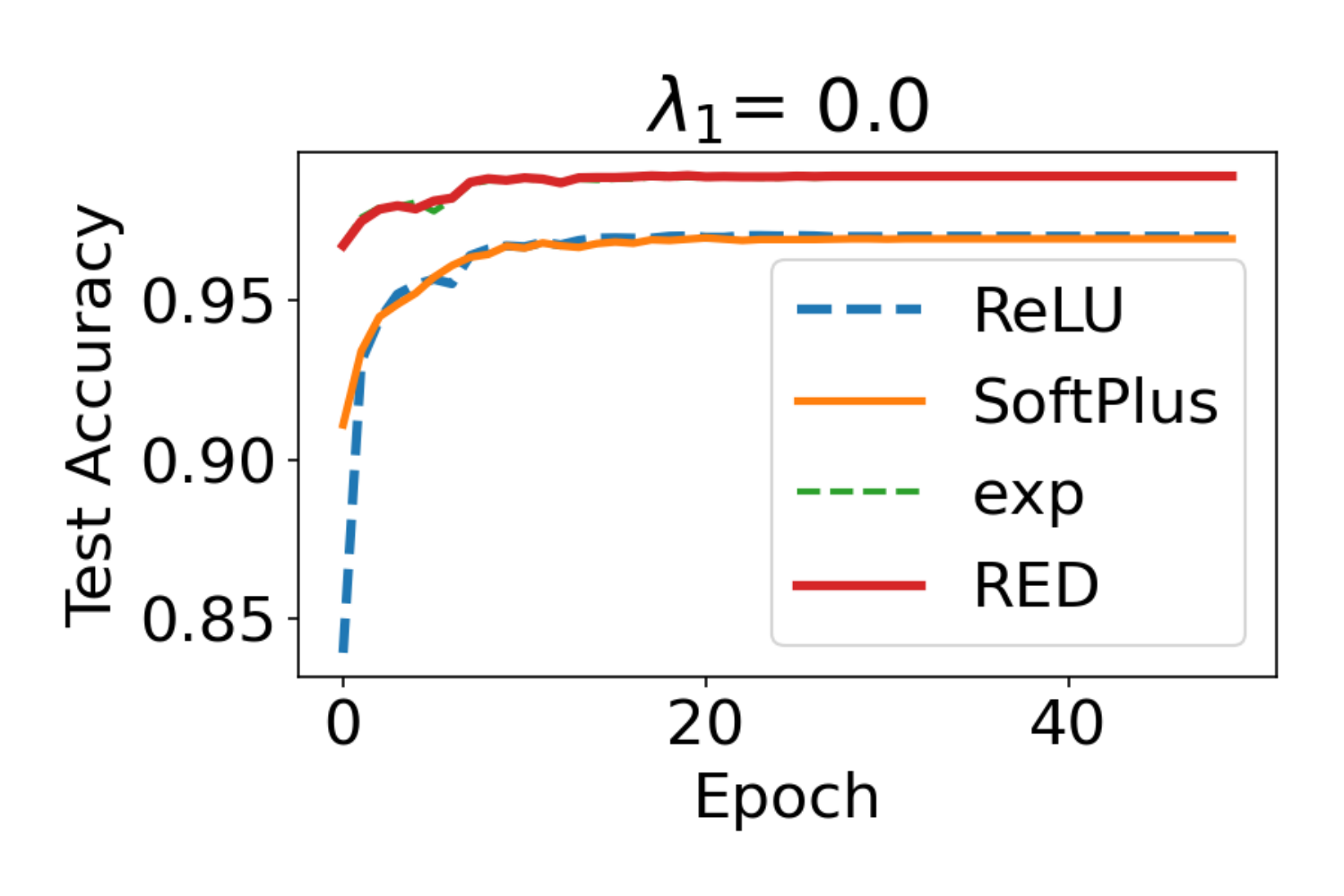}
}
\subfigure[Trend for $\lambda_1 = 1.0$]{
  \includegraphics[width=0.30\linewidth]{images/Mnist/impactActivationNew/new_Evid_act_mnist_exp_act_log_lss_True_drp_1.0_kl.pdf}
}
\subfigure[Trend for $\lambda_1 = 10.0$]{
  \includegraphics[width=0.30\linewidth]{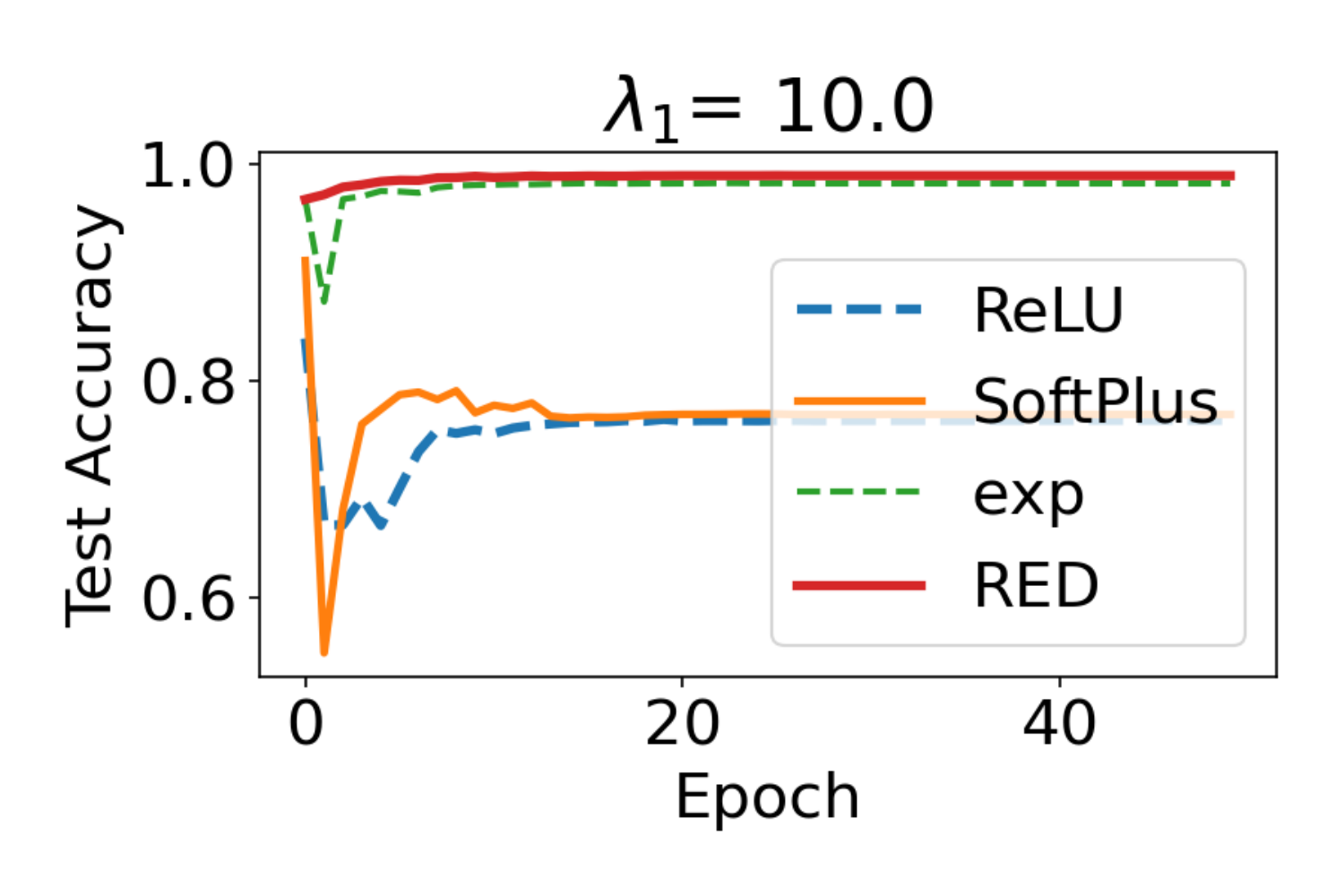}
}
\caption{Impact of Evidential Activation to the test set accuracy of the model trained with Type II based evidential loss (Eqn. \ref{eqn:evLogloss})}
\label{fig:appCorEvRegImpactLogLoss}
\end{figure} 

\subsubsection{Correct Evidence Regularization} \label{app:secImpactOfCorrectEvReg}
We introduce the novel correct evidence regularization term to train the evidential model (Section \ref{sec:evModelTrainingLoss}). 
In this section, we present additional results for the evidential model that uses $\exp$ activation. We trained the model using  evidential losses with different incorrect evidence regularization strengths ( $\lambda_1 = 0, \;  1.0 \; \& \; 10.0$). As can be seen( Figure \ref{fig:appCorEvRegImpactMSE}, and  \ref{fig:appCorEvRegImpactDig}), the model with proposed correct-evidence regularization leads to improved generalization compared to the baseline model as the proposed correct-evidence regularization term enables the evidential model to learn from zero-evidence samples instead of ignoring them. Moreover, even though strong incorrect evidence regularization hurts both model's generalization, the proposed regularization leads to a more robust model that generalizes better. Finally, the MSE-based evidential model is hurt the most with strong incorrect evidence regularization as thee MSE based evidential loss is bounded in the range $[0, 2]$, and the incorrect evidence-regularization term may easily dominate the overall loss compared to other evidential losses. This can be seen in Figure \ref{fig:appCorEvRegImpactMSE}(c) where the incorrect evidence regularization strength is large i.e. $\lambda_1 = 10.0$ and the evidential model fails to train. Due to strong incorrect evidence regularization, the model may have learned to map all training samples to zero-evidence region. However, with the proposed regularization, the model continues to learn and achieves good generalization performance.  
\begin{figure}[ht!] 

\centering
\subfigure[Trend for $\lambda_1 = 0.0$]{
  \includegraphics[width=0.30\linewidth]{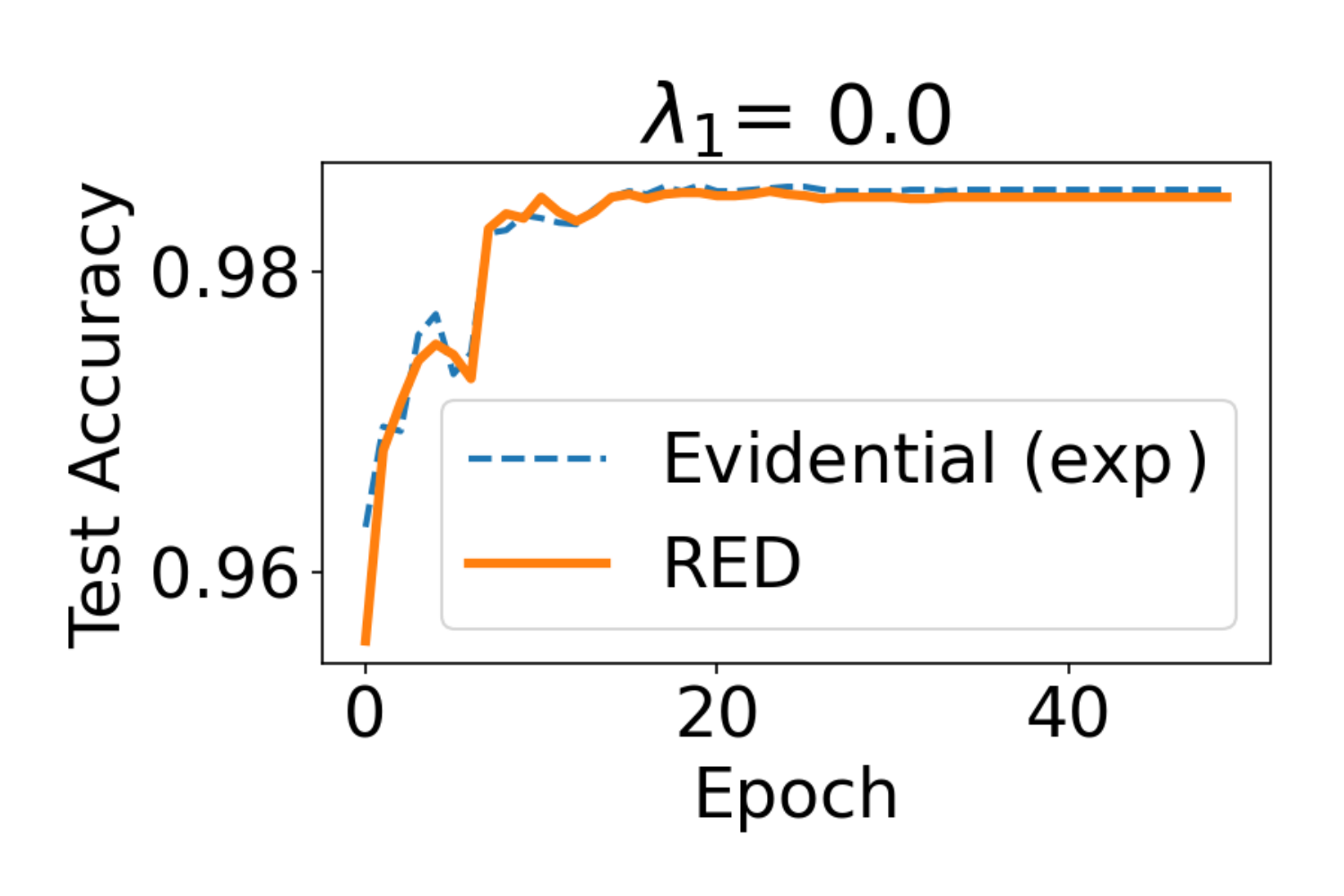}
}
\subfigure[Trend for $\lambda_1 = 1.0$]{
  \includegraphics[width=0.30\linewidth]{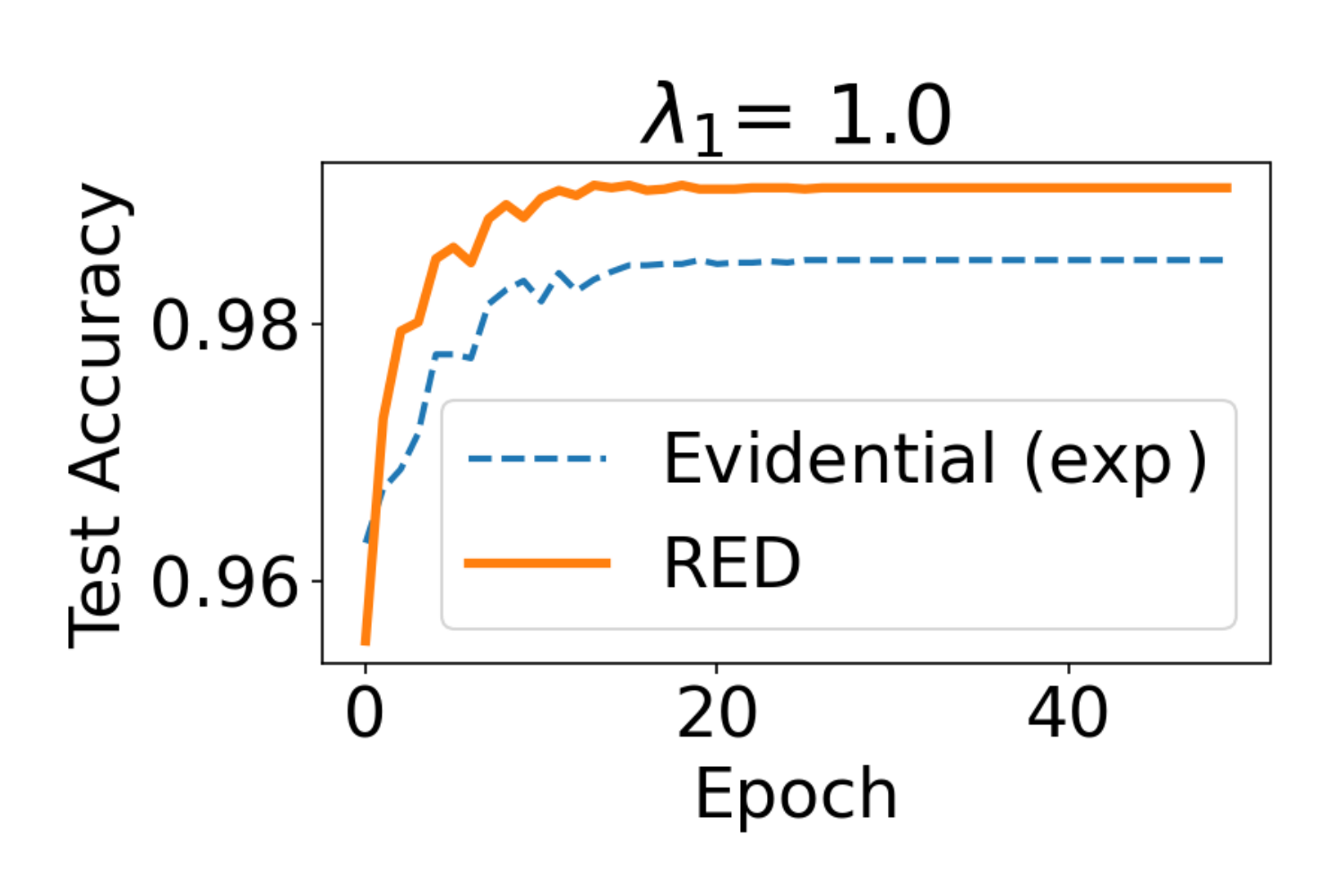}
}
\subfigure[Trend for $\lambda_1 = 10.0$]{
  \includegraphics[width=0.30\linewidth]{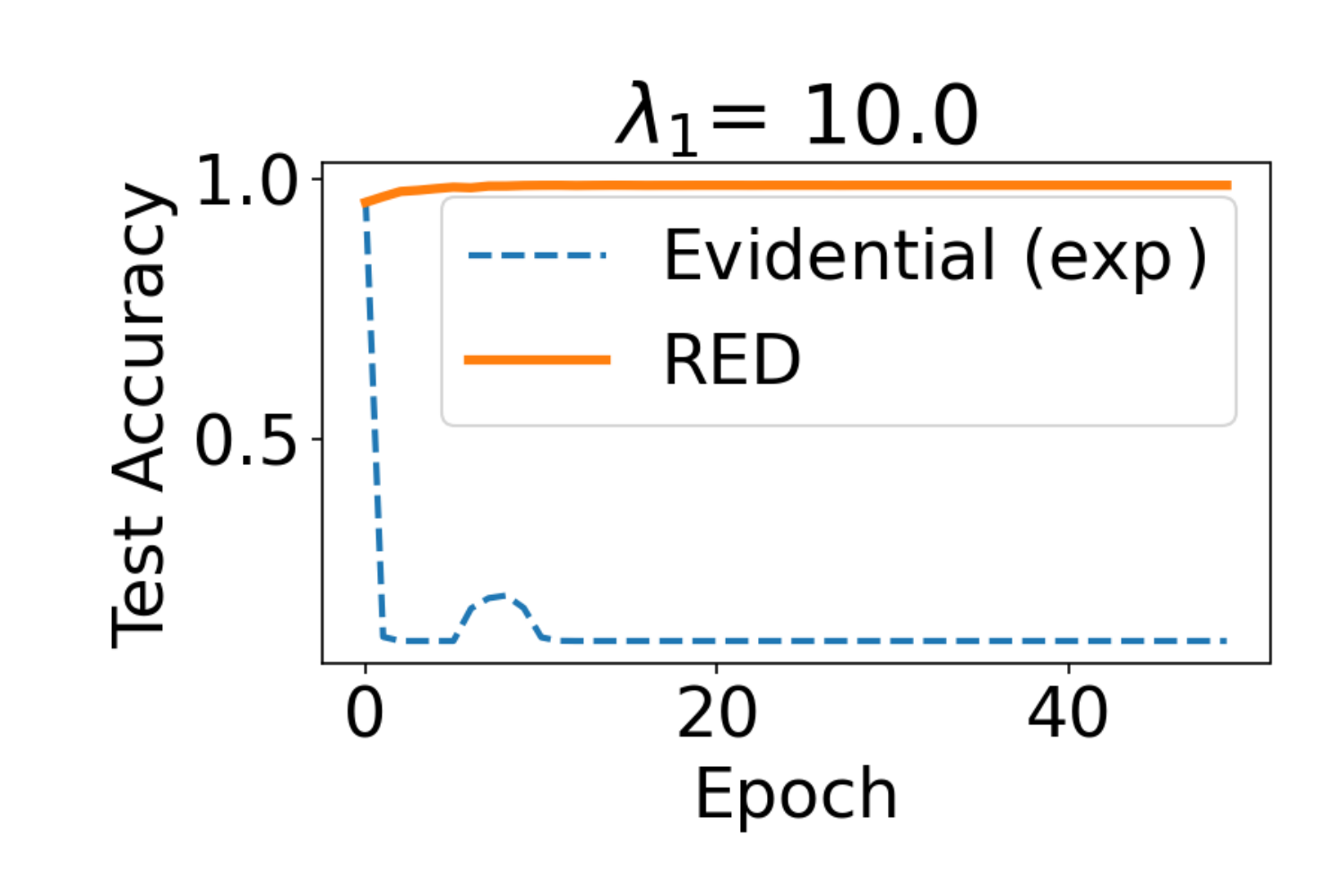}
}
\caption{Impact of proposed Correct Evidence Regularization to the test set accuracy of the evidential model( Trained with Eqn. \ref{eqn:evMSEloss})}
\label{fig:appCorEvRegImpactMSE}
\end{figure} 
\begin{figure}[ht!] 
\centering
\subfigure[Trend for $\lambda_1 = 0.0$]{
  \includegraphics[width=0.30\linewidth]{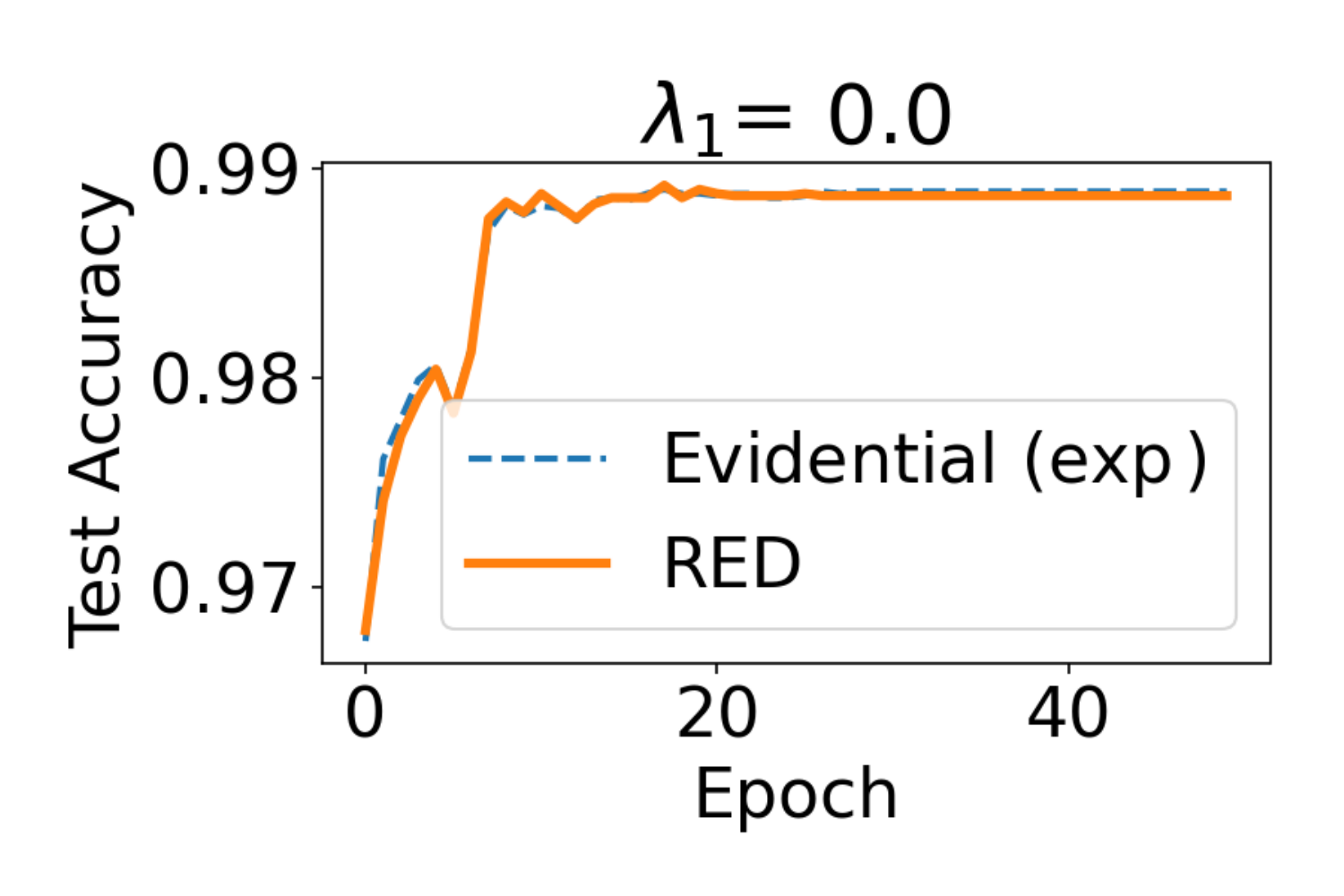}
}
\subfigure[Trend for $\lambda_1 = 1.0$]{
  \includegraphics[width=0.30\linewidth]{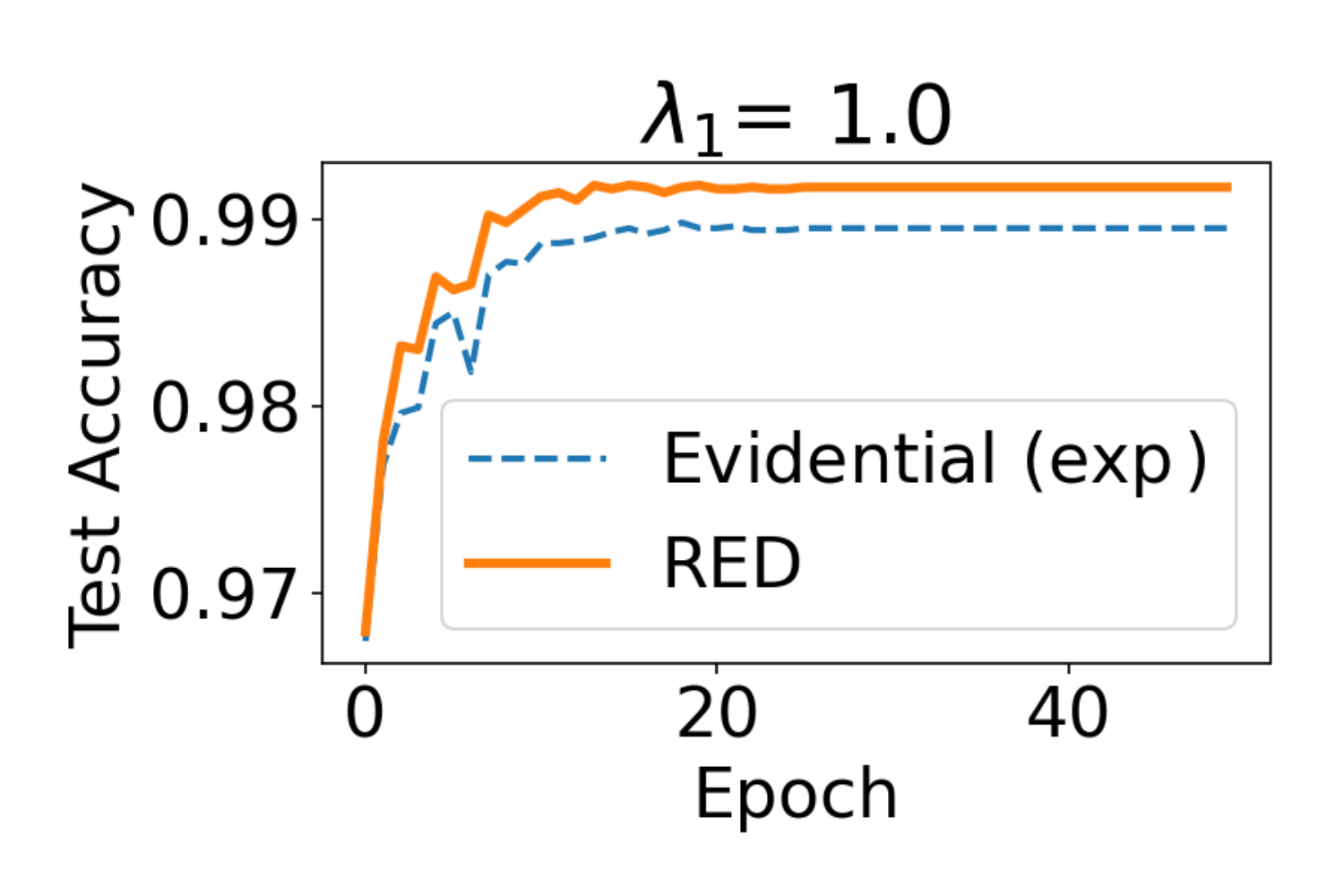}
}
\subfigure[Trend for $\lambda_1 = 10.0$]{
  \includegraphics[width=0.30\linewidth]{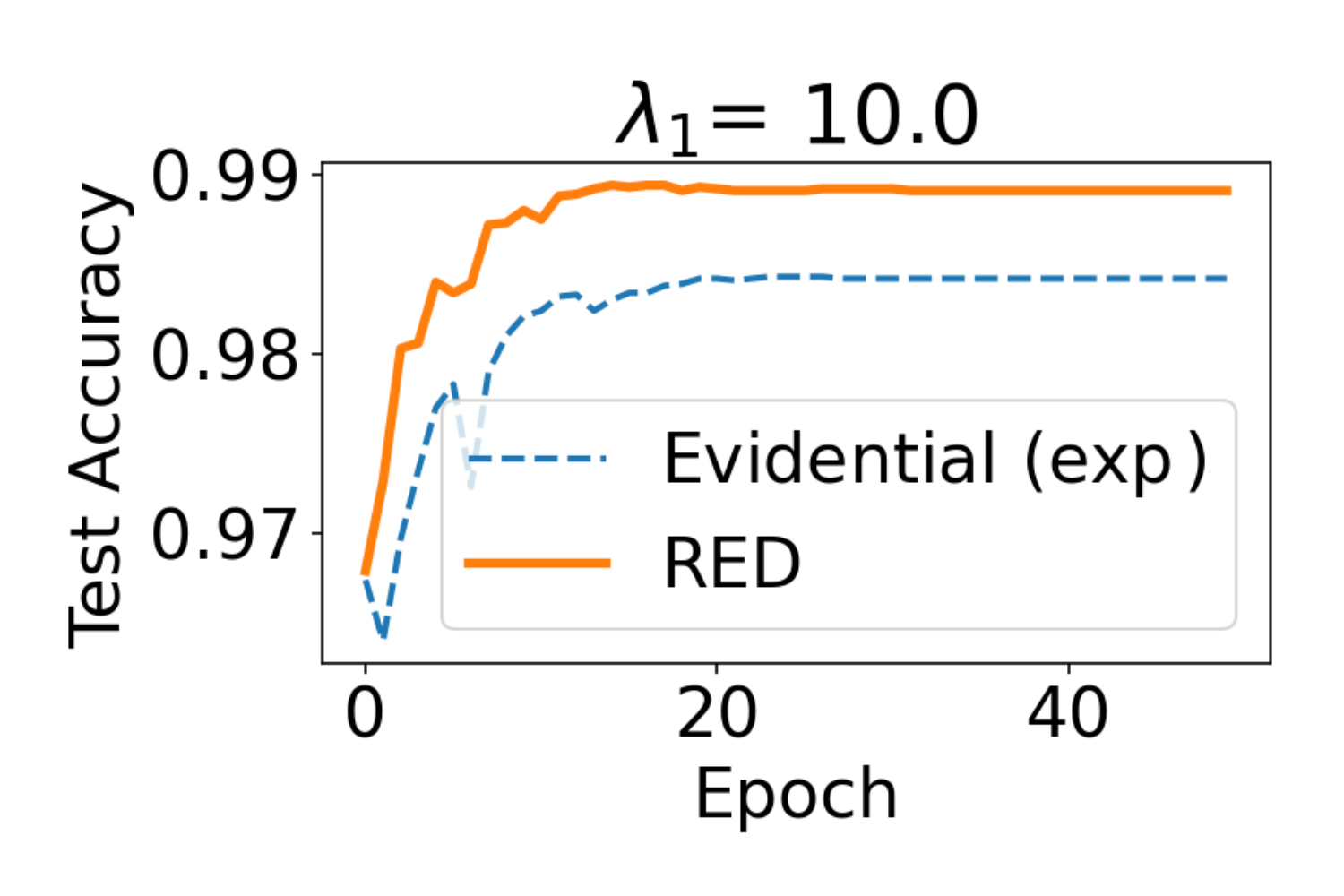}
}
\caption{Impact of proposed Correct Evidence Regularization to the test set accuracy of the evidnetial model (Trained with Eqn. \ref{eqn:evDigammaloss})}
\label{fig:appCorEvRegImpactDig}
\end{figure} 

\subsubsection{Few-Shot Classification Experiments}
\label{subsec:FewShotLearningSetting}
Ideas presented in this work address the fundamental limitation of evidential classification framework that enables the evidential model to acquire knowledge from all the training samples. Using these ideas, evidential framework can be extended to challenging classification problems to the reasonable predictive performance. To this end, we experiment with few-shot classification using $1$-shot and $5$-shot classification for the \textit{mini}-ImageNet dataset \cite{vinyals2016matching}. We consider the ResNet-12 backbone, classifier-baseline model \cite{chen2021meta}, and its evidential extension. Table \ref{tab:fewShotClassification} shows the results for $1$-shot and $5$-shot classification experiments. As can be seen, the \texttt{ReLU} and \texttt{Softplus} based evidential models have suboptimal performance as they avoid many training samples of the zero-evidence region. In contrast, the $\exp$ model has a better learning capacity that leads to superior performance. Finally, the proposed model RED can learn from all training samples, which leads to the best generalization performance among all the evidential models. 

\begin{table}[ht]
\centering
\small
    \caption{Few-Shot Classification Accuracy comparison: \textit{mini}-ImageNet dataset
    }
\begin{tabular}{|p{0.66\textwidth}|}
    \hline
    \centering
    Standard CE Model: $1$ Shot: $57.9_{\pm0.2}\%$; $5$-Shot: $76.9_{\pm 0.2}\%$
\end{tabular}
\begin{tabular}{|p{0.66\textwidth}|}
    \hline
    $1$-Shot Experiments\\
    \hline
\end{tabular}

\begin{tabular}{|p{0.12\textwidth}|p{0.11\textwidth}|p{0.11\textwidth}|p{0.11\textwidth}|p{0.11\textwidth}|}

\hline
Regularization & \texttt{ReLU} &\texttt{SoftPlus} & $\exp$  &\textbf{RED (Ours)}\\
\hline
 $\lambda_1 = 0.000$    &$38.78_{\pm3.75}$&$51.60_{\pm0.40}$&$57.11_{\pm0.09}$&$56.27_{\pm0.15}$\\
  $\lambda_1 = 0.100$   &$31.15_{\pm1.69}$&$48.87_{\pm0.21}$&$56.43_{\pm0.03}$&$\bf{58.03_{\pm0.39}}$\\
   $\lambda_1 = 1.000$ &$20.00_{\pm0.00}$&$43.81_{\pm0.56}$&$27.43_{\pm0.88}$&$54.68_{\pm0.45}$\\
\hline
\end{tabular}

\begin{tabular}{|p{0.66\textwidth}|}
    \hline
    $5$-Shot Experiments\\
    \hline
\end{tabular}

\begin{tabular}{|p{0.12\textwidth}|p{0.11\textwidth}|p{0.11\textwidth}|p{0.11\textwidth}|p{0.11\textwidth}|}

\hline
Regularization & \texttt{ReLU} &\texttt{SoftPlus} & $\exp$  &\textbf{Ours}\\
\hline
 $\lambda_1 = 0.000$    &$52.66_{\pm5.32}$&$67.22_{\pm0.17}$&$75.87_{\pm0.09}$&$75.31_{\pm0.13}$\\
  $\lambda_1 = 0.100$   &$43.95_{\pm3.72}$&$66.14_{\pm0.05}$&$74.08_{\pm0.13}$&$\bf{76.05_{\pm0.17}}$\\
   $\lambda_1 = 1.000$ &$20.00_{\pm0.00}$&$61.96_{\pm0.61}$&$34.01_{\pm1.46}$&$72.32_{\pm0.20}$\\
\hline
\end{tabular}
\label{tab:appFewShot}
\label{tab:fewShotClassification}
\end{table}

\subsubsection{Complex Dataset/Model Experiments}
{We also carry out experiment for a challenging 200-class classification problem over Tiny-ImageNet based on \cite{huynh2022vision}. We adapt the Swin Transformer to be
evidential, and train all the models for 20 epochs with Evidential log loss (Eqn. \ref{eqn:evLogloss}). In this setting, ReLU based evidential
model achieves 85.25\% accuracy, softplus based model achieves 85.15 \% accuracy, the exponential model improves over
both to achieve 89.93 \% accuracy, and our proposed model RED outperforms all the evidential models to achieve the greatest
accuracy of 90.14\%, empirically validating our theoretical analysis.}

\subsection{Limitations and Future works} \label{ap:limitationsAndFutureWorks}
We carried out a theoretical investigation of the Evidential Classification models to identify their fundamental limitation: their inability to learn from \textit{zero evidence regions}. The empirical study in this work is based on classification problems. We next plan to extend the ideas to develop Evidential Segmentation and Evidential Object Detection models. Moreover, this work identifies limitations of Evidential MSE loss in \eqref{eqn:evMSEloss}, and we plan to carry out a thorough theoretical analysis to analyze other evidential losses given in \eqref{eqn:evLogloss} and \eqref{eqn:evDigammaloss}). The proposed evidential model, similar to existing evidential classification models, requires hyperparameter tuning for $\lambda_1$ i.e. the incorrect evidence regularization hyperparameter.

{In addition, extending evidential models to noisy and incomplete data settings and investigating the benefits of leveraging uncertainty
information could be interesting future work. Finally, It will be an interesting future work to extend the analysis and evidential models to tasks beyond classification, for instance to build effective evidential segmentation and object detection models.}

\begin{table}[ht]
\centering
\small
    \caption{Classification performance comparison: MNIST dataset
    }
    \label{tab:appMnistRes}
\begin{tabular}{|p{0.66\textwidth}|}
    \hline
    \centering
    Standard CE Model: $99.21_{\pm 0.03}$\%
\end{tabular}
\begin{tabular}{|p{0.66\textwidth}|}
    \hline
    Log loss\\
    \hline
\end{tabular}
\begin{tabular}{|p{0.12\textwidth}|p{0.11\textwidth}|p{0.11\textwidth}|p{0.11\textwidth}|p{0.11\textwidth}|}

\hline
Regularization & \texttt{ReLU} &\texttt{SoftPlus} & $\exp$  &\textbf{RED (Ours)}\\
\hline
 $\lambda_1 = 0.000$    &$97.06_{\pm0.19}$&$97.07_{\pm0.24}$&$98.85_{\pm0.03}$&$98.82_{\pm0.04}$\\
  $\lambda_1 = 1.000$   &$98.19_{\pm0.08}$&$98.21_{\pm0.05}$&$98.79_{\pm0.02}$&$\bf{99.10_{\pm0.02}}$\\
   $\lambda_1 = 10.000$ &$83.17_{\pm4.54}$&$80.37_{\pm18.70}$&$98.14_{\pm0.07}$&$98.84_{\pm0.03}$\\
\hline
\end{tabular}
\begin{tabular}{|p{0.66\textwidth}|}
    \hline
    Evidential CE loss\\
    \hline
\end{tabular}
\begin{tabular}{|p{0.12\textwidth}|p{0.11\textwidth}|p{0.11\textwidth}|p{0.11\textwidth}|p{0.11\textwidth}|}

\hline
 $\lambda_1 = 0.000$    &$97.03_{\pm0.21}$&$97.09_{\pm0.21}$&$98.84_{\pm0.02}$&$98.81_{\pm0.01}$\\
  $\lambda_1 = 1.000$   &$98.27_{\pm0.02}$&$98.36_{\pm0.02}$&$98.87_{\pm0.03}$&$\bf{99.12_{\pm0.02}}$\\
   $\lambda_1 = 10.000$ &$97.46_{\pm1.02}$&$97.14_{\pm1.42}$&$98.31_{\pm0.07}$&$98.84_{\pm0.04}$\\
\hline
\end{tabular}
\begin{tabular}{|p{0.66\textwidth}|}
    \hline
    Evidential MSE loss\\
    \hline
\end{tabular}
\begin{tabular}{|p{0.12\textwidth}|p{0.11\textwidth}|p{0.11\textwidth}|p{0.11\textwidth}|p{0.11\textwidth}|}

\hline
 $\lambda_1 = 0.000$    &$96.18_{\pm0.02}$&$96.20_{\pm0.03}$&$98.42_{\pm0.03}$&$98.41_{\pm0.06}$\\
  $\lambda_1 = 1.000$   &$97.41_{\pm0.22}$&$97.45_{\pm0.16}$&$98.35_{\pm0.05}$&$\bf{99.02_{\pm0.00}}$\\
   $\lambda_1 = 10.000$ &$19.93_{\pm6.98}$&$27.14_{\pm6.37}$&$27.17_{\pm3.72}$&$98.76_{\pm0.03}$\\
\hline
\end{tabular}
\label{ap:CompleteMNISTResults}
\end{table}

\begin{table}[ht]
\centering
\small
    \caption{Classification performance comparison: Cifar10 Dataset
    }
    \label{tab:appCifar10Res}
\begin{tabular}{|p{0.66\textwidth}|}
    \hline
    \centering
    Standard CE Model: $95.43_{\pm0.02}\%$
\end{tabular}
\begin{tabular}{|p{0.66\textwidth}|}
    \hline
    Log loss\\
    \hline
\end{tabular}
\begin{tabular}{|p{0.12\textwidth}|p{0.11\textwidth}|p{0.11\textwidth}|p{0.11\textwidth}|p{0.11\textwidth}|}

\hline
Regularization & \texttt{ReLU} &\texttt{SoftPlus} & $\exp$  &\textbf{RED (Ours)}\\
\hline
 $\lambda_1 = 0.000$    &$43.83_{\pm14.60}$&$95.19_{\pm0.10}$&$95.35_{\pm0.02}$&$95.03_{\pm0.14}$\\
 $\lambda_1 = 0.100$    &$41.43_{\pm19.60}$&$95.18_{\pm0.11}$&$95.11_{\pm0.10}$&$\bf{95.24_{\pm0.06}}$\\
 $\lambda_1 = 1.000$    &$38.42_{\pm15.64}$&$94.94_{\pm0.22}$&$93.95_{\pm0.06}$&$94.78_{\pm0.17}$\\
 $\lambda_1 = 10.000$   &$10.00_{\pm0.00}$&$32.42_{\pm6.99}$&$23.29_{\pm5.24}$&$90.96_{\pm0.35}$\\
 $\lambda_1 = 50.000$   &$10.00_{\pm0.00}$&$10.00_{\pm0.00}$&$12.47_{\pm3.49}$&$65.09_{\pm0.74}$\\
\hline
\end{tabular}
\begin{tabular}{|p{0.66\textwidth}|}
    \hline
    Evidential CE loss\\
    \hline
\end{tabular}
\begin{tabular}{|p{0.12\textwidth}|p{0.11\textwidth}|p{0.11\textwidth}|p{0.11\textwidth}|p{0.11\textwidth}|}

\hline
 $\lambda_1 = 0.000$    &$79.19_{\pm16.06}$&$95.32_{\pm0.17}$&$95.38_{\pm0.10}$&$\bf{95.40_{\pm0.14}}$\\
 $\lambda_1 = 0.100$    &$75.97_{\pm20.56}$&$95.12_{\pm0.05}$&$95.33_{\pm0.03}$&$95.08_{\pm0.07}$\\
  $\lambda_1 = 1.000$   &$75.83_{\pm20.74}$&$94.99_{\pm0.08}$&$94.65_{\pm0.04}$&$94.74_{\pm0.11}$\\
 $\lambda_1 = 10.000$   &$10.00_{\pm0.00}$&$89.63_{\pm0.38}$&$56.54_{\pm4.80}$&$91.71_{\pm0.23}$\\
 $\lambda_1 = 50.000$   &$10.00_{\pm0.00}$&$27.03_{\pm2.62}$&$25.33_{\pm6.66}$&$62.98_{\pm0.84}$\\
\hline
\end{tabular}
\begin{tabular}{|p{0.66\textwidth}|}
    \hline
    Evidential MSE loss\\
    \hline
\end{tabular}
\begin{tabular}{|p{0.12\textwidth}|p{0.11\textwidth}|p{0.11\textwidth}|p{0.11\textwidth}|p{0.11\textwidth}|}

\hline
 $\lambda_1 = 0.000$   &$\bf{95.43_{\pm0.05}}$&$95.35_{\pm0.15}$&$95.10_{\pm0.04}$&$94.92_{\pm0.12}$\\
 $\lambda_1 = 0.100$    &$95.15_{\pm0.10}$&$95.04_{\pm0.05}$&$95.14_{\pm0.03}$&$95.03_{\pm0.13}$\\
  $\lambda_1 = 1.000$   &$49.68_{\pm29.48}$&$93.51_{\pm0.03}$&$18.98_{\pm1.82}$&$94.90_{\pm0.20}$\\
 $\lambda_1 = 10.000$   &$10.00_{\pm0.00}$&$10.00_{\pm0.00}$&$10.00_{\pm0.00}$&$90.15_{\pm0.71}$\\
 $\lambda_1 = 50.000$   &$10.00_{\pm0.00}$&$10.00_{\pm0.00}$&$10.00_{\pm0.00}$&$27.11_{\pm24.20}$\\
\hline
\end{tabular}
\label{ap:CompleteCifar10Results}
\end{table}

\begin{table}[ht]
\centering
\small
    \caption{Classification performance comparison: Cifar100 dataset
    }
    \label{tab:appCifar100AllRes}
\begin{tabular}{|p{0.66\textwidth}|}
    \hline
    \centering
    Standard CE Model: $75.67\pm0.11$
\end{tabular}
\begin{tabular}{|p{0.66\textwidth}|}
    \hline
    Log loss\\
    \hline
\end{tabular}
\begin{tabular}{|p{0.12\textwidth}|p{0.11\textwidth}|p{0.11\textwidth}|p{0.11\textwidth}|p{0.11\textwidth}|}

\hline
Regularization & \texttt{ReLU} &\texttt{SoftPlus} & $\exp$  &\textbf{RED (Ours)}\\
\hline
 $\lambda_1 = 0.000$    &$56.69_{\pm5.83}$&$73.85_{\pm0.20}$&$76.25_{\pm0.16}$&$76.26_{\pm0.27}$\\
 $\lambda_1 = 0.001$    &$61.27_{\pm3.79}$&$74.48_{\pm0.17}$&$76.12_{\pm0.04}$&$\bf{76.43_{\pm0.21}}$\\
 $\lambda_1 = 0.010$    &$54.20_{\pm5.93}$&$75.56_{\pm0.43}$&$76.02_{\pm0.16}$&$76.14_{\pm0.09}$\\
 $\lambda_1 = 0.100$    &$20.29_{\pm4.54}$&$75.67_{\pm0.22}$&$72.72_{\pm0.26}$&$74.62_{\pm0.21}$\\
  $\lambda_1 = 1.000$   &$1.00_{\pm0.00}$&$37.60_{\pm0.82}$&$2.59_{\pm0.52}$&$68.62_{\pm0.03}$\\
  $\lambda_1 = 2.000$   &$1.00_{\pm0.00}$&$1.57_{\pm0.35}$&$0.97_{\pm0.06}$&$62.33_{\pm0.52}$\\
\hline
\end{tabular}
\begin{tabular}{|p{0.66\textwidth}|}
    \hline
    Evidential CE loss\\
    \hline
\end{tabular}
\begin{tabular}{|p{0.12\textwidth}|p{0.11\textwidth}|p{0.11\textwidth}|p{0.11\textwidth}|p{0.11\textwidth}|}

\hline
 $\lambda_1 = 0.000$    &$66.37_{\pm3.47}$&$73.73_{\pm0.38}$&$75.91_{\pm0.20}$&$76.19_{\pm0.22}$\\
 $\lambda_1 = 0.001$    &$68.62_{\pm2.41}$&$74.44_{\pm0.08}$&$76.23_{\pm0.09}$&$\bf{76.35_{\pm0.06}}$\\
 $\lambda_1 = 0.010$    &$71.94_{\pm0.66}$&$75.45_{\pm0.12}$&$75.95_{\pm0.14}$&$76.13_{\pm0.24}$\\
 $\lambda_1 = 0.100$    &$67.25_{\pm1.84}$&$75.75_{\pm0.21}$&$74.02_{\pm0.09}$&$74.69_{\pm0.13}$\\
  $\lambda_1 = 1.000$   &$1.00_{\pm0.00}$&$73.10_{\pm0.20}$&$37.36_{\pm0.73}$&$69.40_{\pm0.16}$\\
  $\lambda_1 = 2.000$   &$1.00_{\pm0.00}$&$52.99_{\pm0.56}$&$12.94_{\pm1.11}$&$63.93_{\pm0.34}$\\
\hline
\end{tabular}
\begin{tabular}{|p{0.66\textwidth}|}
    \hline
    Evidential MSE loss\\
    \hline
\end{tabular}
\begin{tabular}{|p{0.12\textwidth}|p{0.11\textwidth}|p{0.11\textwidth}|p{0.11\textwidth}|p{0.11\textwidth}|}

\hline
 $\lambda_1 = 0.000$    &$35.76_{\pm2.81}$&$20.45_{\pm1.41}$&$75.70_{\pm0.47}$&$75.55_{\pm 0.24}$\\
 $\lambda_1 = 0.001$    &$31.49_{\pm 0.31}$&$15.74_{\pm 0.47}$&$42.95_{\pm 0.76}$&$\bf{75.73_{\pm 0.27}}$\\
 $\lambda_1 = 0.010$    &$13.60_{\pm2.44}$&$1.00_{\pm0.00}$&$1.00_{\pm0.00}$&$75.35_{\pm0.16}$\\
 $\lambda_1 = 0.100$    &$1.00_{\pm0.00}$&$1.00_{\pm0.00}$&$1.00_{\pm0.00}$&$74.00_{\pm0.13}$\\
  $\lambda_1 = 1.000$   &$1.00_{\pm0.00}$&$1.00_{\pm0.00}$&$1.00_{\pm0.00}$&$66.61_{\pm0.46}$\\
  $\lambda_1 = 2.000$   &$1.00_{\pm0.00}$&$1.00_{\pm0.00}$&$1.00_{\pm0.00}$&$63.01_{\pm0.83}$\\
\hline
\end{tabular}
\label{ap:CompleteCifar100Results}
\end{table}

\end{document}